%% file: sample_paper.tex
\documentclass[twoside]{article}
\usepackage[accepted]{aistats2023}

\usepackage{cite}
\usepackage{amsmath,amssymb,amsfonts}
\usepackage{graphicx}
\usepackage{textcomp}
\usepackage{xcolor}
\usepackage{algpseudocode}
\usepackage[linesnumbered,ruled,vlined]{algorithm2e}
\def\BibTeX{{\rm B\kern-.05em{\sc i\kern-.025em b}\kern-.08em
    T\kern-.1667em\lower.7ex\hbox{E}\kern-.125emX}}
\usepackage{textcomp}
\usepackage{bm} % math bold 
\usepackage{color,soul}

\usepackage{microtype}
\usepackage[utf8]{inputenc} % allow utf-8 input
\usepackage[T1]{fontenc}    % use 8-bit T1 fonts
\usepackage{url}            % simple URL typesetting
\usepackage{nicefrac}       % compact symbols for 1/2, etc.
\usepackage{microtype,mathrsfs}
\usepackage{graphicx}
\usepackage{subfigure}
\usepackage{booktabs} % for professional tables
\usepackage{float}
\usepackage{amsfonts}       % blackboard math symbols
\usepackage{caption}
\usepackage{mathtools,amssymb,comment}
\usepackage{amsmath}
\usepackage{bbm}
\usepackage{breqn} % used for dmath
\usepackage{enumitem}
\usepackage{tcolorbox}
\usepackage{cite}

\renewcommand{\baselinestretch}{1.05}
\usepackage{hyperref}

\definecolor{darkred}{RGB}{200,0,0}
\definecolor{darkgreen}{RGB}{0,120,0}
\definecolor{darkblue}{RGB}{0,0,150}

\hypersetup{colorlinks=true, linkcolor=darkred, citecolor=darkgreen, urlcolor=darkblue}

\usepackage{amsthm}

\usepackage{stfloats}
\def\BibTeX{{\rm B\kern-.05em{\sc i\kern-.025em b}\kern-.08em
    T\kern-.1667em\lower.7ex\hbox{E}\kern-.125emX}}

\input{commands}

\newcommand{\widesim}[2][1.5]{
  \mathrel{\overset{#2}{\scalebox{#1}[1]{$\sim$}}}}

\def\bea#1\eea{\begin{align}#1\end{align}}

\theoremstyle{plain}
\newtheorem{theorem}{Theorem}

\newtheorem{lemma}[theorem]{Lemma}
\newtheorem{proposition}[theorem]{Proposition}

\begin{document}
\twocolumn[
\aistatstitle{On Generalization of Decentralized Learning with Separable Data}
\aistatsauthor{ Hossein Taheri \And Christos Thrampoulidis }

\aistatsaddress{ University of California, Santa Barbara \And  University of British Columbia}  ]
%\maketitle
%%%%%%%%%%%%%%%%%%%%%%%%% ABSTRACT %%%%%%%%%%%%%%%%%%%%
\begin{abstract}
Decentralized learning offers privacy and communication efficiency when data are naturally distributed among agents communicating over an underlying graph. Motivated by overparameterized learning settings, in which models are trained to zero training loss, we study algorithmic and generalization properties of decentralized learning with gradient descent on separable data. Specifically, for decentralized gradient descent (DGD) and a variety of loss functions that asymptote to zero at infinity (including exponential and logistic losses), we derive novel finite-time generalization bounds. This complements a long line of recent work that studies the generalization performance and the implicit bias of gradient descent over separable data, but has thus far been limited to centralized learning scenarios. Notably, our generalization bounds approximately match in order their centralized counterparts. Critical behind this, and of independent interest, is establishing novel bounds on the training loss and the rate-of-consensus of DGD for a class of self-bounded losses. Finally, on the algorithmic front, we design improved gradient-based routines for decentralized learning with separable data and empirically demonstrate orders-of-magnitude of speed-up in terms of both training and generalization performance.
 \end{abstract}
\tolerance=1
\emergencystretch=\maxdimen
\hyphenpenalty=10000
\hbadness=10000
%\begin{IEEEkeywords}
%multi-agent learning, logistic regression, distributed learning. 
%\end{IEEEkeywords}

%%%%%%%%%%%%%%%%%%%%%    INTRODUCTION    %%%%%%%%%%%%%%%%%%%%
%\blfootnote{Correspondence to hossein@ucsb.edu}

\section{INTRODUCTION}
\subsection{Motivation}
Machine learning tasks often revolve around inference from data using empirical risk minimization (ERM):
\bea\label{eq:logreg}
\min_{w\in\R^d}\hat F(w):=\frac{1}{n}\sum_{i=1}^{n} f\left(w, x_{i}\right).
\eea
Here $f:\R^d\times\R^{d'}\rightarrow \R$ is a loss function and $x_i:=y_i a_i$, where $(a_i,y_i)_{i=1}^n \simiid \mathcal{D}$ represent features and labels, sampled from a distribution $\mathcal{D}$. In large scale machine learning, due to privacy concerns and communication constraints, data points are often distributed on a set of local computing agents. Decentralized learning methods aim at minimizing the global loss function \eqref{eq:logreg} while agents communicate their parameters on an underlying connected graph. The most ubiquitous of these algorithms is Decentralized Gradient Descent (DGD). Here the $\ell$\,th agent runs a step of gradient descent followed by an averaging step in which every agent replaces its parameter with the average of its neighbors \cite{nedic2009distributed}:
\bea\label{eq:dec_main}
w^{(t+1)}_\ell = \sum_{k\in\Nn_\ell} A_{\ell k} w_k^{(t)} - \eta_t\nabla \hat F_\ell (w_\ell^{(t)}). 
\eea
The superscripts signify the iteration number and $A_{\ell k}$ refers to the averaging weights used by agent $\ell$ for the parameter of agent $k\in\mathcal{N}_\ell$ where $\mathcal{N}_\ell$ is the set of neighbors of agent $\ell$. The global loss $\hat F$ is the average of local loss functions $\hat F_\ell,\, \ell\le N$, where each $\hat F_\ell$ is formed as the average empirical risk evaluated on the local training dataset $\mathcal{S}_\ell$ of the $\ell$\,th agent:
\bea\label{eq:localloss}
\hspace{-0.1cm}\hat F(w)= \frac{1}{N}\sum_{\ell=1}^N \hat F_\ell (w),\;\; \hat F_\ell(w)= \frac{1}{n_\ell} \sum_{x_j\in\mathcal{S}_\ell} f(w,x_j),
\eea
where $n_\ell$ denotes the dataset size of agent $\ell$. Convergence properties of the train loss $\hat F(\cdot)$ in DGD have been studied extensively in literature, e.g., \cite{nedic2009distributed,nedic2014distributed,yuan2016convergence,lian2017can,nedic2016stochastic}. The bulk of these studies build upon classical optimization theory \cite{nesterov2003introductory} suited for studying the train loss per iteration. In particular, it is well-stablished in the literature that DGD converges at the rate $\frac{1}{T}\sum_{t=1}^T \hat F(\bar w^{(t)})-\hat F^\star = O(\frac{1}{\sqrt{T}})$ for smooth convex functions \cite{nedic2014distributed}. Here $\bar w^{(t)}$ is the average of local parameters $w_\ell^{(t)}$. Our results in Sections \ref{sec:convex}-\ref{sec:LIC} show a rate of $\hat F(\bar w^{(T)}) = O(\frac{(\log T)^2}{T})$ and $\|W^{(T)} - \bar W^{(T)}\|_F^2=O(\frac{(\log T)^4}{T^2})$ for the training loss and consensus error of DGD over separable data with ``exponentially tailed'' losses.
% In Section \ref{sec:LIC}, we improve this bound to a last iterate convergence rate $\hat F(\bar w^{(T)}) =O((\log T)^2/T)$, for a subset of self-bounded gradient losses and under an additional assumption on self-boundedness of the Hessian.

The study of generalization performance of DGD algorithms in the literature is mostly limited to empirical observations e.g., \cite{jiang2017collaborative,wang2019slowmo,koloskova2019decentralized}, making the theory behind test error performance largely unexplored. Moreover, the traditional wisdom in convergence analysis of DGD algorithms assumes the existence of a finite norm minimizer, which is often the case for ERM with non-separable training data, e.g. \cite{koloskova2020unified}. However, modern machine learning models operate in over-parameterized settings where the model perfectly interpolates the training data, i.e., it achieves perfect accuracy on the training data \cite{zhang2021understanding}. Understanding the challenges imposed by over-parameterization and the behavior of gradient descent on separable data has been the subject of several recent works \cite{soudry2018implicit,ji2018risk,arora2019fine,nacson2019convergence,chizat2020implicit,shamir2021gradient,ji2021fast,ji2021characterizing,schliserman2022stability}. Yet, they are all focused on centralized GD, while here we study the impact of the consensus error of DGD on both training and generalization errors. 

% In particular, in the case of linearly separable data and by using ERM with decreasing losses such as logistic loss $f(w,x)=\log(1+\exp(-w^\top x))$ and exponential loss $f(w,x)=\exp(-w^\top x)$, it is shown by \cite{ji2018risk,soudry2018implicit} that at iteration $t$ of gradient descent $\|w^{(t)}\|_2$ grows at the rate $O(\log(t))$. 
%%%%%%%%%%%%%%%%%%%%%%%%%%%%%%%%%%%%%%%%%%%%%%%%%%

Our first goal is to complement prior general results on the convergence of training loss in DGD by considering specific, but commonly encountered, settings in ERM over separable data. This includes the analysis of non-smooth objectives such as the exponential loss, analysis of logistic regression in the separable regime where the optimum is achieved at infinity, and analysis of objectives satisfying the PL condition. The second goal is to study, for the first time in these settings, convergence rates of the DGD test loss $F(\bar w^{(t)}):=\E_{x\sim\mathcal{D}}[f(\bar w^{(t)},x)]$.  Finally, we leverage recent advances in the study of centralized learning with separable data to design fast algorithms for decentralized learning. We discuss our contributions below.

\paragraph{Contributions.} 
%Our contributions are summarized as follows.
%
In Sections \ref{sec:convex} and \ref{sec:polyak}, we derive convergence rates for the training and test loss of DGD over separable data. Our results hold for convex losses satisfying realizability and self-boundedness, as well as, convex losses satisfying self-boundedness and the PL condition. In Section \ref{sec:LIC}, we prove under additional self-boundedness assumptions on the Hessian and gradient, which hold for exponentially tailed losses, that the test loss bound can be improved to approximately match the test loss bounds of centralized GD. When specialized to decentralized logistic regression on separable data, our results provide the first generalization guarantees of DGD.
In Section \ref{sec:alg}, we propose two algorithms for speeding up the convergence of decentralized learning under separable data. Numerical experiments demonstrate that our proposed algorithms significantly improve both the train test error of decentralized logistic regression.

%%%%%%%%%%%%%%%%%%%%%%%%%%%%%%%%%%%%%%%%%%%%%%%%%%%%

\subsection{Further related works}
\paragraph{Decentralized learning.} Over the last few years there have been numerous research works which consider the convergence of first order methods for decentralized learning; an incomplete list includes \cite{nedic2009distributed,nedic2014distributed,yuan2016convergence,lian2017can,jiang2017collaborative,assran2018stochastic,pu2020push,koloskova2020unified,kovalev2020optimal,xin2021improved,toghani2022communication,toghani2022scalable}. While DGD is suboptimal for strongly-convex objectives \cite{nedic2014distributed,nedic2016stochastic}, alternative algorithms, namely EXTRA and Grading Tracking, for achieving exponential rate appeared in \cite{shi2015extra,nedic2017achieving} and were studied further in \cite{koloskova2021improved,xin2021improved}. More recently, \cite{lin2021quasi} proposes accelerated methods for improving generalization and training accuracy of decentralized algorithms; however, their study of generalization error is empirical. While this paper was nearing completion we became aware of the recent works \cite{sun2021stability,richards2020graph} which study the generalization bounds of decentralized methods for Lipschitz convex losses (see also \cite{richards2020decentralised,sun2022high}). However, we consider exponentially tailed losses under the separable data regime and prove faster convergence and generalization rates under these conditions. Compared to these works, we also propose improved algorithms for learning with separable data. 
Finally, we highlight that our rates on the train loss are comparable to \cite[Theorem 2]{koloskova2020unified}. While \cite{koloskova2020unified} also derives convergence of DGD train loss on separable data, their analysis is valid only for bounded optimizers. In contrast, we derive training loss bounds which are true for the case of unbounded optimizers as is the case for logistic regression over separable data.

\paragraph{Implicit bias of GD.} An early work on the behavior of ERM with vanishing regularization on separable data appeared in \cite{Rosset2003MarginML}. Closely related, a line of recent works \cite{soudry2018implicit,ji2018risk,nacson2019convergence,ji2020directional,ji2021characterizing,shamir2021gradient,schliserman2022stability} studies the parameter convergence, as well as training and test loss convergence, of gradient descent on separable data, showing that for (a class of) monotonic losses the solution to ERM and the max-margin solution are the same in direction., i.e., $\|\hat w^{(t)} - \hat w_{_{\rm MM}}\|\rightarrow 0$. Here $\hat w^{(t)}:={{w}^{(t)}}/{\| {w}^{(t)}\|}$ and $\hat w_{_{\rm MM}}:={{w}_{_{\rm MM}}}/{\|{w}_{_{\rm MM}}\|}$, where the vector ${w}_{_{\rm MM}}$ is the solution to the hard-margin support vector machine problem,
$$
w_{_{\rm MM}}:= \arg\min_{w\in \R^d} \|w\| \;\;\; \text{s.t.} \;\;\;\; y_i w^{\top}a_i \ge1,\;\;\;  \forall i \in  [n]. \vspace{-0.03in}
$$
%In over-parameterized settings, the training dataset most often can be perfectly interpolated i.e., the model can achieve zero training error after sufficient iterations of gradient descent.
% It is known \cite{Rosset2003MarginML} that in linearly separable datasets, the normalized weight vector in gradient descent (with vanishing ridge-regularization parameter) converges to the normalized max-margin vector  i.e., $\|{{w}^{(t)}}/{\| {w}^{(t)}\|} - {{w}_{_{\rm MM}}}/{\|{w}_{_{\rm MM}}\|}\|\rightarrow 0$, where the vector ${w}_{_{\rm MM}}$ is the solution to the hard-margin support vector machine (SVM) problem defined as following,
%$$
%w_{_{\rm MM}}:= \arg\min_{w\in \R^d} \|w\| \;\;\; \text{s.t.} \;\;\;\; y_i w^{\top}a_i \ge1,\;\;\;  \forall i\le n. 
%$$
Notably, \cite{ji2018risk,soudry2018implicit} characterized the rate of directional convergence to be $\|\hat w^{(T)} - \hat w_{_{\rm MM}}\| = O({1}/{\log (T)})$ and for the training loss to be $\hat F(w^{(T)}) =O(\frac{1}{\eta T})$. Recently, Shamir \cite{shamir2021gradient} and Schliserman and Koren \cite{schliserman2022stability} showed that the test loss of GD for logistic regression on linearly separable data satisfies $F(w^{(T)})= \tilde{O}(\frac{1}{\eta T}+\frac{1}{n})$ signifying that overfitting does not happen during the iterates of GD. In Section \ref{sec:LIC} (Remark \ref{rem:improved}), we show that the test loss of DGD with logistic regression on linearly separable data satisfies $\E[F(\bar w^{(T)})]  = \tilde{O}(  \frac{1}{\eta T} +  \frac{1 }{n}+ \eta^2)$, where the expectation is taken over training samples chosen i.i.d. from the dataset. As we explain, the term $\eta^2$ captures the impact of consensus error (i.e., decentralization) on the generalization rate. %This signifies that, under data separability, the rate of decay for the test loss is essentially identical for centralized and decentralized scenarios. \ct{I don't see that. There is an $\eta^2$ term in the decentralized bound you wrote.}

%This contrast between the rates signifies that the test error is undergoing slow changes (as the margin gap is slowly improving) even when the training loss is extremely small and the misclassification train error is exactly zero. 
\par
While directional convergence is significantly slow for gradient descent, following the update rule
$
w^{(t+1)} =w^{(t)} - \eta_t\frac{\nabla{\hat F}(w^{(t)})}{\|\nabla{\hat F}(w^{(t)})\|},
$
it can be improved to $1/\sqrt{t}$ with decaying $\eta_t$ at rate $1/\sqrt{t}$ for linear models \cite{nacson2019convergence}. Furthermore \cite{taheri2023NGD} proved improved training convergence of this algorithm for two-layer neural networks, suggesting the benefits extend to non-linear settings. These results apply to centralized optimization scenarios. However, in decentralized learning settings, the local loss functions are kept private and any information about the global loss, such as its gradient $\|\nabla \hat F(\bar w^{(t)})\|$ is hidden from the agents. In Section \ref{sec:alg}, we propose algorithms which address these challenges and extend the normalized GD update rule to decentralized learning scenarios. Furthermore, we prove the asymptotic convergence of normalized local parameters $w_i^{(t)}/\| w_i^{(t)}\|$ to the solution of centralized GD.  
\paragraph{Notation}
We use $\|\cdot\|$ to denote the $\ell_2$-norm of vectors and the operator norm of matrices. The Frobenius norm of a matrix $W$ is shown by $\|W\|_F$. The set $\{i\in\mathbb{N}:i\le N\}$ is denoted by $[N]$. The gradient and hessian of a function $F:\R^d\rightarrow\R$ are denoted by $\nabla F(\cdot)$ and $\nabla^2 F(\cdot)$, respectively. For functions $f,g:\R\rightarrow\R$, we write $f(t)=O(g(t))$ when $|f(t)|\le M g(t)$ after $t\ge t_0$ for positive constants $M,t_0$. Finally, we write $f(t)=\tilde{O}(g(t))$ when $f(t)=O(g(t)h(t))$ for a polylogarithmic function $h$. 
%%%%%%%%%%%%%%%%%%%%%%% MAIN RESULTS%%%%%%%%%%%%%%%%%%%%%

\section{MAIN RESULTS}\label{sec:main}
Throughout the paper we make the following standard assumption on the mixing matrix $A=[A_{ij}]_{N\times N}$ corresponding to the  underlying connected network. 
\begin{ass}[Mixing matrix]\label{ass:mixing}
The mixing matrix $A\in \R^{N\times N}$ is symmetric, doubly stochastic with bounded spectrum i.e., $|\lambda_i(A)| \in (0,1]$ and $\lambda_2(A)<1$. 
\end{ass}
First, we state a lemma which relates the generalization loss of DGD at iteration $t$ to its train loss and consensus error up to iteration $t$. The lemma is derived based on a stability analysis \cite{bousquet2002stability,hardt2016train,lei2020fine}. Specifically we use a self-boundedness and a realizability assumption \cite{schliserman2022stability} which makes the stability analysis feasible for settings such as logistic regression on separable data. Additionally, we assume convexity and $L$-smoothness of the loss function. Formally, we assume the following, where for simplicity, we use the short-hand $f_x(w):=f(w,x)$ for the loss incurred at a generic $x\in\Dc$ in the data distribution $\Dc$.% and $f_i(w):=f(w,x_i), i\le n$ for the loss at points in the training set.}% We require the following assumptions for all $x$.} 
%For simplicity, we denote $f_i(w):=f(w,x_i)$ for $i\le n$. 

\begin{ass}[Convexity]\label{ass:convex}
The loss functions $f_x:\R^d\rightarrow\R$ are convex and differentiable, satisfying, 
$
f_x(w) \leq f_x(v)+\langle\nabla f_x(w),w-v\rangle.
$
\end{ass}
\begin{ass}[Smoothness]\label{ass:smooth}
The loss functions $f_x:\R^d\rightarrow\R$ are $L$-smooth and differentiable, i.e. 
%\bea\nn
$f_x(w)\le f_x(v) + \langle\nabla f_x(v),w-v\rangle + \frac{L}{2} \|w-v\|^2.$
%\eea
\end{ass}
\begin{ass}[Self-boundedness of the gradient]\label{ass:self-bounded}
The loss functions $f_x:\R^d\rightarrow\R$ satisfy the self-boundedness property with the parameters {$c>0$ and $\alpha\in[\frac{1}{2},1]$, i.e.,}
%$(c,\alpha)$ , 
\bea\nn
\| \nabla f_x(w)\| \le c  \left(f_x(w)\right)^\alpha.
\eea
\end{ass}\vspace{-0.02in}
Assumption \ref{ass:self-bounded} is weaker than Assumption \ref{ass:smooth}, since an $L$-smooth non-negative function $f$ satisfies $\|\nabla f(w)\|^2 \le 2L (f(w)-f^\star)\leq 2L f(w)$, where $f^\star:=\inf_w f(w)\geq0$.  However, we make use of the smoothness property whenever it suits the analysis, particularly to bound  training loss. 
\par 
Additionally, we make the following assumptions: All local parameters are initiated at zero i.e, $w_\ell^{(1)}=0$ for all $\ell\le N$. We assume for simplicity of exposition, that each agent has access to $n/N$ ($n_\ell=n/N$) samples from the dataset. The general case can be treated with minor modifications. We also assume that $f_x(w)\ge 0$ for all $w$ and the minimum of each loss is zero i.e., $f_i^{\star}=0$. 
%\ct{I removed data separability assumption. I find it confusing and is anyways never explicitly used as written.}
%Throughout the paper, we assume that the training data is separable.
% \begin{ass}[Data separability]\label{ass:sep}
% The training loss $\hat F$ in Eq.\eqref{eq:localloss} satisfies the data separability condition i.e., $\hat F^\star=0$, where $\hat F^\star:=\inf_{w\in \R^d} \hat F(w).$
% \end{ass}

%However, we make use of the $L$-smoothness property mainly for bounding the training error. \red{TBD}
Before our key lemma, we introduce a few necessary notations. We define matrix $W^{(t)}\in\R^{N\times d}$ as the concatenation of all agents' parameters at iteration $t$, i.e., $W=[w_1^{(t)},\cdots,w_N^{(t)}]^\top$. We also denote by $\bar w^{(t)}:= \frac{1}{N}\sum_{\ell=1}^N w_\ell^{(t)}$  the average of local parameters, and denote by $\bar W ^{(t)}= [\bar w^{(t)},\cdots,\bar w^{(t)}]\in \R^{N\times d}$ its concatenated matrix. %similarly,  we define $\bar \w_i$ and $\bar W_i$, by leaving out the i'th sample.  

%We also recall that $F,\hat F$ show the global test loss and train loss. 
\begin{lemma}[Key lemma, Informal version]\label{lem:main}
 Let Assumptions \ref{ass:mixing}-\ref{ass:self-bounded} hold. Consider the iterates of decentralized gradient descent in Eq.\eqref{eq:dec_main} with a fixed positive step-size $\eta\le \frac{2}{L}$. Then, for the test loss $F$ at iteration $T\ge1$, it holds that
\bea
\E \left[F(\bar w^{(T)})\right]&\lesssim \;\;\E \left[\hat F( \bar w^{(T)})\right]\nn \\&+ \frac{\eta^2 L^2 c^2T^{2}}{n^{3-2\alpha}}\E\Big[(\frac{1}{T}\sum_{t=1}^{T}\hat F(\bar {w}^{(t)}))^{2\alpha}\Big] \nn\\&+\frac{\eta^2 L^4}{N} \E\Big[(\sum_{t=1}^{T}\|W^{(t)}-\bar W^{(t)}\|_F)^2\Big],\label{eq:lem}
\eea
where the expectation is over the training set of $n$ i.i.d samples.
\end{lemma}

The precise statement and the proof of Lemma \ref{lem:main} are deferred to Appendix \ref{sec:lemma1}. Lemma \ref{lem:main} bounds the test loss with respect to the train loss and the consensus error. In the following sections, we show how Lemma \ref{lem:main} yields test loss bounds on DGD by establishing bounds on the train loss and consensus errors under different assumptions on the loss function.

It is worth remarking that Eq. \eqref{eq:lem} is in fact valid not only for DGD, but also for Decentralized Gradient Tracking (DGT). DGT is another popular algorithm for distributed learning that can accelerate train error convergence over DGD by modifying the update in Eq. \eqref{eq:dec_main} such that each agent keeps a running estimate of the global gradient \cite{nedic2017achieving}. The reason why \eqref{eq:lem} continues to hold for DGD is that
%This is merely because 
the proof of Lemma \ref{lem:main} only relies on the updates of the ``averaged'' parameter $\bar w^{(t)}:=\frac{1}{N}\sum_{\ell=1}^N w_\ell$ and that the update rule of $\bar w^{(t)}$ for both DGD and DGT is derived as $\bar w^{(t)} = \bar w^{(t-1)} -  \frac{\eta}{N}\sum_{\ell=1}^N \nabla \hat F_\ell(w_\ell^{(t-1)})$. Thus, starting with  Eq.\eqref{eq:lem} one can also obtain test loss bounds of DGT after replacing appropriate bounds of DGT for the training loss and consensus error. We leave this to future work.
%\ct{This is very interesting to me. You are saying that the stability bound is the same for DGT. Can't you then use some plug-in value for train/consensus error from previous works to get DGT test loss bound? That would be a great addition and something that they asked already at neurips.}
\par

%%%%%%%%%%%%%%%%%%%%%%%%%%%%%%%%%%%%%%%%%%%%%%%%%%%%%%%%%%%
%%%%%%%%%%%%%%%% Training Error/ DECENTRALIZED%%%%%%%%%%%%%%%%%%%%%%%%%%

%%%%%%%%%%%%%%%%%%%%%%%%%%%%% Lemma%%%%%%%%%%%%%%%%%%%%%%%%%%%

%\begin{lemma}
%for L-smooth, $\mu$-convex losses in a decentralized setting we have
%\bea
%\|\bar{W}_{t+1} - W^ *\|^2 \le (1-\eta \mu/2) \| \bar W^{(t)} -W^*\|^2 - \eta (\hat F(\bar W^{(t)}) - \hat F(W^*)) + 3\eta L/n  \| \bar W^{(t)} - W_{t}\|_F^2
%\eea

%\end{lemma}
%%%%%%%%%%%%%%%%%%%%%%%%%%%%% Lemma%%%%%%%%%%%%%%%%%%%%%%%%%%%

%\begin{lemma} Under the assumptions of Lemmas...., for a fixed step-size $\eta<\sqrt{\frac{1-\alpha_1}{2\alpha_2 L^3}}$, it holds that for any $w\in\R^d$
%\bea
%\frac{1}{T} \sum_{t=1}^{T-1} \hat F(\bar {w}^{(t)}) \le \frac{2\|w\|^2}{\eta T}+ 4\hat F(w)
%\eea
%\end{lemma}

%
%\red{Case 1:} $\eta=\eta = O(1)$
%
%In this case $\hat F(\bar w^{(T)})=O(1/T)$ and $\|W_T-\bar W_T\|_F = O(1/\sqrt{T})$. Thus the upper bound on generalization error becomes vacuous. 
%\vspace{.5in}
\subsection{Convergence with general convex losses}\label{sec:convex}
The upper-bound in Eq.\eqref{eq:lem} shows how the consensus error and train loss of DGD affect the test loss. 

The next lemma bounds the training loss and consensus error of DGD for general convex losses. The proof is deferred to Appendix \ref{sec:pf-lem2}
\begin{lemma}[Training bounds for convex losses] \label{lem:trainloss-convex} Under Assumptions \ref{ass:mixing}-\ref{ass:smooth},
%\ref{ass:self-bounded}, 
for any $w\in\R^d$ and for a fixed step-size $$\eta<\frac{1}{L}\min\left\{1-\alpha_1,\sqrt{\frac{1-\alpha_1}{2\alpha_2 }}\right\},$$ where $\alpha_1\in (3/4,1),\alpha_2>4$ are parameters that depend only on the mixing matrix, the train loss and consensus error of DGD \eqref{eq:dec_main} satisfy:
\bea\label{eq:trainloss_convex}
\frac{1}{T} \sum_{t=1}^{T} \hat F(\bar {w}^{(t)}) &\le \frac{2\|w\|^2}{\eta T}+ 4\hat F(w),\\
\frac{1}{NT}\sum_{t=1}^T \|W^{(t)}- \bar W^{(t)}\|_F^2 &\le \frac{\alpha_2\eta^2L^2}{1-\alpha_1}\big(\frac{2\|w\|^2}{\eta T}+ 4\hat F(w)\big).\nn
\eea
\end{lemma}

To bound the training loss for functions $f(\cdot)$ where the optimum is attained at infinity we need a realizability assumption. In particular, we choose $w\in\R^d$ (in Lemma \ref{lem:trainloss-convex}) using the following.

\begin{ass}[Realizability]\label{ass:realizable}
The loss functions $f_x:\R^d\rightarrow\R$ satisfy the realizability condition, i.e. $\exists$ decreasing function $\rho:\R_+\rightarrow\R_+$ such that for every $\eps>0$ there exists $\hat w\in \R^d$ with $\|\hat w\|\le \rho(\eps)$ that satisfies $f_x(\hat w)\le\eps$.
\end{ass}

The set of Assumptions \ref{ass:convex}-\ref{ass:realizable} covers classification over linearly separable data with logistic loss, in addition to losses with other exponential-type tails $\exp(-w^r)$ and polynomial tail $w^{-r}$, for $r>0$. 

\begin{remark}[Training loss of DGD on separable data]\label{rem:trainloss}The realizability assumption as stated appeared recently in \cite{schliserman2022stability} (and was implicitly used in \cite{ji2018risk,shamir2021gradient}). 
%as the required assumption for bounding the training error under separable data. 
It can be checked that for linearly separable training data with margin $\gamma$, loss functions with an exponential tail such as logistic loss satisfy this assumption with $\rho(\eps)= \frac{1}{\gamma}\log(\frac{1}{\eps})$ (e.g., see Proposition \ref{prop:realizable} and \cite[Lemma 4]{schliserman2022stability}). Based on Lemma \ref{lem:trainloss-convex}, this leads to the following bound for DGD training loss for all $\eps>0$,
\bea\label{eq:trainlosscvx}
\frac{1}{T} \sum_{t=1}^{T} \hat F(\bar {w}^{(t)}) \le \frac{2\log(1/\eps)^2}{\gamma^2\eta T}+ 4\eps.
\eea
In particular, choosing $\eps=1/T$, gives a rate of $O(\frac{(\log\,T)^2}{\eta T})$, surprisingly matching up to logarithmic factors the corresponding rate for centralized GD in \cite[Theorem 1.1]{ji2018risk}. 
\end{remark}
%\begin{defn}
%The global loss $\hat F$ (Eq. \eqref{eq:logreg}) is in the interpolating regime if every stationary point of $\hat F$ is a stationary point of $f_i$ for all $i\le n$.
%\end{defn}

\begin{remark}\label{rem:distribution}
The bounds of Lemma \ref{lem:trainloss-convex} are true for any dataset $\{x_i\}_{i\in[n]}$ provided that Assumptions \ref{ass:convex} and \ref{ass:smooth} hold for all $f_x=f_{x_i}=f(w,x_i):=f_i(w), i\in[n]$. Similarly, \eqref{eq:trainlosscvx} holds provided Assumption \ref{ass:realizable} is true over the training set (i.e. provided the training dataset is separable). However, bounding the test loss in Lemma \ref{lem:main}, requires bounding the \emph{expectation over all datasets} of the train/consensus errors. This is guaranteed by  Assumptions \ref{ass:convex}-\ref{ass:realizable} as they hold for any point $x$ in the distribution. 
\end{remark}

\begin{theorem}[Test loss with convex losses]\label{thm:testloss_cvx}
Under Assumptions \ref{ass:mixing}-\ref{ass:realizable}, by choosing $$\eta<\frac{1}{L\sqrt{T}}\min\left\{{1-\alpha_1},\sqrt{\frac{1-\alpha_1}{2\alpha_2}}\right\}$$ where $\alpha_1\in (3/4,1),\alpha_2>4$ are parameters that depend only on the mixing matrix and assuming $\eps\le\frac{\rho(\eps)^2}{\eta T}$, the test error of DGD for iteration $T\geq 1$ satisfies:
 \bea 
 \frac{1}{T}\sum_{t=1}^T\;&\E\left[F(\bar w^{(t)})\right]  =\nn\\ &O \Big( \frac{\rho(\eps)^2}{\sqrt{T}} +  \frac{L^2 c^2 \rho(\eps)^{4\alpha} }{n^{3-2\alpha}} T^{1-\alpha}
 + \frac{L^4\rho(\eps)^2}{\sqrt{T}}
 \Big),\label{eq:testlosscvx}
\eea
where the expectation is over the training set of $n$ i.i.d samples.
\end{theorem} 
\begin{remark}[DGD with logistic regression never overfits]\label{rem:logr}
The proof of Theorem \ref{thm:testloss_cvx} is delayed to Appendix \ref{sec:pf-thm3}. As in Remark \ref{rem:trainloss}, we take logistic regression on separable data with margin $\gamma>0$ as our case study. For logistic regression (as well as other loss functions with an exponential tail), it can be verified that the self-boundedness assumption holds with $\alpha=1$. Similar to Remark \ref{rem:trainloss} it holds that $\rho(\eps)=\frac{1}{\gamma}\log(\frac{1}{\eps})$, thus choosing $\eps=1/\sqrt{T}$ results in a test loss rate $\tilde{O}({\frac{1}{\sqrt{T}}  + \frac{1}{n}})$ by Eq.\eqref{eq:testlosscvx}. This indicates that the upper-bound decreases at a  rate of $\tilde{O}(\frac{1}{\sqrt{T}})$ until after $T=n^2 \cdot(\max({\frac{1}{L c}},\frac{L}{c}))^4$ iterations where the upper bound essentially reduces to $\tilde{\mathcal O} (\frac{L^2 c^2}{n})$. Additionally, the fact that the upper-bound is decreasing proves that with appropriate choice of step-size, overfitting never happens along the path of DGD at any iteration.
\end{remark}
\par
\begin{remark}[Log factors]\label{rem:Log} The attentive reader will have recognized in Remarks \ref{rem:trainloss} and \ref{rem:logr} that due to the ``$\rho(\eps)=\Oc(\log(T))$'' factor, the upper bound on the test loss in Eq. \eqref{eq:testlosscvx} increases (very) slowly with $\log^4(T)$. Note that this term becomes dominant only  when $T$ is exponentially large with respect to the sample size $n$ and the margin $\gamma$. Our experiments in Sec. \ref{sec:furtherex} confirm this slow logarithmic increase late in the training phase. Analogous behavior, but  for centralized GD training, are discussed in \cite{soudry2018implicit,schliserman2022stability}. 
\end{remark}

%\begin{remark}[Early stopping]
%When $\alpha=1/2$ i.e, the case where the boundedness and smoothness assumption are equivalent, assuming $T$ is large enough such that $\eps\le\frac{\rho(\eps)^2}{\sqrt{T}}$ the test loss reduces to $\tilde{O}({\frac{1}{\sqrt{T}} + \frac{L^4}{N\sqrt{T}} + \frac{L^2 c^2}{n^2}\sqrt{T}})$
%\end{remark}

%%%%%%%%%%%%%%%%%%%%%%%%%%%%%%%%%%%%%%%%%%%%%%%%%%%%%%%%
 \subsection{On the convergence of DGD with exponentially-tailed losses}\label{sec:LIC}
 In this section, we show that our guarantees can be improved for exponentially tailed losses.  First, we note that the bounds in Lemma \ref{lem:trainloss-convex} and Theorem \ref{thm:testloss_cvx} hold for the average loss across iterations $t\le T$. 
 %By convexity of $f$ the bounds can be equivalently written for $\hat F(\frac{1}{T}\sum_{t=1}^T \bar w^{(t)})$ and $F(\frac{1}{T}\sum_{t=1}^T \bar w^{(t)})$, making the bounds valid only for the time average of $\bar w^{(t)}$. 
 It is straight-forward to see that if DGD is a descent algorithm i.e., $\hat F(\bar w^{(t+1)}) \le \hat F(\bar w^{(t)}) $ for all $t\le T$, then $\hat F(\bar w^{(T)})\le \frac{1}{T} \sum_{t=1}^{T} \hat F(\bar {w}^{(t)})$; thus implying that the upper-bounds on training and test loss hold for the last iterate of DGD. We will prove that DGD is indeed a ``descent algorithm'' for a class of convex losses which include popular choices such as the logistic loss and even non-smooth choices including the exponential loss. Moreover, we show that the consensus error of Lemma \ref{lem:trainloss-convex} as well as the test loss bounds of Theorem \ref{thm:testloss_cvx} can be improved compared to the results of the previous section. 
 
 In particular, we use the following assumptions together with the self-boundedness gradient assumption (Assumption \ref{ass:self-bounded}) with $\alpha=1$ as well as the convexity assumption. 
 \begin{ass}[Self-bounded Hessian]\label{ass:laplace}
The local losses $\hat F_\ell:\R^d\rightarrow\R$ satisfy the following for the Hessian matrices $\nabla ^2 \hat F_\ell$ and a positive constant $h$,
\bea\nn
\| \nabla^2 \hat F_\ell(w)\| \le h\, \hat F_\ell(w).
\eea
\end{ass}

\begin{ass}[Self-lowerbounded gradient]\label{ass:8}
The global loss satisfies for a constant $\tau$ that $$\|\nabla \hat F (w)\|\ge\tau \hat F(w).$$
\end{ass}

 Assumptions \ref{ass:convex}, \ref{ass:self-bounded}, \ref{ass:laplace} and \ref{ass:8} include linear classification with non-smooth losses such as the exponential loss, losses with super-exponential tails ($\exp(-x^r), r>1$) and the logistic loss; e.g., see Proposition \ref{propo:exp} in the appendix. 

\begin{theorem}[Last iterate convergence of DGD]\label{lem:exp_dsc}
Consider DGD with the loss functions and mixing matrix satisfying Assumptions \ref{ass:mixing},\ref{ass:convex},\ref{ass:laplace},\ref{ass:8}  and Assumption \ref{ass:self-bounded} with $\alpha=1$ and $c=h$. Assume that the step-size satisfies $\eta<\frac{\delta}{\hat F(1)}$, for a constant $\delta$ depending only on the mixing matrix and on $\tau, h$, then DGD is a descent algorithm i.e, for all $t\ge1$ it holds that $\hat F(\bar w^{(t+1)})\le \hat F(\bar w^{(t)})$. Moreover, the train loss and the consensus error of DGD at iteration $T$ satisfy the following for all $w\in \R^d$, \bea
\hat F(\bar w^{(T)}) &\le 4\hat F(w) + \frac{2\|w\|^2}{\eta T},\nn\\
\left\|W^{(T)} - \bar W^{(T)}\right\|_F^2 &= O \left(  h^2\eta^2 \hat F^2(w) + \frac{h^2\|w\|^4}{T^2}\right).\label{eq:consensus_improved}
\eea
\end{theorem}

The proof of Theorem \ref{lem:exp_dsc} is included in Appendix \ref{sec:pf-thm4}. In the following remark, we discuss the implications of this result.

\begin{remark}[Improved rates]\label{rem:improved}
While similar to Lemma \ref{lem:trainloss-convex}, for logistic regression we have $\hat F(\bar w^{(T)})=\tilde{O}(\frac{1}{\eta T} + \frac{1}{T})$, for the consensus error rate we have by applying Theorem \ref{lem:exp_dsc} and noting that $\rho(\eps)=\log(1/\eps)/\gamma$ , 
\bea
\left\|W^{(T)} - \bar W^{(T)}\right\|_F^2 &= O \left(  h^2 \eta^2 \eps^2 + \frac{h^2 (\log (1/\eps))^4}{\gamma^4 T^2}\right).\nn
\eea
After choosing $\eps=1/T$, we have the improved rate $\|W^{(T)} - \bar W^{(T)}\|_F^2=\tilde{O}(\frac{1}{T^2})$, which is superior over the rate $\tilde{O}(\frac{1}{T})$ for general convex losses with constant $\eta$ (Lemma \ref{lem:trainloss-convex}).
For the test loss, employing Lemma \ref{lem:main} with the new rates for the consensus error leads to the following rate for DGD with logistic regression,
 \bea \label{eq:test_improved}
\E\left[F(\bar w^{(T)})\right]  = \tilde{O} \Big(  \frac{1}{\eta T} +  \frac{1 }{n}
+ \eta^2\Big).
\eea
%\ct{I don't see the above. What I see is:
%\bea 
%\E[F(\bar w^{(T)})]  = \tilde{O} \Big(  \frac{1}{\eta T} +  \frac{1 }{n}
%+ \frac{\eta^4}{NT}\Big).
%\eea
%}
In accordance to Remark \ref{rem:distribution}, we can conclude the above from Lemma \ref{lem:main} provided Assumptions \ref{ass:self-bounded} and \ref{ass:8}. Thus, the bounds of Theorem \ref{lem:exp_dsc} remain true for all training sets within  the data distribution.
We note that the resulting bound in \eqref{eq:test_improved} is a superior rate for the test loss of logistic regression, compared to the rate of Remark \ref{rem:logr}. Concretely, setting $\eta=1/T^{1/3}$ gives a rate of $\tilde{O}(1/T^{2/3}+1/n)$, faster than the $\tilde{O}(1/\sqrt{T}+1/n)$ rate in Remark \ref{rem:logr}. On the other hand, it is slightly slower compared to its centralized counterpart $\tilde{O}(1/T +  1/n)$ in \cite{shamir2021gradient,schliserman2022stability}. As revealed by Lemma \ref{lem:main}, the additional $\eta^2$ factor in \eqref{eq:test_improved} captures impact of the consensus term, which is unavoidable in decentralized learning. 
% Notably, this is comparable to the test loss rate for ``centralized'' logistic regression on linearly separable data, derived as $\E[F(\bar w^{(T)})]  = \tilde{O}( \frac{1}{\eta T} +  \frac{1 }{n})$ in \cite{shamir2021gradient,schliserman2022stability}. 
 \end{remark}

%%%%%%%%%%%%%%%%%%%%%%%%% PL %%%%%%%%%%%%%%%%%%%%%%%%%%%%%%
\subsection{Convergence under the PL condition}\label{sec:polyak}
Next, we show how our previous results change when the global loss satisfies the $\mu$-PL condition. Formally, the PL condition \cite{Polyak1963GradientMF,Lojasiewicz} is defined as follows.  
\begin{ass}[PL condition]\label{ass:polyak}
The loss function $\hat F:\R^d\rightarrow\R$ satisfies the Polyak-Lojasiewic(PL) condition with parameter $\mu>0$:
$
\| \nabla \hat F(w)\|^2 \ge 2\mu (\hat F(w)-\hat F^\star).
$
\end{ass}
The next lemma shows that DGD enjoys an exponential rate under the PL condition and  smoothness.
and data separability (i.e., $\hat F^\star = 0$). 
See Appendix \ref{sec:pf-lem5} for a proof. 
 \begin{lemma}[Train loss under the PL condition]\label{lem:pl-tr}
Let Assumptions \ref{ass:mixing},\ref{ass:smooth} and \ref{ass:polyak} hold and let the step-size $\eta\le\min\{\frac{1-\alpha_1}{\mu},\frac{1}{2L^2}\sqrt{\frac{(1-\alpha_1)\mu}{\alpha_2}},\frac{1}{L}\}$, where the constants $\alpha_1\in(3/4,1)$ and $\alpha_2>4$ depend only on the mixing matrix. Define  $\zeta:=1-\frac{\eta\mu}{2}$, then under the data separability assumption, the iterates of DGD satisfy for all $t\ge1$,
\bea
\hat F(\bar w^{(t)})  &\le \zeta^{t-1} \hat F(\bar w^{(1)}),\nn\\
\frac{1}{N}\left \|W^{(t)} - \bar W^{(t)}\right\|_F^2  &\le   \frac{2\alpha_2\eta^2L^2\hat F(\bar w^{(1)})}{1-\alpha_1}\zeta^{t-1}.\nn
\eea
 \end{lemma}
 
 We use this lemma combined with our key lemma \ref{lem:main} to obtain the test loss bound in the next theorem. The proof is provided in Appendix \ref{sec:pf-thm6}.

\begin{theorem}[Test loss under the PL condition]\label{thm:PL}
Let Assumptions \ref{ass:mixing}-\ref{ass:self-bounded} hold. Further assume \ref{ass:polyak} holds for all training sets in the distribution. 
%and \ref{ass:polyak} hold, 
%and let 
Let $\eta$ and $\zeta$ be as in Lemma \ref{lem:pl-tr}. Then the iterates of DGD satisfy for all $T\ge1$,
 \bea%\label{eq:pl_test_bound}
 \E\left[F(\bar w^{(T)})\right] =O \Big(\zeta^T  + \frac{ L^2 c^2}{n^{3-2\alpha} \mu^{2\alpha}}  (\eta T)^{2-2\alpha} +\frac{\eta^2 L^4}{\mu^2}\Big).\nn
\eea
\end{theorem} 
\begin{remark}
The bound above involves $(\eta T)^{2-2\alpha}$. When $\alpha<1$, as in the case of smooth functions such as highly over-parameterized Least-squares $f(w,x)=(1-w^\top x)^2$ where $d\gg n$, the bound becomes vacuous as it is increasing with $T$. This suggests the existence of overfitting in DGD under such scenarios; with the optimal value of $T$ achieved at the very early steps of training.
%gradient descent
See Appendix \ref{sec:plexp} for experiments that confirm this behavior. 
%On the other hand, when $\alpha=1$, test loss decreases exponentially fast to reach the quantity $O(\frac{ L^2 c^2}{n \mu^{2}}+\frac{\eta^2 L^4}{\mu^2 N} )$.

\end{remark}

 %%%%%%%%%%%%%%%%%%%%%%%%%%%%%%%%%%%%%%%%%%%%%%%%%%%%%
% \begin{remark}[Over-parameterized Least-squares]
% \end{remark}
 \subsection{Improved algorithms for decentralized learning with separable data}\label{sec:alg}
In this section, we consider decentralized learning with exponentially tail losses on separable data and propose modifications to the DGD algorithm for improving the convergence rates based on the normalized GD mechanism. 

Our first algorithm --Fast Distributed Logistic Regression($\mathsf{FDLR}$)-- is summarized in Algorithm \ref{alg:alg1}. Each agent keeps two local variables ${w}_\ell,{v}_\ell\in \R^d$ which are also communicated to neighbor agents at each round. In matrix notation, Algorithm \ref{alg:alg1} has the following updates:
\bea
&W^{(t+1)} = A(W^{(t)} - \eta \widetilde{V}^{(t)}),\nn\\
&V^{(t+1)}= A\,V^{(t)} + \nabla{\hat F}(W^{(t+1)})-\nabla{\hat F}(W^{(t)}).\nn
\eea

\SetKwInput{KwInput}{Input}               
\SetKwInput{KwOutput}{Output}   
\begin{algorithm}[t]
\caption{$\mathsf{FDLR}$}\label{alg:alg1}
 \KwInput{Initial values $w_\ell^{(1)},{v}_\ell^{(1)}\in \R^d$ for all agents $\ell\in[N]$, step size $\eta_t$ and mixing matrix $A=[A_{\ell k}]_{N\times N}$ }
% \KwOutput{Model {w}_i}

 \For{$t=1,\dots,T$ all agents $\ell\in[N]$ in parallel }
 {
 $w_\ell^{(t+\frac{1}{2})} = w_\ell^{(t)}-\eta_t \frac{{v}_\ell^{(t)}}{\|{v}_\ell^{(t)}\|}$\\[1pt]
send and receive local variables ${w}_\ell^{(t+\frac{1}{2})}$ and ${v}_\ell^{(t)}$\\[3pt]
$w_\ell^{(t+1)} = \sum_{k\in\Nn_\ell} A_{\ell k} {w}_k^{(t+\frac{1}{2})}$\\[4pt]
${\small {v}_\ell^{(t+1)}=\sum_{k\in\Nn_\ell} A_{\ell k}{v}_k^{(t)} + \nabla{\hat F}_\ell(w_\ell^{(t+1)})-\nabla{\hat F}_\ell(w_\ell^{(t)})}$
 }
%  \KwOutput{Your output}
\end{algorithm}

\begin{algorithm}[t]
\caption{$\mathsf{FDLR}$ with Nesterov momentum}\label{alg:alg2}
 \KwInput{Initial values ${w}_\ell^{(1)},{v}_\ell^{(1)},z^{(1)}\in \R^d$ for all agents $\ell\in[N]$, hyper-parameters $\eta_t,\gamma_t$ and mixing matrix $A=[A_{\ell k}]_{N\times N}$ }
 
 \For{$t=1,\dots,T$ all agents $i\in[N]$ in parallel }
 {
${z}_\ell^{(t+1)} = \gamma_t ({z}_\ell^{(t)} +\frac{{v}_\ell^{(t)}}{\|{v}_\ell^{(t)}\|})$\\[4pt]
${w}_\ell^{(t+\frac{1}{2})} = {w}_\ell^{(t)} - \eta_t({z}_\ell^{(t+1)}+\frac{{v}_\ell^{(t)}}{\|{v}_\ell^{(t)}\|})$\\[2pt]
send and receive local variables ${w}_\ell^{(t+\frac{1}{2})}$ and ${v}_\ell^{(t)}$\\[4pt]
${w}_\ell^{(t+1)} = \sum_{k\in\Nn_\ell} A_{\ell k}{w}_k^{(t+\frac{1}{2})}$\\[4pt]
${\small {v}_\ell^{(t+1)} =  \sum_{k\in\Nn_\ell} A_{\ell k}{v}_k^{(t)} +  \nabla{\hat F}_\ell({w}_\ell^{(t+1)})-\nabla{\hat F}_\ell({w}_\ell^{(t)})}$
 }
%  \KwOutput{Your output}
\end{algorithm}

As in \eqref{eq:dec_main}, $A\in \R^{N\times N}$ is the mixing matrix of the undirected network of agents, which satisfies the regularity conditions in Assumption \ref{ass:mixing}. Furthermore $W^{(t)}, V^{(t)},\nabla{\hat F}(W^{(t)})\in \R^{N\times d}$ are formed by stacking ${w}_\ell^{(t)}, {v}_\ell^{(t)}$ and local gradients $\nabla{\hat F}_\ell({w}_\ell^{(t)})$ for all $\ell\in[N]$ as their rows. The matrix $\widetilde{V}^{(t)}\in\R^{N\times d}$ is formed by concatenation of the vectors ${v}_\ell^{(t)}/\|{v}_\ell^{(t)}\|$ as its rows.
In step (2) of Algorithm \ref{alg:alg1}, every agent $\ell$ runs in parallel an update rule which resembles the distributed gradient descent update rule (aka Eq. \eqref{eq:dec_main}), with the difference that the local gradient $\nabla{\hat F}_\ell({w}_\ell^{(t)})$ is replaced by ${v}_\ell^{(t)}/\|{v}_\ell^{(t)}\|$. In the next step, agents send their local parameters ${w}_\ell^{(t)},{v}_\ell^{(t)}$ to their neighbors. Step (4) is the consensus step at which agent $\ell$ computes a weighted average of ${w}_k^{(t+1/2)}$ sent from neighbor agents $k$, in order to update ${w}_\ell^{(t)}$. 
Step (5) uses the newly computed local gradient $\nabla{\hat F}_\ell({w}_\ell^{(t+1)})$ and the gradient computed in the previous step to updates the local parameter ${v}_\ell^{(t)}$.  The purpose behind introducing the variable ${v}_\ell^{(t)}$ is to estimate the global gradient. This idea is previously used in the gradient tracking algorithm (e.g. see \cite{nedic2017achieving}) and the idea also relates to stochastic variance reduced gradient (SVRG) \cite{johnson2013accelerating}.
The following theorem proves that for exponentially decaying loss functions and separable data, $\mathrm{FDRL}$ with time-decaying step-size $\eta_t=1/\sqrt{t}$ converges successfully in direction to the solution of centralized gradient descent. The proof is provided in Appendix \ref{sec:alg_proofs}.

%While we do not prove convergence rates for the proposed algorithms, we study their asymptotic convergence. Notably, the next theorem shows that for a time-decaying step-size $\eta=1/\sqrt{t}$, the solution of $\mathrm{FDRL}$ converges in direction to the solution of centralized gradient descent. The proof is provided in Appendix \ref{sec:alg_proofs}. We defer numerical performance of both Algorithms to the next section.  

\begin{theorem}[Asymptotic convergence of $\mathsf{FDLR}$]\label{thm:alg}
Let the sequence $\{{w}_\ell^{(t)}\}$ be generated by $\mathsf{FDRL}$(Algorithm \ref{alg:alg1}) trained with logistic or exponential loss on a separable dataset with $\eta_t=O (1/\sqrt{t})$. Then, for all $\ell\in [N]$,
$
\lim_{t\rightarrow\infty} {{w}_\ell^{(t)}}/{\|{w}_\ell^{(t)}\|} =  {{w}_{_{\rm MM}}}/{\|{w}_{_{\rm MM}}\|},
$
where ${w}_{_{\rm MM}}$ is the solution to max-margin problem. 
\end{theorem}

 %loss.
 Based on the above result, we anticipate that $\mathsf{FDRL}$ has good test performance. In fact, we will show in Section \ref{sec:num} that $\mathsf{FDRL}$ achieves good test performance orders of magnitude faster than DGD. To get some insight on this and also on the nature of the $\mathsf{FDLR}$ updates  consider the infinite time limit. In this limit, when the matrix $A$ satisfies the mixing Assumption \ref{ass:mixing}, it can be checked that
%we can deduce that, 
$
V^{(\infty)} = \frac{1}{n} \mathbf{1 1}^{\top} \nabla{\hat F}(W^{(\infty)}).
$
Hence, as $t\rightarrow\infty$ the variables ${v}_\ell^{(t)}$ for all agents converge to the same global gradient $\sum_{\ell=1}^N \nabla{\hat F}_\ell ({w}_\ell^{(t)})$.
%While each ${v}_\ell^{(t)}$ is distinct among agents, each of them is estimating the global loss of distributed optimization i.e., $\sum_{\ell=1}^N \nabla{\hat F}_\ell ({w}_\ell^{(t)})$. 
Realizing this, we can see that 
%With this, 
Step (2) of $\mathsf{FDLR}$ is asymptotically approximating a normalized GD update, i.e., for large $t$, each agent performs an update $w_{\ell}^{(t+1/2)}\approx w_{\ell}^{(t)}-\eta_t\frac{\sum_{\ell=1}^N \nabla{\hat F}_\ell ({w}_\ell^{(t)})}{\|\sum_{\ell=1}^N \nabla{\hat F}_\ell ({w}_\ell^{(t)})\|_2}$. Previously, normalized gradient descent has been used to speed 
%which provably speeds 
up convergence in centralized logistic regression over separable data\cite{nacson2019convergence}. Here, we essentially extend this idea to a decentralized setting and argue that $\mathsf{FDLR}$ is the canonical way to do so. In particular, the idea of introducing additional variables $v_\ell$ that keep track of the global gradient is critical for the algorithm's success. That is, a naive implementation with updates $w_{\ell}^{(t+1/2)}= w_{\ell}^{(t)}-\eta_t{\nabla{\hat F_\ell} ({w}_\ell^{(t)})}/{\|\nabla{\hat F_\ell}({w}_\ell^{(t)})\|_2}$ based only on the local gradients would fail. At the other end, just introducing variables $v_\ell$ without performing a normalized gradient update (i.e. implementing gradient tracking) also fails to give significant speed ups over DGD. See Section \ref{sec:num} for experiments in support of this claim.

%We can formalize the intuition presented above 

\par
We also present a yet improved Algorithm \ref{alg:alg2}, which combines $\mathsf{FDLR}$ with Nesterov Momentum.  The key innovation of  Algorithm \ref{alg:alg2} compared to $\mathsf{FDLR}$ is its step (3), where now the local parameter $w_\ell^{(t)}$ is updated by a weighted average ($z_\ell^{(t+1)}$) of normalized gradients from previous iterations. Similar to our previous remarks regarding $\mathsf{FDLR}$, extending the Nesterov accelarated variant of normalized GD for centralized logistic regression \cite{ji2021fast} to the distributed setting is more subtle as now each agent has access only to local gradients. Our experiments in Section \ref{sec:num} verify the correctness of the proposed implementation of Algorithm \ref{alg:alg2} as it achieves significant speed ups over both DGD and $\mathsf{FDLR}$.

%A Nesterov accelerated variant of normalized GD has already been studied by \cite{ji2021fast} for centralized logistic regression. In distributed case, the problem is more subtle as each agent has access only to local gradients. 
%Algorithm \ref{alg:alg2} unifies $\mathsf{FDLR}$ with a Nesterov Momentum. 
%in step (3) of Algorithm \ref{alg:alg2} the local parameter $w_\ell^{(t)}$ is updated by a weighted average ($z_\ell^{(t+1)}$) of normalized gradients from previous iterations.  

% While we do not prove convergence rates for the proposed algorithms, we study their asymptotic convergence. Notably, the next theorem shows that for a time-decaying step-size $\eta=1/\sqrt{t}$, the solution of $\mathrm{FDRL}$ converges in direction to the solution of centralized gradient descent. The proof is provided in Appendix \ref{sec:alg_proofs}. We defer numerical performance of both Algorithms to the next section.  

% \par
% \begin{theorem}[Asymptotic convergence of $\mathsf{FDLR}$]\label{thm:alg}
% Let the sequence $\{{w}_\ell^{(t)}\}$ be generated by $\mathsf{FDRL}$(Algorithm \ref{alg:alg1}) for a separable dataset and with $\eta_t=O (1/\sqrt{t})$. Then for all $\ell\in [N]$,
% $
% \lim_{t\rightarrow\infty} {{w}_\ell^{(t)}}/{\|{w}_\ell^{(t)}\|} =  {{w}_{_{\rm MM}}}/{\|{w}_{_{\rm MM}}\|},
% $
% where ${w}_{_{\rm MM}}$ is the solution to max-margin problem. 
% \end{theorem}

\section{NUMERICAL EXPERIMENTS}\label{sec:num}
In this section, we present numerical experiments to verify our theoretical results and demonstrate the benefits of our proposed algorithms. We begin with a numerical study of the performance of $\mathsf{FDLR}$. 
\subsection{Experiments on $\mathsf{FDLR}$}

In Fig. \ref{fig:fig1}(Left), we compare the performance of $\mathsf{FDLR}$ and its momentum variant  to DGD and gradient tracking (GT) for exponential loss with signed measurements (i.e., $y = \sign (a^\top {w}^\star)$ for samples $a$, labels $y$ and the true vector of regressors ${w}^\star$) with $n = 100$, $d = 25$. 
The underlying graph is selected as an Erdos-Rènyi graph with $N=50$ agents and connectivity probability $p_c= 0.3$. On the $y-$axis, directional convergece represents the distance of normalized ${w}_\ell^{(t)}$ to the normalized final solution for agent $\ell=1$, i.e., $\|\frac{{w}_1^{(t)}}{\| {w}_1^{(t)}\|} - \frac{{w}_{_{\rm MM}}}{\|{w}_{_{\rm MM}}\|}\|$ (see Theorem \ref{thm:alg}). The hyper-parameters $\eta_t,\gamma_t$ are fine-tuned for each algorithm to represent the best of each algorithm and the final values are $\eta_t = 0.1, 0.05, 0.5$ and $0.2$ for Distributed GD, GT, Alg. \ref{alg:alg1} and Alg. \ref{alg:alg2}, respectively and $\gamma_t = 0.8$ for Alg. \ref{alg:alg2}. 
Our algorithms significantly outperform the well-known distributed learning algorithms in directional convergence to the final solution. Regardless, in this case we noticed that the gain obtained by including the momentum is small. 
In Fig. \ref{fig:fig1}(Right), we consider a binary classification task on a real-world dataset (two classes of the UCI WINE dataset \cite{winedata}) where $d=13$ and $n=107$. We compare the performance of $\mathsf{FDLR}$ (blue line) and its momentum variant (red line) with DGD and DGT on an Erdos-Rènyi graph with $N=10$ and $p_c=0.4$. The hyper-parameters are fine-tuned to $\eta_t=12,1,0.9,2$ for DGD, DGT, Alg. \ref{alg:alg1} and Alg. \ref{alg:alg2} respectively, and $\gamma_t = 0.88$ for Alg. \ref{alg:alg2}. Notably, while Alg. \ref{alg:alg1} significantly outperforms both DGD and DGT, the benefits of adding the momentum are also significant in this case as Alg. \ref{alg:alg2} demonstrates a faster rate of convergence than Alg. \ref{alg:alg1}. 

The two plots in Fig. \ref{fig:fig2} (Left) illustrate the train and test errors of DGD/DGT and our proposed algorithms for the same setting as Fig. \ref{fig:fig1}(Left) with $n=800$ and $d=50$. Here the hyper-parameters are fine-tuned to be $\eta_t = 0.01, 0.01, 0.4, 0.5$ for DGD,DGT, Alg. \ref{alg:alg1} and Alg. \ref{alg:alg2}, respectively and $\gamma_t = 0.5$ for Alg. \ref{alg:alg2}. Fig. \ref{fig:fig2}(Right) shows the training and test losses. Here, we use the same dataset with $N=10,p_c=0.4$ and an exponential loss. The hyper-parameters are fine-tuned to $\eta_t = 0.01, 0.012, 0.4, 0.6$ and $\gamma_t=0.9$. 
Both of our algorithms outperform the commonly used DGD and DT, in both training error and test error performance. Also, the gains of adding the momentum are significant, since $\mathsf{FDLR}$ with Nesterov momentum (Algorithm \ref{alg:alg2}) reaches an approximation of its final test accuracy in $~50$ iterations, while the same happens for $\mathsf{FDLR}$ with approximately $300$ iterations. 
\begin{figure*}[t!]
\includegraphics[width=5.8cm, height=4cm]{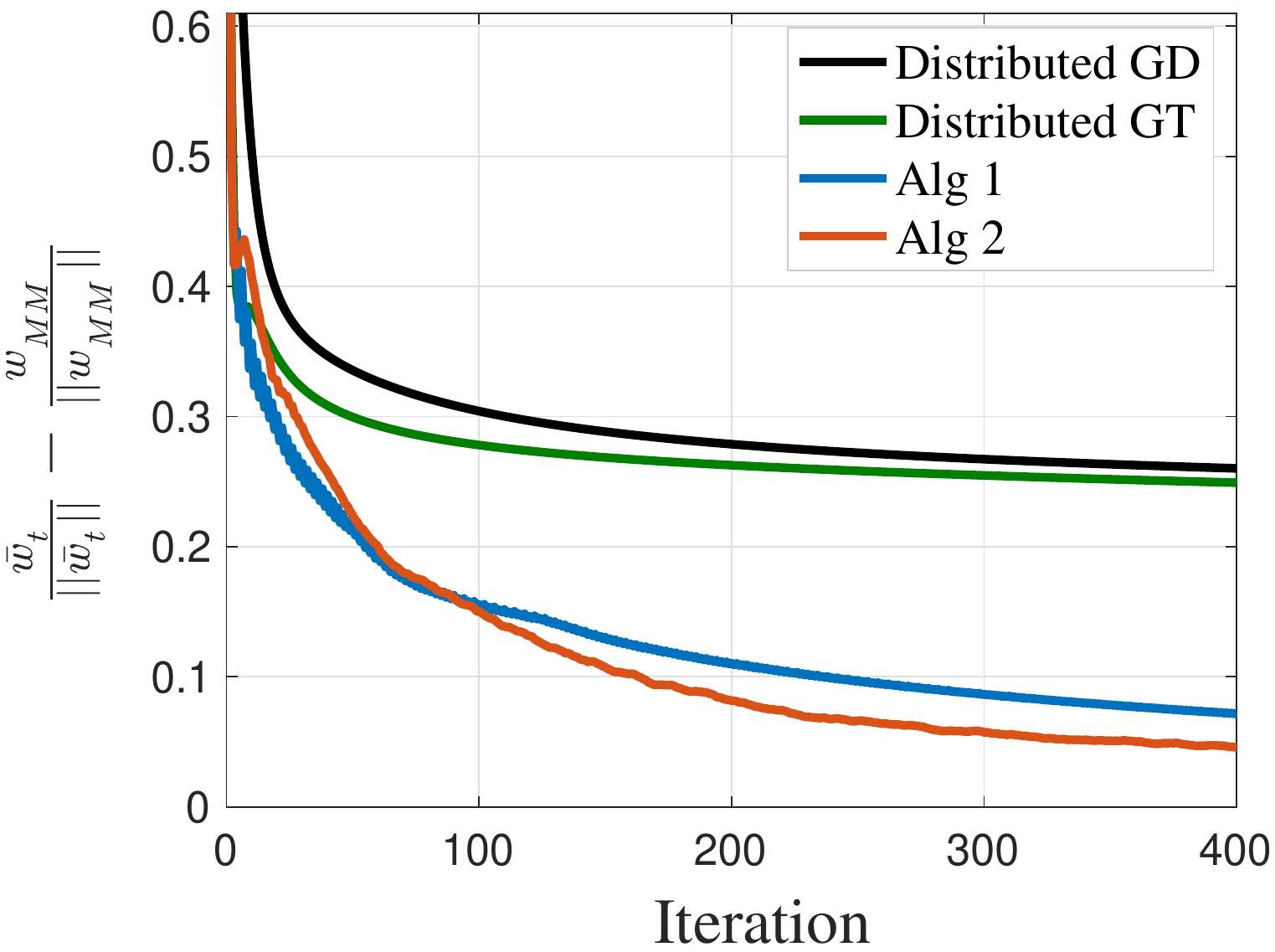}
\centering
\;\;\;\;\;\;\;~~~~\includegraphics[width=5.8cm, height=4cm]{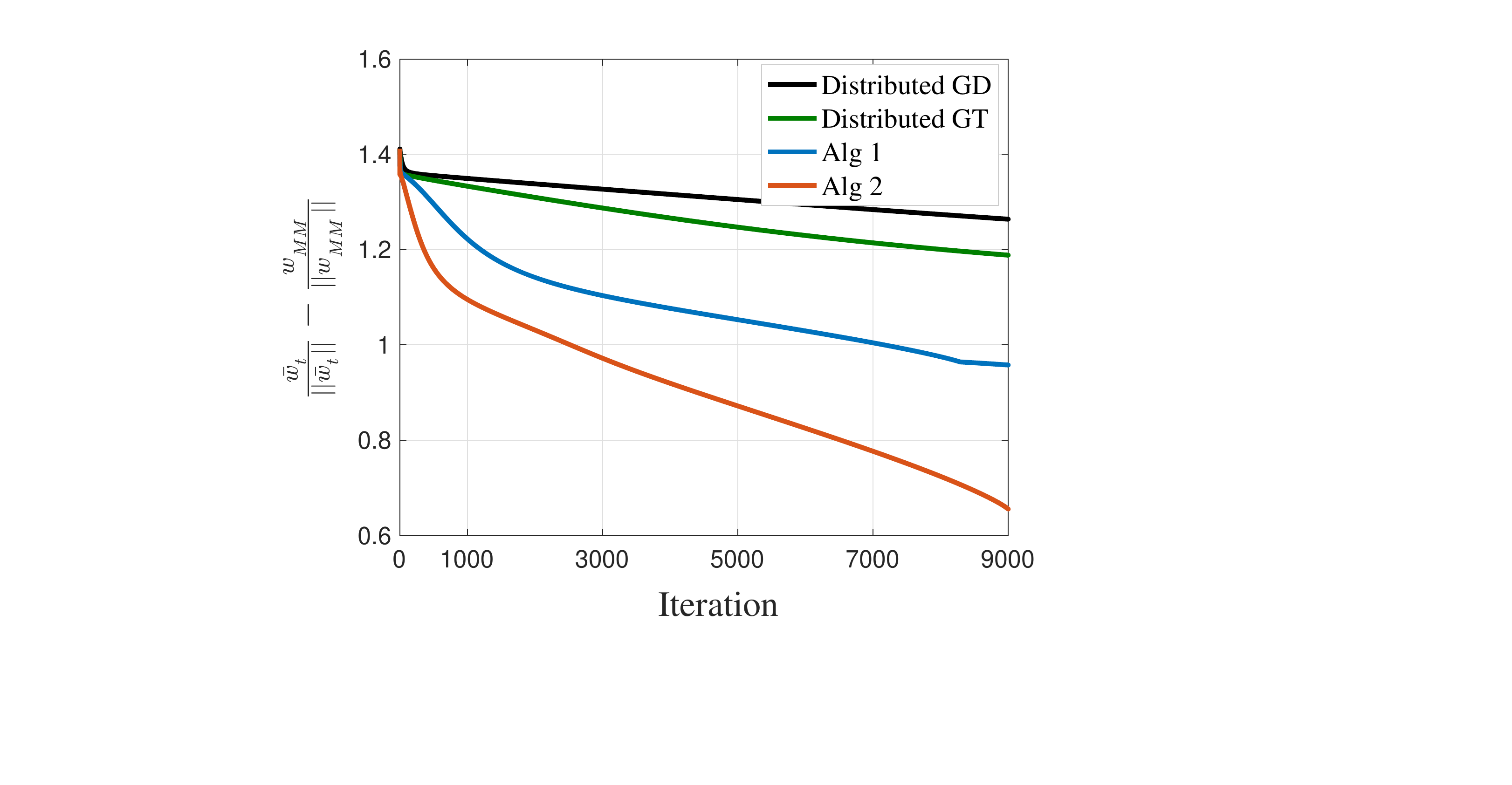}
\centering
\caption{Directional parameter convergence of our proposed Algorithms \ref{alg:alg1}-\ref{alg:alg2} compared to the vanilla distributed gradient descent and gradient tracking algorithms on (Left) synthetic data $y=\sign(a^\top w^\star)$  and (Right) on two classes from the UCI WINE dataset.}
\label{fig:fig1}
\vspace{.1in}
\end{figure*}
\begin{figure*}[t!]
\includegraphics[width=4.1cm,height=3.3cm]{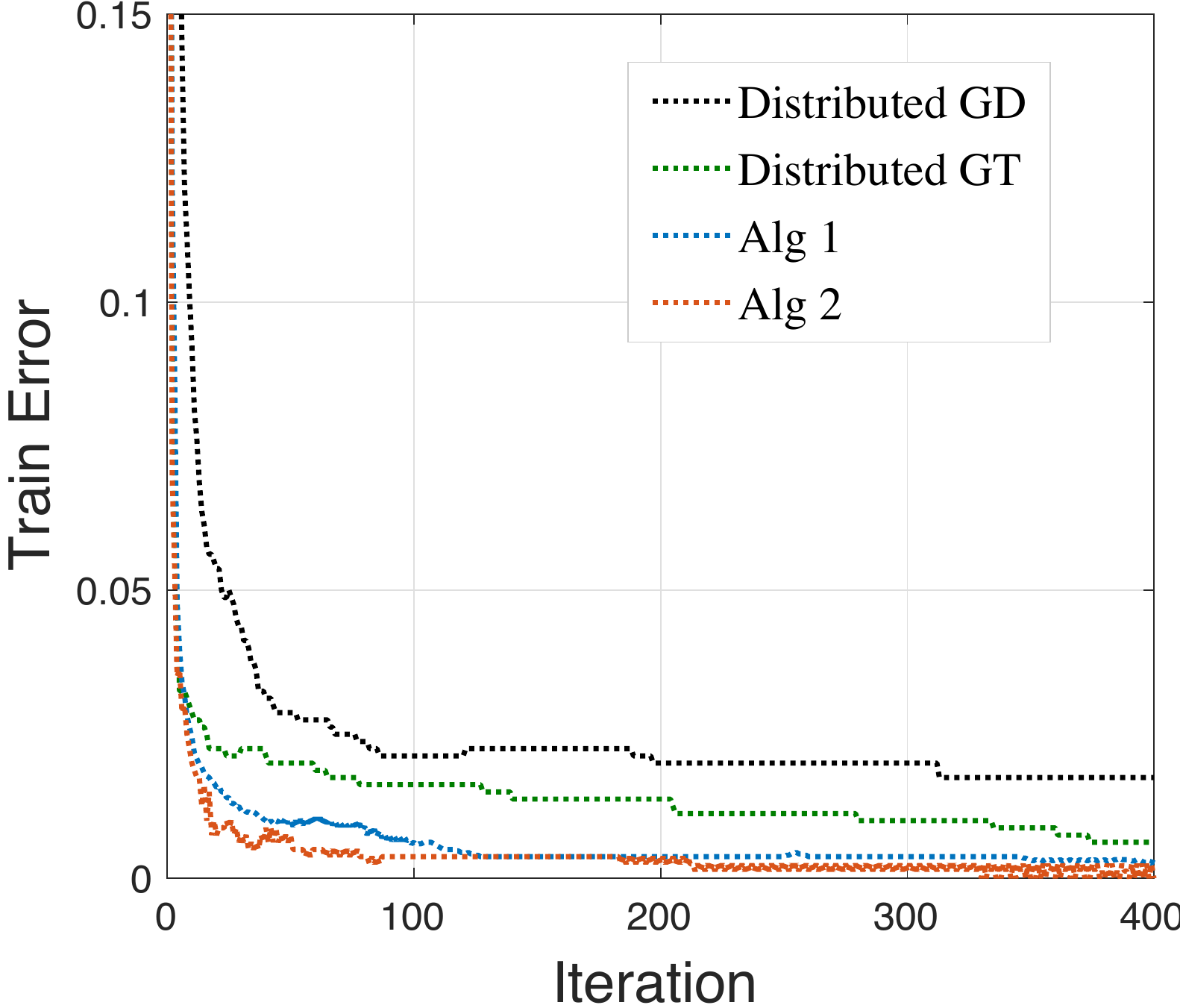}
\centering
\includegraphics[width=4.1cm,height=3.3cm]{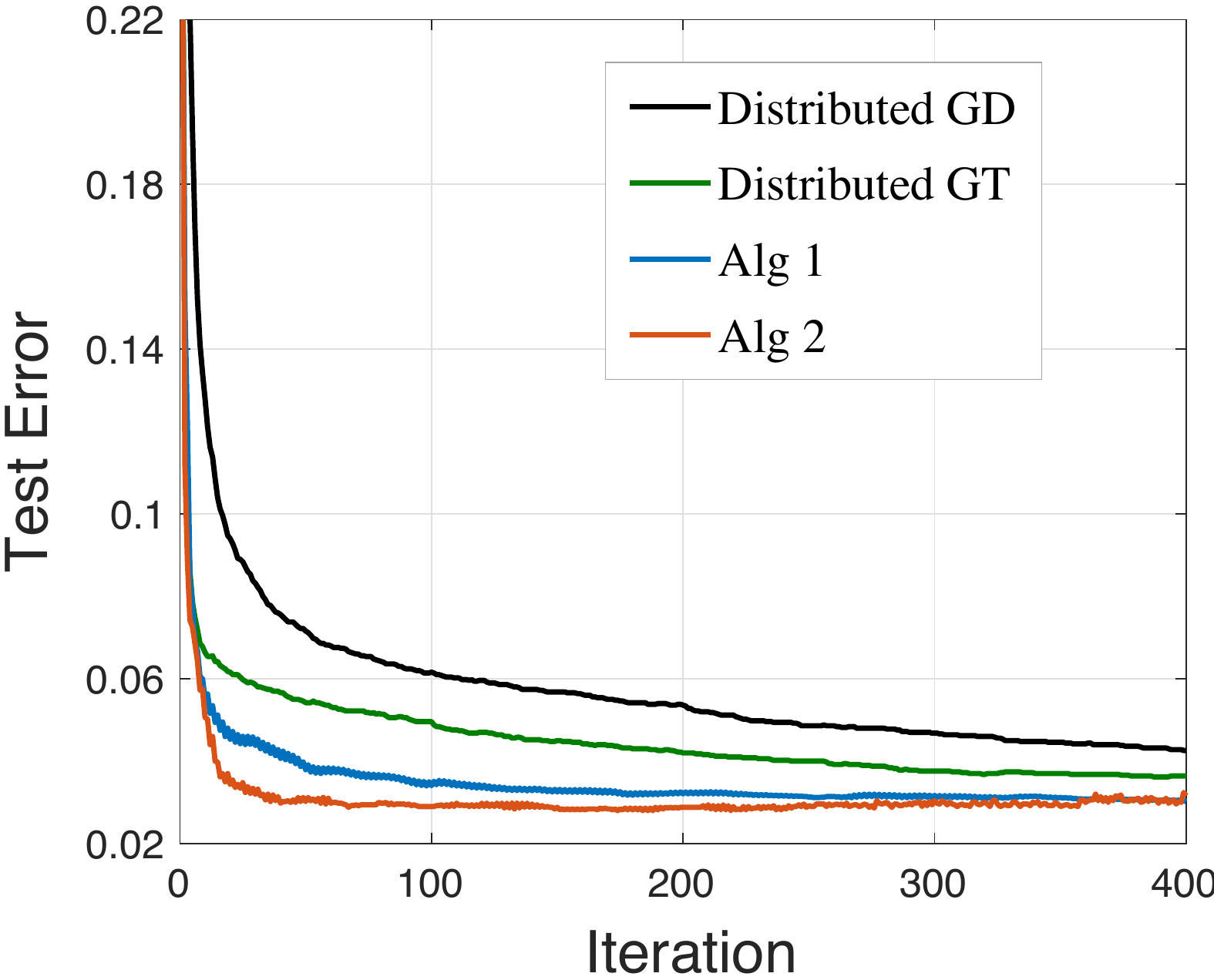}
\;\;\centering
\includegraphics[width=4.1cm,height=3.3cm]{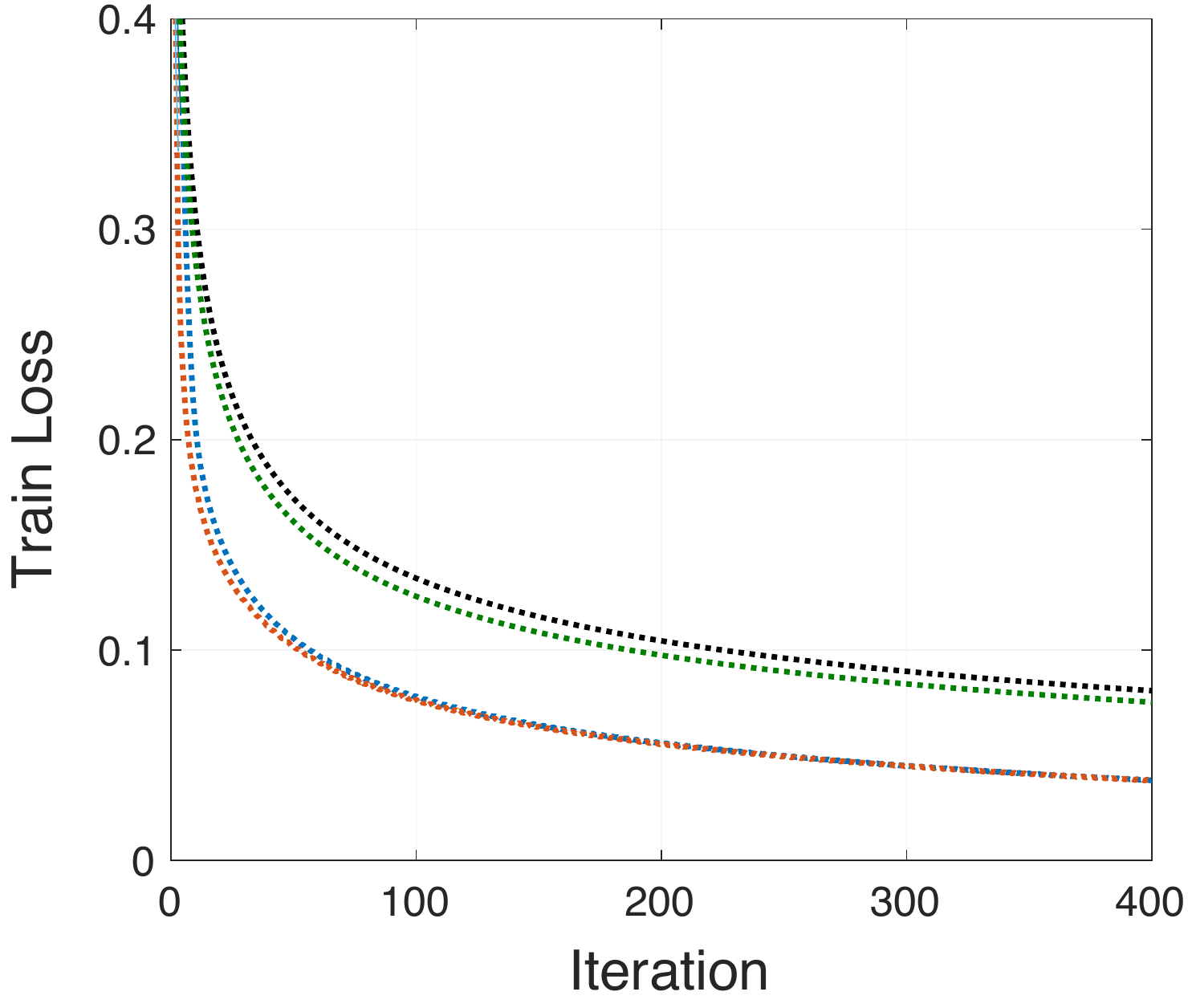}
\includegraphics[width=4.1cm,height=3.3cm]{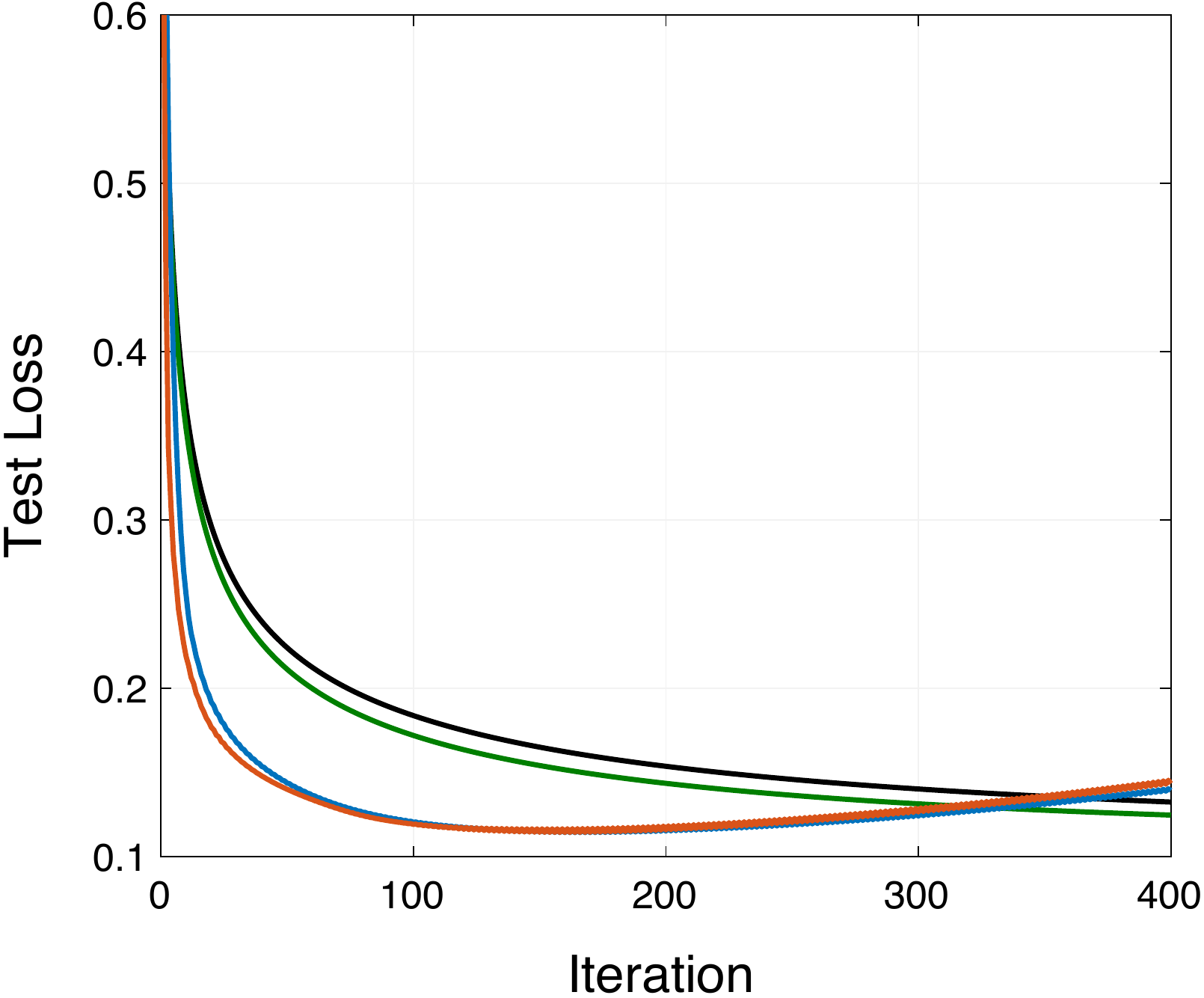}
\centering
\caption{Training/test misclassification errors and train/test losses for our proposed algorithms compared to the decentralized gradient descent and gradient tracking algorithms on synthetic data $y=\sign(a^\top w^\star)$ with $n=800,d=50$.}
\label{fig:fig2}
\vspace{0.2in}
\end{figure*}
An interesting phenomenon in Fig. \ref{fig:fig2} (see also Fig. \ref{fig:3}(Right)) is the behavior of test loss: while during the starting phase the test loss is monotonically decreasing, after sufficient iterations the test loss starts increasing. This behavior of test loss is indeed captured by the bounds on the test loss of DGD in Theorems \ref{thm:testloss_cvx}-\ref{lem:exp_dsc} and Remarks \ref{rem:logr}-\ref{rem:improved}. In particular, the increase in test loss is observed in the bound for the test loss $O(\frac{(\log T )^2}{\sqrt{T}}+\frac{(\log T)^2}{n})$ in Remark \ref{rem:logr}, where the presence of the term $\frac{(\log T)^2}{n}$ suggests that the bound after sufficient iterations starts to slowly increase. See also Remark \ref{rem:Log}. 

\subsection{Experiments on convergence of DGD}\label{sec:furtherex}
\begin{figure*}[t!]
\includegraphics[width=5cm,height=3.7cm]{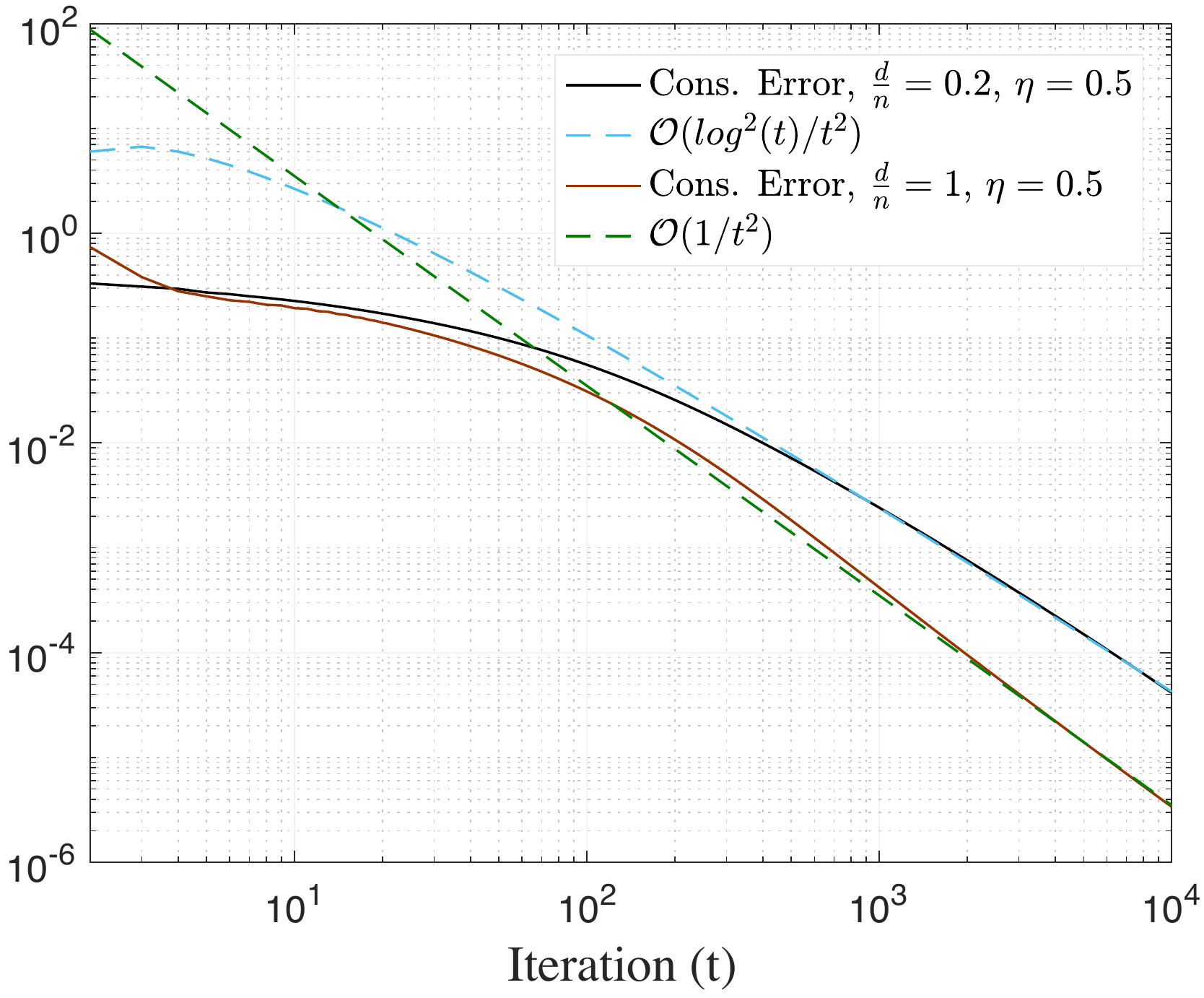}
\;\centering
%\caption{Misclassification Train Errors(\%) for our proposed algorithms compared to the vanilla decentralized Gradient-Descent algorithm in logistic regression with Signed measurements.}
\; \includegraphics[width=4.8cm,height=3.7cm]{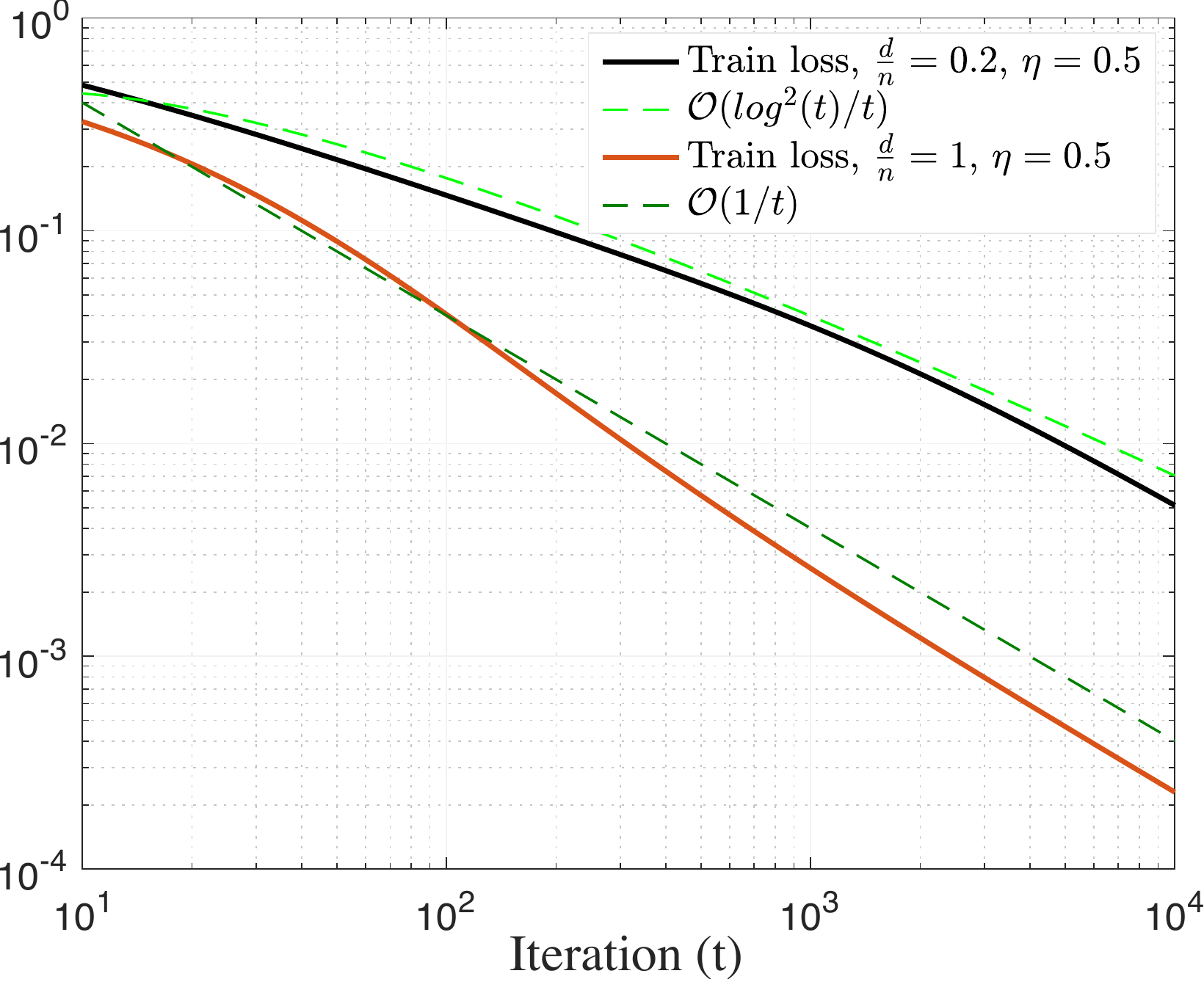}
\centering
\;\;\includegraphics[width=5.2cm,height=3.7cm]{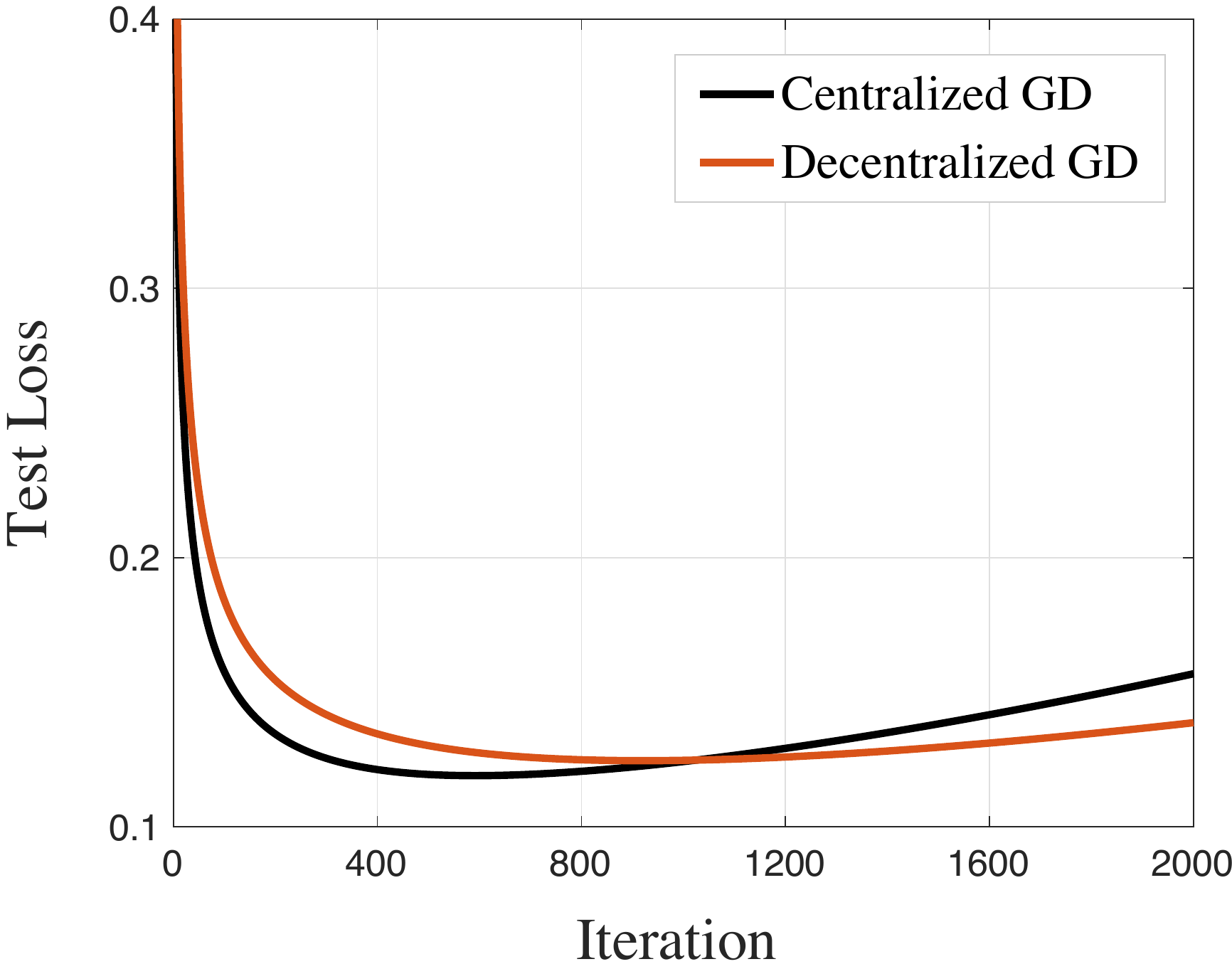}
\caption{Consensus error, train loss and test loss for DGD with exponential loss and linearly separable data. Left and middle plots verify the rates $\tilde{O}(1/t^2)$ and $\tilde{O}(1/t)$ (Theorem \ref{lem:exp_dsc} and Remark \ref{rem:improved}) for \emph{consensus error} and \emph{training loss} of DGD. 
%with logistic regression, respectively. 
Right plot shows test loss for DGD and GD show approximately similar convergence behavior under separable data.}
\label{fig:3}
\end{figure*}

Next, we investigate the convergence behavior of DGD for the train loss and the consensus error. We consider the same network topology, mixing matrix and data setup as in the last figure. Left and middle plots in Fig. \ref{fig:3} show the consensus error $\frac{1}{N}\|W^{(t)}-\bar W^{(t)}\|_F^2$ and the train loss $\hat F(\bar w^{(t)})$ in solid lines, for two over-parameterization ratios $d/n$. We recall that $d$ and $n$ represent the dimension of the ambient space and the dataset size, respectively.
 The dashed lines help to show for each solid line, the approximate rate of convergence after sufficient number of iterations. Notably, we observe that the convergence rates on the consensus error ($\tilde{O}(1/t^2)$) and on the train error ($\tilde{O}(1/t)$) stated in Remark \ref{rem:improved} are attained in both cases (recall that $\tilde{O}(\cdot)$ hides logarithmic factors). 
Fig. \ref{fig:3} (Right) depicts the Test loss of DGD for $d/n=0.05$. For comparison, the corresponding curve for centralized GD is also shown. Here step-sizes are fine-tuned to represent the best of each algorithm. In agreement with our findings in Remark \ref{rem:improved}, we observe approximately similar behavior between the convergence behavior of two algorithms. 
%As a final remark, 
As before, the slight increase in the curves of test loss are due to the logarithmic factors in test loss upperbouds. 
%%%%%%%%%%%%%%%%%%%%%%%%%%%%%%%%%%%%%%%%%%%%%%%%%%%%%%%%%%%%%%%%%%%%%%%%%%%%%%%%%%%%%%%%%%%%%%%%%%%%%%%%%%%%%%%%%%%%%%%%%%%%%%%%%%%%%%%%%%%%%%%%%%%%%%%%%%%%%%%%%%%%%%%%%%%%%%%%%%%%%%%%
\section{CONCLUSIONS}\label{sec:conc}
We studied the behavior of train loss and test loss of decentralized gradient descent (DGD) methods when training dataset is separable. To the best of our knowledge, this yields the first rigorous guarantees for the generalization error of DGD in such a setting. For the same setting, we also proposed fast algorithms and empirically verified that they accelarate both training and test accuracy. We believe our work opens several directions, with perhaps the 
%Finally, we note that this line of work can be extended in several directions, with the 
most exciting one being the analysis of non-convex objectives. We are also interested in extending our results to other distributed settings such as federated learning \cite{kairouz2021advances} and Gradient Tracking e.g., \cite{nedic2017achieving}.
\section*{Acknowledgements}
This work is partially supported by NSF Grant CCF-2009030.
%%%%%%%%%%%%%%%%%%%%%%%%%%%%%%%%%%%%%%%%%%%%%%%%%%%%%%%%
\bibliographystyle{apalike}
\bibliography{main_adv}
\clearpage

\clearpage
\appendix
\onecolumn
\section*{APPENDIX}
In this section, we present the proofs of all theorems and lemmas stated in the main body. We organize the appendix as follows, 
\par
The formal statement and proof of Lemma \ref{lem:main} are included in Appendix \ref{sec:lemma1}. 
\par
The proofs for Section \ref{sec:convex} are included in Appendix \ref{sec:convex_proof}. 
\par
The proofs for Section \ref{sec:LIC} are included in Appendix \ref{sec:LIC_proofs}. 
\par
The proofs for Section \ref{sec:polyak} are included in Appendix \ref{sec:polyak_proofs}. 
\par
The proof of Theorem \ref{thm:alg} is included in Appendix \ref{sec:alg_proofs}.
\par
Auxiliary results on our assumptions are included in Appendix \ref{sec:aux}.
\par
Finally, we conduct complementary experiments in Appendix \ref{sec:appG}. 
\subsection*{Notation}
Throughout the appendix we use the following notations:
\begin{align*}
&\bar w:= \frac{1}{N} \sum_{\ell=1}^N w_\ell,~~~\bar W := [\bar w,\bar w,\cdots,\bar w]^\top\in\R^{N\times d},\\[3pt]
&W:= [w_1,\cdots,w_N]^\top\in \R^{N\times d},\\[3pt]
&\hat F(W):=\frac{1}{N}\sum_{\ell=1}^N \hat F_\ell(w_\ell),\\[3pt]
&\nabla\hat F(w):= \frac{1}{n}\sum_{i=1}^n \nabla f(w,x_i),\\[3pt]
&\nabla \hat F(W) := [\nabla \hat F_1 (w_1),\nabla\hat F_2(w_2),\cdots,\nabla\hat F_N(w_N)]^\top\in \R^{N\times d},\\[3pt]
&\bar\nabla \hat F(W):= \frac{1}{N} \sum_{\ell=1}^N \nabla \hat F_\ell(w_\ell),\\[3pt]
&\nabla\hat F_\ell({w_\ell}):= \frac{N}{n} \sum_{x_j\in\mathcal{S}_\ell} \nabla f(w_\ell,x_j).
\end{align*}
where recall that $d$ is the dimension of ambient space, $n$ is the total sample size, $N$ is the number of agents and each agent has access to $n/N$ samples. 

\section{Proof of Lemma \ref{lem:main}}\label{sec:lemma1}
%\begin{lemma}
%\bea
%F(\bar w^{(T)})&\le  \hat F( \bar w^{(T)}) +2G T^{1-\alpha} c/n (\sum_{t=1}^T\eta \hat F(\bar {w}^{(t)}))^{\alpha} +4GL\sum_{t=1}^T \eta\|W^{(t)}-\bar W^{(t)}\|_F
%\eea
%\end{lemma}
\begin{lemma}[Formal statement of Lemma \ref{lem:main}]\label{lem:main-formal}
Consider the iterates of decentralized gradient descent in Eq.\eqref{eq:dec_main} with a fixed positive step-size $\eta\le \frac{2}{L}$. Let Assumptions \ref{ass:mixing}-\ref{ass:self-bounded} hold. Then for the test loss $F$ at iteration $T\ge1$, it holds that
\bea\label{eq:lem-formal}
\E\left[F(\bar w^{(T)})\right] \le &\;\; 4 \E\left[\hat F( \bar w^{(T)})\right] +   \frac{9L^2 c^2\eta^2T^{2}}{n^{3-2\alpha}}\E\left[\left(\frac{1}{T}\sum_{t=1}^{T}\hat F(\bar {w}^{(t)})\right)^{2\alpha}\right] \\
&+\frac{9L^4\eta^2}{N} \E\left[\left(\sum_{t=1}^{T}\|W^{(t)}-\bar W^{(t)}\|_F\right)^2\right]+\frac{9L^4\eta^2}{N} \frac{1}{n}\sum_{i=1}^n \E\left[\left(\sum_{t=1}^{T}\|W_{\neg i}^{(t)}-\bar W_{\neg i}^{(t)}\|_F\right)^2\right].\nn
\eea

where the expectation is over training samples and $W_{\neg i}^{(t)},\bar{W}_{\neg i}^{(t)}$ denote the parameter matrix and averaged parameter matrix at iteration $t$ for the DGD algorithm when the $i$-th data sample is left out. 
\end{lemma}
\begin{proof}
The proof relies on algorithmic stability\cite{bousquet2002stability,hardt2016train}. Specifically, we build on the framework introduced by \cite{lei2020fine} (and also used recently by \cite{schliserman2022stability}). Unlike these works, our analysis is for decentralized gradient descent.

We define $w_{\ell,\neg i}^{(t)}$ as the parameter of agent $\ell$ resulting from decentralized gradient descent at iteration $t$, when the $i^{\rm th}$ training sample $i\le n$ is left out during training. We emphasize that the $i^{th}$ sample may or may not belong to the dataset of agent $\ell$. 

We define $\bar w_{\neg i}^{(t)}\in \R^d$
\begin{align*}
\bar w_{\neg i}^{(t)}&:=\frac{1}{N} \sum_{j=1}^N w_{j,\neg i}^{(t)},
\end{align*}
as the average of all agents' parameters at iteration $t$, when the $i^{th}$ sample is left out of the algorithm.  
Thus, the parameter matrices $W_{\neg i}^{(t)},\bar W_{\neg i}^{(t)}\in \R^{N\times d}$ are defined as follows,
\begin{align*}
W_{\neg i}^{(t)}&:= [w_{1,\neg i}^{(t)},w_{2.\neg i}^{(t)},\cdots,w_{N,\neg i}^{(t)}],\\
\bar W_{\neg i}^{(t)}&:= [\bar w_{\neg i} ^{(t)},\bar w_{\neg i} ^{(t)},\cdots,\bar w_{\neg i} ^{(t)}].
\end{align*}
The first step in the proof is to bound the term $\frac{1}{n}\sum_{i=1}^n \|\bar w^{(t)} - \bar w_{\neg i}^{(t)}\|^2$. By definition of DGD in Eq.\eqref{eq:dec_main}, we have the following update rule for the averaged parameter,
\bea\nn
\bar w^{(t+1)}=\bar w^{(t)} - \eta  \bar \nabla \hat F (W^{(t)}).
\eea
Analogously,
\bea\nn
\bar w_{\neg i}^{(t+1)}=\bar w_{\neg i}^{(t)} - \eta  \bar \nabla \hat F (W_{\neg i}^{(t)})=\bar w_{\neg i}^{(t)} - \frac{\eta}{n}\sum_{\ell=1}^N\sum_{x_j\in S_\ell, x_j\neq x_i} \nabla f (w_{\ell,\neg i}^{(t)},x_j).
\eea

Thus by adding and subtracting $\bar \nabla \hat F (\bar W^{(t)})$ and $\bar \nabla \hat F (\bar W_{\neg i}^{(t)})$, we have
\begin{align*}
&\|\bar w^{(t+1)} - \bar w_{\neg i}^{(t+1)}\| \\
&~~~~= \Big\| \bar w^{(t)} - \eta  \bar \nabla \hat F (W^{(t)})- (\bar w_{\neg i}^{(t)} - \eta \bar \nabla \hat F(W_{\neg i}^{(t)}))\Big\|\\[3pt]
&~~~~ = \Big\| \bar w^{(t)} - \eta  \bar \nabla \hat F (\bar W^{(t)}) +  \eta (\bar \nabla \hat F ( W^{(t)}) - \bar \nabla \hat F (\bar W^{(t)}) ) \\
& \hspace{1.5in}- (\bar w_{\neg i}^{(t)} - \eta \nabla \hat F(\bar W_{\neg i}^{(t)})) + \eta \bar \nabla \hat F(W_{\neg i}^{(t)}) - \eta  \nabla \hat F(\bar W_{\neg i}^{(t)})\Big\| \\[3pt]
& ~~~~\le \Big\| \bar w^{(t)} - \eta  \bar \nabla \hat F (\bar W^{(t)})   - (\bar w_{\neg i}^{(t)} - \eta \bar \nabla \hat F(\bar W_{\neg i}^{(t)}))\Big\| \\
&\hspace{1.5in}+ \eta \Big\|\bar \nabla \hat F(W_{\neg i}^{(t)}) -  \nabla \hat F(\bar W_{\neg i}^{(t)})\Big\| + \eta \Big\|\bar \nabla \hat F ( W^{(t)}) -  \nabla \hat F (\bar W^{(t)}) \Big\|.
\end{align*}
For the last term, using smoothness, we can write
\begin{align*}
 \Big\|\bar \nabla \hat F ( W^{(t)}) -  \nabla \hat F (\bar W^{(t)}) \Big\| &= \frac{1}{N}\Big\|\sum_{\ell=1}^N \nabla \hat F_{\ell}(w_\ell^{(t)})-\nabla \hat F_\ell(\bar w^{(t)})\Big\| \\
 &\le \frac{1}{N}\sum_{\ell=1}^N \Big\| \nabla \hat F_\ell (w_{\ell}^{(t)})-\nabla \hat F_\ell(\bar w^{(t)})\Big\|\\
 &\le \frac{L}{N} \sum_{\ell=1}^N \| w_\ell^{(t)} - \bar w^{(t)}\|\\
 &\le  \frac{L}{\sqrt{N}}  (\sum_{\ell=1}^N \| w_\ell^{(t)} - \bar w^{(t)}\|^2)^{1/2}= \frac{L}{\sqrt{N}} \|W^{(t)}-\bar W^{(t)}\|_F.
\end{align*}
The second term is upper-bounded similarly. Using these bounds, splitting the gradient $\nabla \hat F (\bar W^{(t)})=\frac{1}{n}\sum_{i^\prime\ne i} f(\bar w^{(t)},x_{i^\prime}) + \frac{1}{n} f(\bar w^{(t)},x_i)$, using smoothness and convexity $\|w+\eta\nabla f(w)-v-\eta\nabla f(v)\|\le \|w-v\|$ for $\eta\le2/L$ \cite{nesterov2003introductory} and employing Assumption \ref{ass:self-bounded} we can write
\begin{align*}
\|\bar w^{(t+1)} - \bar w_{\neg i}^{(t+1)}\| 
& \le \| \bar w^{(t)} - \eta   \nabla \hat F (\bar W^{(t)})   - (\bar w_{\neg i}^{(t)} - \eta  \nabla \hat F(\bar W_{\neg i}^{(t)}))\| \\
&\hspace{1.5in}+ \frac{\eta L}{\sqrt{N}} \|W_{\neg i}^{(t)}-\bar W_{\neg i}^{(t)}\|_F  + \frac{\eta L}{\sqrt{N}}\|W^{(t)}-\bar W^{(t)}\|_F   \\
&\le \frac{1}{n} \sum_{i^\prime\neq i}\|\bar {w}^{(t)} - \eta \nabla f(\bar {w}^{(t)},x_{i^\prime}) - \bar w_{\neg i}^{(t)} + \eta \nabla f(\bar w_{\neg i}^{(t)},x_{i^\prime})\| + \frac{1}{n} \|\bar {w}^{(t)} - \eta \nabla f(\bar {w}^{(t)},x_i)-\bar w_{\neg i}^{(t)}\|\\
& \hspace{1.5in}+ \frac{\eta L}{\sqrt{N}} \|W_{\neg i}^{(t)}-\bar W_{\neg i}^{(t)}\|_F + \frac{\eta L}{\sqrt{N}}\|W^{(t)}-\bar W^{(t)}\|_F \\
&\le \| \bar {w}^{(t)} - \bar w_{\neg i}^{(t)}\| + \frac{\eta}{n} \|\nabla f(\bar {w}^{(t)},x_i)\|+\frac{\eta L}{\sqrt{N}}\|W_{\neg i}^{(t)}-\bar W_{\neg i}^{(t)}\|_F + \frac{\eta L}{\sqrt{N}}\|W^{(t)}-\bar W^{(t)}\|_F \\
&\le \| \bar {w}^{(t)} - \bar w_{\neg i}^{(t)}\| + \frac{c \eta}{n} \; (f(\bar {w}^{(t)},x_i))^{\alpha}+\frac{\eta L}{\sqrt{N}}\|W_{\neg i}^{(t)}-\bar W_{\neg i}^{(t)}\|_F + \frac{\eta L}{\sqrt{N}}\|W^{(t)}-\bar W^{(t)}\|_F.
\end{align*}

By summing over $t\in[T]$,
\begin{align*}
\|\bar w^{(T+1)} - \bar w_{\neg i}^{(T+1)}\| &\le \frac{c\eta}{n}\sum_{t=1}^{T} (f(\bar {w}^{(t)},x_i))^{\alpha} +\frac{\eta L}{\sqrt{N}} \sum_{t=1}^{T}\|W_{\neg i}^{(t)}-\bar W_{\neg i}^{(t)}\|_F 
+\frac{\eta L}{\sqrt{N}}\sum_{t=1}^{T}\|W^{(t)}-\bar W^{(t)}\|_F.
\end{align*}
We define for the ease of notation the following two consensus terms, 
\begin{align}\label{eq:consensus_terms}
e^{(T)} &:= \Big(\sum_{t=1}^{T}\|W^{(t)}-\bar W^{(t)}\|_F\Big)^2\qquad\text{and}\qquad
e_{\neg{i}}^{(T)} :=  \Big(\sum_{t=1}^{T}\|W_{\neg i}^{(t)}-\bar W_{\neg i}^{(t)}\|_F\Big)^2.
\end{align}
Thus the bound for the squared term can be written as follows
\begin{align*}
\|\bar w^{(T+1)} - \bar w_{\neg i}^{(T+1)}\|^2 &\le \frac{3c^2 \eta^2}{n^2}\left(\sum_{t=1}^{T} \left(f(\bar {w}^{(t)},x_i)\right)^{\alpha}\right)^2 +\frac{3\eta^2 L^2}{N} e^{(T)} +\frac{3\eta^2 L^2}{N} e_{\neg{i}}^{(T)}.
\end{align*}

By averaging over $i\in[n]$ and noting that $\alpha\in [1/2,1]$ so that $x^\alpha$ is concave, we conclude that 
\begin{align*}
\frac{1}{n}\sum_{i=1}^n \|\bar w^{(T+1)} - \bar  w^{(T+1)}_{\neg i}\|^2 
&\le\frac{3c^2\eta^2}{n^3}\sum_{i=1}^n\left(\sum_{t=1}^{T} \left(f(\bar {w}^{(t)},x_i)\right)^{\alpha}\right)^2+\frac{3\eta^2 L^2}{Nn} \sum_{i=1}^n e_{\neg{i}}^{(T)} +\frac{3\eta^2 L^2}{N}e^{(T)}
\\
&\le \frac{3c^2\eta^2T^2}{n^3}\sum_{i=1}^n\left(\frac{1}{T}\sum_{t=1}^{T}f(\bar {w}^{(t)},x_i)\right)^{2\alpha}+\frac{3\eta^2 L^2}{Nn} \sum_{i=1}^n e_{\neg{i}}^{(T)} +\frac{3\eta^2 L^2}{N}e^{(T)}\\
&\le \frac{3c^2\eta^2T^{2(1-\alpha)}}{n^{3-2\alpha}}\left(\sum_{t=1}^{T}\hat F(\bar {w}^{(t)})\right)^{2\alpha}+\frac{3\eta^2 L^2}{Nn} \sum_{i=1}^n e_{\neg{i}}^{(T)} +\frac{3\eta^2 L^2}{N}e^{(T)}.
\end{align*}
Thus we have for iteration $T$:
\begin{align}\label{eq:stabilityn}
\frac{1}{n} \sum_{i=1}^n \|\bar w^{(T)} - \bar w_{\neg i}^{(T)}\|^2 \le  \frac{3c^2\eta^2T^{2}}{n^{3-2\alpha}}\left(\frac{1}{T}\sum_{t=1}^{T}\hat F(\bar {w}^{(t)})\right)^{2\alpha}+\frac{3\eta^2 L^2}{Nn} \sum_{i=1}^n e_{\neg{i}}^{(T)} +\frac{3\eta^2 L^2}{N}e^{(T)}.
\end{align}
Next we use \cite[Lemma 7]{schliserman2022stability} (see also \cite[Theorem 2]{lei2020fine}), which states that for the $L$-smooth loss $f$, the test error of the output $w$ of an algorithm taking as input a dataset $(x_1,\ldots,x_n)$ size $n$, satisfies the following,
\begin{align*}
\E[F(w)]\le 4 \E[\hat F( w)] + \frac{3L^2}{n} \sum_{i=1}^n \E[\|w-w_{\neg i}\|^2],
\end{align*}
where expectations are taken over the training set $(x_1,x_2,\cdots,x_n)$. 
We replace $w$ with $\bar w^{(T)}$ and by using \eqref{eq:stabilityn} (which we can do because it holds true for all datasets since Assumptions \ref{ass:mixing}-\ref{ass:self-bounded} hold for every sample $x$ in the distribution), 
\begin{align*}
\E[F(\bar w^{(T)})]&\le 4 \E[\hat F( \bar w^{(T)})] + \frac{3L^2}{n} \sum_{i=1}^n \E[\| \bar w^{(T)} - \bar w_{\neg i}^{(T)}\|^2]\\
&\le 4 \E[\hat F( \bar w^{(T)})] +   \frac{9L^2 c^2\eta^2T^{2}}{n^{3-2\alpha}}\E[(\frac{1}{T}\sum_{t=1}^{T}\hat F(\bar {w}^{(t)}))^{2\alpha}]\\
&\hspace{1in}+\frac{9L^4\eta^2}{Nn} \sum_{i=1}^n \E[e_{\neg{i}}^{(T)}] +\frac{9L^4\eta^2}{N} \E[e^{(T)}].
\end{align*}
This leads to \eqref{eq:lem-formal} and completes the proof. 
\end{proof}

Finally, we explain the informal version of the lemma presented in the main body (Lemma \ref{lem:main}). Compared to the bound in Eq. \eqref{eq:lem-formal}, the informal Lemma \ref{lem:main} combines the consensus-error term $e^{(T)}$ with the average leave-one-out consensus-error term $\frac{1}{n}\sum_{i\in[n]}e_{\neg{i}}^{(T)}$ (recall the definitions in \eqref{eq:consensus_terms}). It is convenient doing that for the following reason. To apply Lemma \ref{lem:main-formal}, we need upper bounds on $e^{(T)}$ and $e_{\neg{i}}^{(T)}$ (for specific assumptions on the function class that is optimized). We do this in the section that follows. It turns out that the bounds we obtain for the consensus-error term $e^{(T)}$ also holds for the leave-one-out consensus error terms $e_{\neg{i}}^{(T)}, i\leq[n].$ The reason for that is that our bounds are not affected by the sample-size, but rather they depend crucially only on the smoothness parameter of the train loss. It is easy to see that the smoothness parameter of the leave-one-out train loss $\frac{1}{n}\sum_{i'\neq i}f(w,x_{i'})$ is upper-bounded by the smoothness parameter of  $\frac{1}{n}\sum_{i'}f(w,x_{i'})$. See Remark \ref{rem:loo_consensus} for more details.

% Finally, we note that we can combine the two terms for the the consensus error $e_{\neg{i}}^{(T)}, e^{(T)}$, as the consensus error is not affected (order-wise) by removing a sample from the dataset. This can be seen by noting that our bounds on the consensus are not affected by the sample-size\new{, e.g. see Eq. \ref{eq:con-lem12} in Lemma \ref{lem:restatement-trainloss-convex}}. Rather they depend crucially on the smoothness parameter $L$ which upper-bounds the smoothness parameter of the leave-one-out loss. This leads to the informal version of the lemma presented in the main body. 

\section{Proofs for Section \ref{sec:convex}}\label{sec:convex_proof}

\begin{lemma}[Recursions for the consensus error]\label{lem:con}
Let the step-size $\eta\le (1-\lambda)/4L$ where $\lambda:=\max((|\lambda_2 (A)|,|\lambda_N(A)|))^2$. The consensus error of DGD under Assumptions \ref{ass:mixing},\ref{ass:smooth} satisfies the following:
\bea\label{eq:lem8}
\|W^{(t)} - \bar W^{(t)}\|_F^2 < \alpha_1 \|W^{(t-1)} - \bar W^{(t-1)}\|_F^2 + \alpha_2 N \eta^2L^2  \hat F(\bar w^{(t-1)}),
\eea
where $\alpha_1 := \frac{3+\lambda}{4},\alpha_2:= 4(\frac{2}{1-\lambda}-1)$.
\end{lemma}
\begin{proof}
Denoting $A^{\infty}:=\lim_{t\rightarrow\infty} A^t =\frac{1}{N}\one\one^T$, it holds by Assumption \ref{ass:mixing},
\begin{align}
\|AW-\bar W\|_F^2 &= \|(A - A^{\infty})(W-\bar W)\|_F^2
= \sum_{i=1}^N \|(A - A^{\infty}) (W_i-\bar W_i)\|^2\nn\\
&\le  \sum_{i=1}^N \|A - A^{\infty}\|^2 \|W_i-\bar W_i\|^2\le \max(\lambda_2^2 (A),\lambda_N^2(A)) \cdot \| W-\bar W\|_F^2\label{eq:mx},
\end{align}
where $W_i$ is the $i$ th column of $W$. By Assumption \ref{ass:mixing}, $\lambda=\max((|\lambda_2 (A)|,|\lambda_N(A)|))^2<1.$
For the consensus error, we can write,
\begin{align}
\|W^{(t)} - \bar W^{(t)}\|_F^2 &= \| W^{(t)} -\bar W^{(t-1)}-\bar W^{(t)} + \bar W^{(t-1)}\|_F^2\nn \\
&\le \| W^{(t)} - \bar W^{(t-1)}\|_F^2\nn\\
 &=\| A W^{(t-1)} - \eta \nabla\hat F(W^{(t-1)})- \bar W^{(t-1)}\|_F^2 \nn\\
 &\le (1+\beta)\| A W^{(t-1)}-\bar W^{(t-1)}\|_F^2 + (1+\beta^{-1})\eta^2\|\nabla \hat F(W^{(t-1)})\|_F^2,\nn
 \end{align}
 where the second step is due to $\|X-\bar X\|_F\le \|X\|_F$ \cite{koloskova2019decentralized,koloskova2020unified}. The last line holds for any $\beta>0$, due to $\|a+b\|^2\le(1+\beta^{-1})\|a\|^2+(1+\beta^{-1})\|b\|^2$.
 
 Based on this inequality and by noting \eqref{eq:mx} and using the $L-$smoothness assumption, we can deduce that,
 \begin{align}
\|W^{(t)} - &\bar W^{(t)}\|_F^2 \le (1+\beta)\lambda\|W^{(t-1)}-\bar W^{(t-1)}\|_F^2+ (1+\beta^{-1})\eta^2\|\nabla \hat F(W^{(t-1)})\|_F^2\nn\\[3pt]
 &\le (1+\beta)\lambda\|W^{(t-1)}-\bar W^{(t-1)}\|_F^2+ 2(1+\beta^{-1})\eta^2 \| \nabla \hat F(W^{(t-1)})-\nabla\hat F (\bar W^{(t-1)})\|_F^2 \nn\\&\hspace{1.5in}+ 2 (1+\beta^{-1})\eta^2\|\nabla \hat F(\bar W^{(t-1)})\|_F^2 \label{eq:con_lemma8} \\[3pt]
 &\le  (1+\beta)\lambda\|W^{(t-1)}-\bar W^{(t-1)}\|_F^2 + 2(1+\beta^{-1})\eta^2 L^2 \|W^{(t-1)}-\bar W^{(t-1)}\|_F^2
 \nn\\
 &\hspace{1.5in}+4 (1+\beta^{-1})\eta^2 LN  \hat F(\bar w _{t-1})\nn,
\end{align}
where the last step is due to $L-$smoothness and the non-negativity of $\hat F_\ell$, i.e.  
\begin{align*}
\|\nabla \hat F(\bar W^{(t-1)})\|_F^2 &= \sum_{\ell=1}^N \|\nabla \hat F_\ell (\bar w^{(t-1)})\|^2
\le {2 L} \sum_{\ell=1}^N (\hat F_\ell (\bar w^{(t-1)})-\hat F_\ell^\star)
\le 2L N\hat F(\bar w^{(t-1)}) . 
\end{align*}
%\ct{Why have such a loose bound in the last step here and not do $2L/N\hat{F}$}
Thus,
\begin{align*}
\|W^{(t)} - \bar W^{(t)}\|_F^2 < ((1+\beta)\lambda &+ 2(1+\beta^{-1})\eta^2 L^2) \|W^{(t-1)} - \bar W^{(t-1)}\|_F^2+ 4(1+\beta^{-1})\eta^2LN\,\hat F(\bar w^{(t-1)}).
\end{align*}
Next, choose $\beta=(1-\lambda)/(2\lambda)$. Then, it follows from the assumption  $\eta\le (1-\lambda)/4L$ that
\begin{align*}
&(1+\beta)\lambda + 2(1+\beta^{-1})\eta^2 L^2 <(3+\lambda)/4=\alpha_1,\\
&4(1+\beta^{-1})< 4(2/(1-\lambda)-1)=\alpha_2.
\end{align*}
This concludes the lemma. 
\end{proof}
By telescoping summation over the iterates $t=1,\cdots,T$ of the consensus error in Eq.\eqref{eq:lem8}, we end up with the consensus error at iteration $T$. The final expression is stated in the next lemma. 
\begin{lemma}\label{lem:con_acc}
Under the assumptions of Lemma \ref{lem:con}, it holds for $T > 1$ that, 
\bea\nn
\|W^{(T)} - \bar W^{(T)}\|_F^2  < \alpha_1^{T-1}\|W_1 - \bar W_1\|_F^2 + (\alpha_2\eta^2 LN) \sum_{t=1}^{T-1}\alpha_1^{t-1} \hat F( \bar w^{(T-t)}).
\eea
\end{lemma}
\begin{lemma}\label{lem:con_ave}
Under the assumptions of Lemma \ref{lem:con} and the zero initialization assumption for all agents, the average consensus error satisfies, 
\bea\nn
\frac{1}{NT}\sum_{t=1}^T \|W^{(t)}- \bar W^{(t)}\|_F^2 \le \frac{\alpha_2\eta^2L}{(1-\alpha_1)T} \sum_{t=1}^{T-1}\hat F(\bar w^{(t)}).
\eea
\end{lemma}
\begin{proof}
By Lemma \ref{lem:con_acc} and the zero initialization and  non-negativity assumptions, we have
\begin{align*}
\frac{1}{NT}\sum_{t=1}^T \|W^{(t)} - \bar W^{(t)}\|_F^2  &\le \frac{\alpha_2\eta^2L}{T} \sum_{t=2}^T \sum_{\tau=1}^{t-1}\alpha_1^{\tau-1} \hat F(\bar w^{(t-\tau)})\le \frac{\alpha_2\eta^2L}{T}\sum_{\tau=1}^{T-1} \alpha_1^{\tau-1}\sum_{t=1}^{T-\tau}\hat F(\bar w^{(t)})\\
&\le \frac{\alpha_2\eta^2L}{T}\sum_{\tau=1}^{T-1} \alpha_1^{\tau-1}\sum_{t=1}^{T-1}\hat F(\bar w^{(t)})\le \frac{\alpha_2\eta^2L}{(1-\alpha_1)T} \sum_{t=1}^{T-1}\hat F(\bar w^{(t)}).
\end{align*}
\end{proof}

\begin{lemma}\label{lem:conv-bnd}
Under Assumptions \ref{ass:convex},\ref{ass:smooth} and for all $w\in \R^d$, the DGD updates satisfy the  following recursions:
\bea\nn
 \frac{2\eta-4L\eta^2}{T} \sum_{t=1}^{T-1} \hat F(\bar {w}^{(t)}) \le \frac{\|\bar w^{(1)}-w\|^2}{T}+ 2\eta \hat F(w)+ \frac{2L^2\eta^2+\eta L}{NT}\sum_{t=1}^{T-1}\|W^{(t)}-\bar W^{(t)}\|_F^2 
\eea
\end{lemma}
%%%%%%%%%%%%%%%%%%%%%%%%%%%%% Lemma%%%%%%%%%%%%%%%%%%%%%%%%%%%
\begin{proof}
We start by upper bounding the following quantity:
\begin{align*}
\| \bar w^{(t+1)}-w\|^2  &= \|\bar {w}^{(t)} - \eta\bar\nabla\hat F (W^{(t)})-w\|^2=\|\bar {w}^{(t)} -w\|^2 + \eta^2\|\bar\nabla\hat F (W^{(t)})\|^2-2\eta \langle \bar {w}^{(t)} -w,\bar\nabla\hat F (W^{(t)})\rangle
\end{align*}
For the second term above, using $L-$smoothness  and non-negativity of the loss, we obtain:
\begin{align*}
\|\bar\nabla\hat F (W^{(t)})\|^2 &= \|\bar\nabla\hat F (W^{(t)}) - \nabla\hat F(\bar {w}^{(t)}) + \nabla\hat F(\bar {w}^{(t)})\|^2\\
&\le 2\|\bar\nabla\hat F (W^{(t)}) - \nabla\hat F(\bar {w}^{(t)}) \|^2 + 2\| \nabla\hat F(\bar {w}^{(t)})\|^2\\
&\le \frac{2 L^2}{N}\sum_{i=1}^N \|w^{(t)}_\ell- \bar {w}^{(t)}\|^2 + 4L\hat F(\bar {w}^{(t)}).
\end{align*}
\\
For the third term, by using $L-$smoothness and convexity properties we can write,
\begin{align*}
\langle \bar {w}^{(t)} -&w,\bar\nabla\hat F (W^{(t)})\rangle = \frac{1}{N} \sum_{\ell=1}^N \langle \bar {w}^{(t)} -w,\nabla\hat F_\ell (w^{(t)}_\ell)\rangle \\
& = \frac{1}{N}\sum_{\ell=1}^N \langle \bar {w}^{(t)} -w^{(t)}_\ell,\nabla\hat F_\ell (w^{(t)}_\ell)\rangle +\frac{1}{n} \sum_{\ell=1}^N \langle w^{(t)}_\ell -w,\nabla\hat F_\ell (w^{(t)}_\ell)\rangle \\
%&\ge \frac{1}{N} \sum_{\ell=1}^N \hat F_\ell(\bar {w}^{(t)})-\hat F_\ell(w^{(t)}_\ell) -\frac{L}{2}\|w^{(t)}_\ell-\bar {w}^{(t)}\|^2 +\frac{1}{N} \sum_{\ell=1}^N \langle w^{(t)}_\ell -w,\nabla\hat F_\ell (w^{(t)}_\ell)\rangle\\
&\ge  \frac{1}{N} \sum_{\ell=1}^N \big(\hat F_\ell(\bar {w}^{(t)})-\hat F_\ell(w^{(t)}_\ell) \big) -\frac{L}{2}\|w^{(t)}_\ell-\bar {w}^{(t)}\|^2 +\frac{1}{N}\sum_{\ell=1}^N \big(\hat F_\ell(w^{(t)}_\ell)-\hat F_\ell(w)\big)\\
& = \hat F(\bar {w}^{(t)})-\hat F(w) - \frac{L}{2N}\|W^{(t)}-\bar W^{(t)}\|_F^2.
\end{align*}
Combining these inequalities we derive the following:
\begin{align*}
\| \bar w^{(t+1)}-w\|^2\le\|\bar {w}^{(t)} -w\|^2 + \eta^2 \Big(&2 L^2\|W^{(t)}-\bar W^{(t)}\|_F^2/N + 4L\hat F(\bar {w}^{(t)})\Big) \\&-2\eta\Big(\hat F(\bar {w}^{(t)})-\hat F(w) - \frac{L}{2}\|W^{(t)}-\bar W^{(t)}\|_F^2/N\Big).
\end{align*}
Summing these equations for $t=1,2,...,T-1$,
\begin{align*}
\| \bar w^{(T)}-w\|^2 \le \|\bar w^{(1)}-w\|^2 + \sum_{t=1}^{T-1}\frac{2L^2\eta^2+\eta L}{N}\|W^{(t)}-\bar W^{(t)}\|_F^2 + \sum_{t=1}^{T-1} (4L\eta^2-2\eta) \hat F(\bar {w}^{(t)}) +2\sum_{t=1}^{T-1}\eta \hat F(w)
\end{align*}

\begin{align*}
\Longrightarrow~~ \frac{2\eta-4L\eta^2}{T} \sum_{t=1}^{T-1} \hat F(\bar {w}^{(t)}) \le \frac{\|\bar w^{(1)}-w\|^2}{T}+ \frac{2L^2\eta^2+\eta L}{NT}\sum_{t=1}^{T-1}\|W^{(t)}-\bar W^{(t)}\|_F^2 + 2\eta \hat F(w).
\end{align*}
\end{proof}

\subsection{Proof of Lemma \ref{lem:trainloss-convex} }\label{sec:pf-lem2}
\begin{lemma} [Restatement of Lemma \ref{lem:trainloss-convex}]\label{lem:restatement-trainloss-convex}Under Assumptions \ref{ass:mixing}-\ref{ass:self-bounded} and zero initialization, for any $w\in\R^d$ and for a fixed step-size $\eta<\min\{\frac{1-\alpha_1}{L},\frac{1}{L}\sqrt{\frac{1-\alpha_1}{2\alpha_2 }}\}$, where $\alpha_1\in (3/4,1),\alpha_2>4$ are parameters that depend only on the mixing matrix, the following holds for the train loss and consensus error of DGD:
\bea
\frac{1}{T} \sum_{t=1}^{T} \hat F(\bar {w}^{(t)}) &\le \frac{2\|w\|^2}{\eta T}+ 4\hat F(w),\label{eq:train-lem12}\\
\frac{1}{NT}\sum_{t=1}^T \|W^{(t)}- \bar W^{(t)}\|_F^2 &\le \frac{\alpha_2\eta^2L}{(1-\alpha_1)}(\frac{2\|w\|^2}{\eta T}+ 4\hat F(w)).\label{eq:con-lem12}
\eea
\end{lemma}
\begin{proof}
Recalling the initialization $w_\ell^{(1)}=0\Rightarrow \bar{w}^{(1)}=0$ and using  $\eta<1/(4L)$, we deduce from Lemma \ref{lem:conv-bnd} that,
\bea\nn
 \frac{1}{T} \sum_{t=1}^{T-1} \hat F(\bar {w}^{(t)}) \le \frac{\|w\|^2}{\eta T}+ 2\hat F(w)+ \frac{2L^2{\eta}+ L}{NT}\sum_{t=1}^{T-1}\|W^{(t)}-\bar W^{(t)}\|_F^2 
\eea
%\ct{There a square missing from $\eta$}

%\bea
%\|W_T- \bar W_T\|_F^2 \le \alpha_2\eta^2L^2 \sum_{t=1}^{T-1}\alpha_1^{t-1} \hat F(\bar w_{T-t})
%\eea
By Lemma \ref{lem:con_ave},
\bea
% \frac{1}{T} \sum_{t=1}^{T-1} \hat F(\bar {w}^{(t)}) &\le \frac{\|\bar w^{(1)}-w\|^2}{\eta T}+ 2\hat F(w)+(2L^2\eta+ L)\alpha_2\eta^2 L\frac{1}{T}\sum_{t=1}^{T-1}\sum_{k=1}^{t-1}\alpha_1^{k-1} \hat F(\bar w^{(t-k)})\\
 \frac{1}{T} \sum_{t=1}^{T-1} \hat F(\bar {w}^{(t)}) &\le \frac{\|w\|^2}{\eta T}+ 2\hat F(w)+\frac{(2L^2{{\eta}}+ L)\alpha_2\eta^2 L}{T(1-\alpha_1)}\sum_{t=1}^{T-1} \hat F(\bar w^{(t)})\label{eq:is_the_same}\\
 &\le \frac{\|w\|^2}{\eta T}+ 2\hat F(w)+\frac{1}{2T}\sum_{t=1}^{T-1} \hat F(\bar w^{(t)})\nn.
\eea
where the condition on $\eta$ on the lemma's statement ensures that $(2L^2\eta+ L)\alpha_2\eta^2 L/(1-\alpha_1)<1/2$. This gives the statement of the lemma for the training loss in Eq.\eqref{eq:train-lem12}. Appealing again to Lemma \ref{lem:con_ave} for the consensus error yields \eqref{eq:con-lem12}.
\end{proof}

\begin{remark}[Bounds for leave-one-out consensus error]\label{rem:loo_consensus} The bound in Eq.  \eqref{eq:con-lem12} also applies to the leave-one-out consensus-error term $\frac{1}{T}\sum_{t=1}^{T}\|W_{\neg i}^{(t)}-\bar W_{\neg i}^{(t)}\|_F^2$. To see this starting from Lemma \ref{lem:con_ave} note that we still have 
\bea\label{eq:loo_consensus_train}
\frac{1}{T}\sum_{t=1}^T \|W_{\neg{i}}^{(t)}- \bar W_{\neg{i}}^{(t)}\|_F^2 \le \frac{\alpha_2\eta^2LN}{(1-\alpha_1)T} \sum_{t=1}^{T-1}\hat F_{\neg i}(\bar w_{\neg{i},t}),
\eea
where we denote the leave-one-out train loss $\hat{F}_{\neg{i}}(w):=\frac{1}{n}\sum_{i'\neq i}f(w,x_{i'})$. This is true because the smoothness parameter of $\hat{F}_{\neg{i}}(w)$ is $(1-1/n)L\leq L$. Moreover, applying Lemma \ref{lem:conv-bnd} to the leave-one-out loss (and using again that it's smoothness parameter is upper bounded by $L$), we have for all $w$ that
\bea\nn
 \frac{2\eta-4L\eta^2}{T} \sum_{t=1}^{T-1} \hat{F}_{\neg{i}}(\bar{w}_{\neg{i}}^{(t)}) \le \frac{\|\bar{w}_{\neg{i}}^{(1)}-w\|^2}{T}+ 2\eta \hat{F}_{\neg{i}}(w)+ \frac{2L^2\eta^2+\eta L}{NT}\sum_{t=1}^{T-1}\|W_{\neg{i}}^{(t)}-\bar W_{\neg{i}}^{(t)}\|_F^2 
\eea
But, from the initialization assumption $\bar{w}_{\neg i}^{(1)}=0$ and also $\hat{F}_{\neg{i}}(w)\leq \hat{F}(w)$ since the functions are assumed non-negative. Hence, and also using \eqref{eq:loo_consensus_train}, shows that
\bea\nn
 \frac{2\eta-4L\eta^2}{T} \sum_{t=1}^{T-1} \hat{F}_{\neg{i}}(\bar{w}_{\neg{i}}^{(t)}) \le \frac{\|w\|^2}{T}+ 2\eta \hat{F}(w)+  \frac{(2L^2\eta^2+\eta L)\alpha_2\eta^2 L}{T(1-\alpha_1)}\sum_{t=1}^{T-1}\hat F_{\neg i}(\bar w_{\neg{i},t}).
\eea
Note that after using $\eta<1/(4L)$ this is exactly analogous to Eq. \eqref{eq:is_the_same} for the train loss, which leads to the same bound $\hat{F}_{\neg{i}}(\bar{w}_{\neg{i}}^{(t)})\leq \frac{2||w||^2}{\eta T}+4 \hat{F}(w)$ for the leave-one-out loss. Plugging this back to Eq. \eqref{eq:loo_consensus_train} shows that the bound  in \eqref{eq:con-lem12} also holds for the leave-one-out consensus term.
\end{remark}

\subsection{Proof of Theorem \ref{thm:testloss_cvx} }\label{sec:pf-thm3}

We are ready to prove Theorem \ref{thm:testloss_cvx} by combining our results from Lemmas \ref{lem:trainloss-convex} and \ref{lem:main}. We state the proof for general choice of step-size $\eta$. In particular, Theorem \ref{thm:testloss_cvx} follows by the next theorem after choosing $\eta=O(1/\sqrt{T})$.

\begin{theorem}[Theorem \ref{thm:testloss_cvx} for general $\eta$]
Consider DGD under Assumptions \ref{ass:mixing}-\ref{ass:realizable}, and choose $\eta<\min\{\frac{1-\alpha_1}{L},\frac{1}{L}\sqrt{\frac{1-\alpha_1}{2\alpha_2 }}\}$. The following bound holds for the averaged test error of DGD with separable data up to iteration $T$, assuming $\eps\le\rho(\eps)^2/\eta T$,

 \bea \nn
 \frac{1}{T}\sum_{t=1}^T F(\bar w^{(t)})  = O \left( \frac{\rho(\eps)^2}{\eta T} +  \frac{L^2 c^2 \rho(\eps)^{4\alpha} }{n^{3-2\alpha}} (\eta T)^{2-2\alpha}+ {L^4\rho(\eps)^4}\eta^3 T\right).
\eea

\begin{proof}
By Lemma \ref{lem:main}, 
\bea\nn
\E\Big[&F(\bar w^{(t)})\Big] = O \left( \E\Big[\hat F( \bar w^{(t)})\Big] +   \frac{L^2 c^2\eta^2t^{2}}{n^{3-2\alpha}}\E\Big[(\frac{1}{t}\sum_{\tau=1}^{t}\hat F(\bar {w}^{(\tau)}))^{2\alpha}\Big]+\frac{L^4\eta^2}{N} \E\Big[(\sum_{\tau=1}^t \|W^{(\tau)}-\bar W^{(\tau)}\|)^2\Big]\right).\nn
\eea
Thus, by Lemma  \ref{lem:trainloss-convex}, %\ct{How to get to the line below needs explanation. Also, I think there should be no $t^2$ in the second term?}
\bea\nn
&\frac{1}{T}\sum_{t=1}^T \E\Big[F(\bar w^{(t)})\Big]= O \left(\frac{\|w\|^2}{\eta T}+\hat F(w) +   \frac{L^2 c^2\eta^2}{n^{3-2\alpha}} \frac{1}{T}\sum_{t=1}^T t^2 (\frac{\|w\|^2}{\eta t}+ \hat F(w))^{2\alpha}+\frac{L^4\eta^4}{N} \frac{1}{T}\sum_{t=1}^T t^2(\frac{\|w\|^2}{\eta t}+\hat F(w))\right).\nn
\eea
% \ct{Instead of the above, I see:
% \bea\nn
% &\frac{1}{T}\sum_{t=1}^T \E[F(\bar w^{(t)})] \\
% &= O (\frac{\|w\|^2}{\eta T}+\hat F(w) +   \frac{L^2 c^2\eta^2}{n^{3-2\alpha}} \frac{1}{T}\sum_{t=1}^T (\frac{\|w\|^2}{\eta t}+ \hat F(w))^{2\alpha}+\frac{L^4\eta^4}{N} \frac{1}{T}\sum_{t=1}^T t^2(\frac{\|w\|^2}{\eta t}+\hat F(w)))
% \\
% &\leq O (\frac{\rho(\eps)^2}{\eta T} +   \frac{L^2 c^2\eta^{2-2\alpha}\rho(\eps)^{4\alpha}}{n^{3-2\alpha}} \frac{1}{T}\sum_{t=1}^T (\frac{1}{ t})^{2\alpha}+\frac{L^4\eta^3\rho(\eps)^2}{N} \frac{1}{T}\sum_{t=1}^T t)
% \\
% &\leq O (\frac{\rho(\eps)^2}{\eta T} +   \frac{L^2 c^2\eta^{2-2\alpha}\rho(\eps)^{4\alpha}}{n^{3-2\alpha}} \frac{1}{T}\sum_{t=1}^T (\frac{1}{ t})^{2\alpha}+\frac{L^4\eta^3\rho(\eps)^2}{N} \frac{T+1}{2}).
% \eea
% } \ct{Sorry, I can't see it immediately. Can you please elaborate how you get the second term? Specifically the $T^{2-2a}$ part}
By Assumption \ref{ass:realizable} and assuming $\eps \le \rho(\eps)^2/\eta T $, the statement of the theorem follows. %\ct{This needs more explanation. It is not immediately clear how to conclude from the above. E.g. I do not see how the $T^{2-2a}$ part comes about.} 
\end{proof}
\end{theorem} 
%\begin{lemma} Under the assumptions of Lemmas...., for a fixed step-size $\eta<\sqrt{\frac{1-\alpha_1}{2\alpha_2 L^3}}$, it holds that for any $w\in\R^d$
%\bea
%\frac{1}{T} \sum_{t=1}^{T-1} \hat F(\bar {w}^{(t)}) \le \frac{2\|w\|^2}{\eta T}+ 4\hat F(w)
%\eea
%\end{lemma}

%%%%%%%%%%%%%%%%%%%%%%%%%%%%%%%%%%%%%%%%%%%%%%%%%%%%%%%%%
\section{Proofs for Section \ref{sec:LIC}}\label{sec:LIC_proofs}
\begin{lemma}[Iterates of consensus error]\label{lem:con_lem14}
Consider DGD with the loss functions and mixing matrix satisfying Assumptions \ref{ass:mixing},\ref{ass:convex}, \ref{ass:laplace} and Assumption \ref{ass:self-bounded} with $\alpha=1$ and $c=h$. By choosing $\eta_{t-1}\le \frac{1-\lambda}{4h N M_{(t-1)}}$ the consensus error at iteration $t>1$ satisfies
 \begin{align}\label{eq:conse_lem14}
\left\|W^{(t)} - \bar W^{(t)}\right\|_F^2 &< \beta_1  \left\| W^{(t-1)}-\bar W^{(t-1)}\right\|_F^2+ \beta_2 \eta_{t-1}^2 h^2 N^2 \hat F^2(\bar w^{(t-1)})
\end{align}
where we define $$\beta_1:=(3+\lambda)/4, \beta_2 := (4/(1-\lambda)-2),\lambda:= \max\{|\lambda_2(A)|^2,|\lambda_N(A)|^2\}, M_{(t-1)}:= \max \{\hat F(W^{(t-1)}), \hat F(\bar w^{(t-1)})\}.$$
\end{lemma}
\begin{proof}
By Lemma \ref{lem:con} and the inequality \eqref{eq:con_lemma8}, the consensus error satisfies for any $\beta>0$, 
\begin{align}
\Big\|W^{(t)} &- \bar W^{(t)}\Big\|_F^2 < (1+\beta)\lambda\Big\|W^{(t-1)}-\bar W^{(t-1)}\Big\|_F^2 \label{eq:con_lem14}\\
&+ 2(1+\beta^{-1})\eta_{t-1}^2 \Big\| \nabla \hat F(W^{(t-1)})-\nabla\hat F (\bar W^{(t-1)})\Big\|_F^2 + 2 (1+\beta^{-1})\eta_{t-1}^2\Big\|\nabla \hat F(\bar W^{(t-1)})\Big\|_F^2\nn\,.
\end{align}
For the second term in \eqref{eq:con_lem14}, we have the following chain of inequalities,
\bea
\Big\| \nabla \hat F(W^{(t-1)})-&\nabla\hat F (\bar W^{(t-1)})\Big\|_F^2 \nn\\
&=\sum_{\ell=1}^N \| \nabla \hat F_\ell(w^{(t-1)}_i)-\nabla\hat F_\ell (\bar w^{(t-1)})\|^2\nn
\\ &\le \sum_{\ell=1}^N \max_{v_\ell\in[w_\ell^{(t-1)},\bar w^{(t-1)}]} \|\nabla^2 \hat F_\ell(v_\ell)\|^2\| w_\ell^{(t-1)}-\bar w^{(t-1)}\|^2\label{eq:13}
\\ &\le h^2 \sum_{\ell=1}^N \max_{v_\ell\in[w_\ell^{(t-1)},\bar w^{(t-1)}]} (\hat F_\ell(v_\ell))^2\| w_\ell^{(t-1)}-\bar w^{(t-1)}\|^2\label{eq:14}
\\ &= h^2 \sum_{\ell=1}^N (\max_{v_\ell\in[w_\ell^{(t-1)},\bar w^{(t-1)}]} \hat F_\ell(v_\ell))^2\| w_\ell^{(t-1)}-\bar w^{(t-1)}\|^2\nn
\\ &\le h^2 \sum_{\ell=1}^N \max \{\hat F_\ell^2(w_\ell^{(t-1)}),\hat F_\ell^2(\bar w^{(t-1)})\} \| w_\ell^{(t-1)}-\bar w^{(t-1)}\|^2\label{eq:17}
\\&\le h^2 \max \{\max_{k\le N}\hat F_k^2(w_k^{(t-1)}),\max_{k\le N}\hat F_k^2(\bar w^{(t-1)})\} \sum_{\ell=1}^N  \| w_\ell^{(t-1)}-\bar w^{(t-1)}\|^2
\nn
\\
&\le  h^2 N^2 M_{(t-1)}^2  \Big\| W^{(t-1)}-\bar W^{(t-1)}\Big\|_F^2.\label{eq:mt1}
\eea
The Taylor's remainder theorem gives \eqref{eq:13} and $v_i\in[w_i^{(t-1)},\bar w^{(t-1)}]$ denotes a point that lies on the line connecting $w_i^{(t-1)}$ and $\bar w^{(t-1)}$. Also, \eqref{eq:14} is valid due to the self-boundedness of the Hessian stated in Assumption \ref{ass:laplace}. The inequality \eqref{eq:17} follows by the assumption of convexity of $\hat F_i$, due to the fact that for a convex function $f:\R^d\rightarrow\R$ and any two points $w_1,w_2\in \R^d$, it holds that $\max_{v\in[w_1,w_2]}f(v) \le \max \{f(w_1),f(w_2)\}$. To derive \eqref{eq:mt1}, we used $\max_{i\le N}\hat F_i(w_i)\le N \hat F(W)$ and $\max_{i\le N}\hat F_i(\bar w)\le N\cdot\hat F(\bar w)$, which hold since the loss functions are non-negative. 

In order to derive an upper-bound on the last term in \eqref{eq:con_lem14}, we use Assumption \ref{ass:self-bounded} (with $\alpha=1, c=h$):
 \begin{align*}
\Big \|\nabla \hat F(\bar W^{(t-1)})\Big\|_F^2 &= \sum_{\ell=1}^N\Big \|\nabla \hat F_\ell (\bar w^{(t-1)})\Big\|^2 
 \le h^2 \sum_{\ell=1}^N (\hat F_\ell(\bar w^{(t-1)}))^2
 \le h^2 N^2 \hat F^2(\bar w^{(t-1)}).
 \end{align*}
Replacing the upper-bounds back in \eqref{eq:con_lem14}, we conclude
 \begin{align*}
\Big\|W^{(t)} - \bar W^{(t)}\Big\|_F^2 < ((1+\beta)\lambda+ &2(1+\beta^{-1}) \eta_{t-1}^2  h^2 N^2 M_{(t-1)}^2) \Big \| W^{(t-1)}-\bar W^{(t-1)}\Big\|_F^2
 \\&+ 2 (1+\beta^{-1}) \eta_{t-1}^2 h^2 N^2 \hat F^2(\bar w^{(t-1)}).
\end{align*}
Choose $\beta=\frac{1-\lambda}{2\lambda}$. Then by lemma's assumption $\eta_{t-1}\le \frac{1-\lambda}{4h N M_{(t-1)}}$, we can  verify the following two inequalities:
\begin{align*}
&(1+\beta)\lambda+ 2(1+\beta^{-1})\eta_{t-1}^2  h^2 N^2 M_{(t-1)}^2 \le \frac{3+\lambda}{4},\\
& 2(1+1/\beta)\leq \frac{4}{1-\lambda}-2.
\end{align*}
%\ct{Some constants seem off in the first equation since $(1+\beta)\lambda \leq (1+\lambda)/2$ and $\eta_{t-1}^2 2(1+\beta^{-1})  h^2 \new{N^2}M_{(t-1)}^2= \eta_{t-1}^2 2(1+\lambda)/(1-\lambda)  h^2 \new{N^2}M_{(t-1)}^2$, which assuming $\eta_{t-1}\leq 1-\lambda/(4hNM)$ is upper bounded by $ (1+\lambda)/8$. Overall, you get an upper bound $5/8(1+\lambda)$ which can be larger than $(3+\lambda)/4$
%}
This concludes the proof. 
\end{proof}

%%%%%%%%%%%%%%%%%%%%%%%%%%%%%%%%%%%%%%%%%%%%%%%%%%%%%%%%%
%The next lemma bounds the consensus error at iteration $T$, 
By recursively evaluating \eqref{eq:conse_lem14}, we obtain a bound on the consensus error at iteration $T$, which we present next.
\begin{lemma}[Last iterate consensus error]\label{lem:con_acc_lem15}
Under the assumptions and notations of Lemma \ref{lem:con_lem14}, the consensus error at iteration $T$ satisfies
 \begin{align*}
\Big\|W^{(T)} - \bar W^{(T)}\Big\|_F^2 &< \beta_1^{T-1}  \Big\| W^{(1)}-\bar W^{(1)}\Big\|_F^2+ \beta_2 h^2 N^2\sum_{t=1}^{T-1}\beta_1^{t-1} \eta_{_{T-t}}^2\hat F^2(\bar w_{T-t}).
\end{align*}
\end{lemma}

The next lemma obtains a sandwich relation between $F(\bar w^{(T)})$ and $\hat F(W^{(T)})$. This is convenient as it allows replacing $M_{(t)}:=\max(\hat{F}(W^{(t)}),\hat{F}(\bar{w}^{(t)}))$ by either of the two terms with only paying a constant factor of two. See also the remark after the statement of the theorem.

\begin{lemma}\label{lem:mt}
Under the assumptions and notations of Lemma \ref{lem:con_lem14}, with zero initialization $W^{(1)}=\bar W^{(1)}=0$ and by choosing $\eta_{t}\le \frac{(1-\lambda)\sqrt{1-\beta_1}}{8h^2 N M_{(t)}\sqrt{\beta_2}}$ for $t\in[T-1]$, it holds at iteration $T$ that 
%\ct{For the step-size I think it should be $h$ (not $h^2$) and there is also an $N$ in the denominator}
\bea
\frac{1}{2} \hat F(\bar w^{(T)}) \le \hat F(W^{(T)}) \le 2\hat F(\bar w^{(T)}).
\eea
\begin{proof}
First, we prove $\hat F(W^{(T)}) \le 2\hat F(\bar w^{(T)}).$ If $\hat F(W^{(T)}) \le \hat F(\bar w^{(T)}),$ there is nothing to prove. Thus, assume $\hat F(W^{(T)})\ge \hat F(\bar w^{(T)} )$. Then by applying Taylor's remainder theorem, the self-boundedness Assumption \ref{ass:self-bounded} with $c=h,\alpha=1$, convexity of $\hat F$, Lemma \ref{lem:con_acc_lem15} and the restriction on the step-size, in respective order, we have the following inequalities,
\begin{align*}
\hat F(W^{(T)}) &\le |\hat F(W^{(T)})-\hat F(\bar w^{(T)})| + \hat F(\bar w^{(T)})\\
&\le \max_{v\in[\bar W^{(T)},W^{(T)}]}\|\nabla \hat F(v)\|\cdot\|W^{(T)}-\bar W^{(T)}\| + \hat F(\bar w^{(T)})\\
&\le h\cdot\max_{v\in[\bar W^{(T)},W^{(T)}]} \hat F(v) \cdot\|W^{(T)}-\bar W^{(T)}\| + \hat F(\bar w^{(T)})\\
&\le h\cdot\max\{\hat F(W^{(T)}),\hat F(\bar w^{(T)})\}\cdot\|W^{(T)}-\bar W^{(T)}\| + \hat F(\bar w^{(T)})\\
&\le \hat F(W^{(T)}) \Big(\beta_2 h^4 N^2 \sum_{t=1}^{T-1}\beta_1^{t-1} \eta_{_{T-t}}^2\hat F^2(\bar w_{T-t}) \Big)^{1/2}+ \hat F(\bar w^{(T)} )\\
&\le \frac{1}{2}\hat F(W^{(T)}) +  \hat F(\bar w^{(T)}).
\end{align*}
%\ct{Penultimate line, should it be $h^2$ and there is a $N^2$ now missing?}
Thus $\hat F(W^{T})\le 2\hat F(\bar w^{(T)} )$. By exchanging $W^{(T)}$ and $\bar w^{(T)}$ and in a similar style we derive $\hat F(W^{T})\ge \frac{1}{2}\hat F(\bar w^{(T)} )$. This completes the proof of the lemma. 
\end{proof}
\end{lemma}

\begin{remark}\label{rem:Mt}
Lemma \ref{lem:con_acc_lem15} above requires tuning $\eta_t\propto 1/M_{(t)}:=1/\max(\hat{F}(W^{(t)}),\hat{F}(\bar{w}^{(t)})).$ Lemma \ref{lem:mt} shows that abiding by this choice for $t=1,\ldots,T-1$ and any $T>1$ guarantees $\hat{F}(W^{(T)})\leq 2\hat{F}(\bar{w}^{(T)})$. Hence, $M_{(T)}\geq 2 \hat{F}(\bar{w}^{(T)})$. Since this holds for all $T$ and at $t=1, \hat{F}(W^{(1)})=\hat{F}(\bar{w}^{(1})$, it follows by recursion that  Lemma \ref{lem:con_acc_lem15} holds provided $\eta_t\propto 1/\hat{F}(\bar{w}^{(t)})$. We use observation in the proofs below.
\end{remark}

%%%%%%%%%%%%%%%%%%%%%%%%%%%%%%%%%%%%%%%%%%%%%%%%%%%%%%%%%
We are ready to prove Theorem \ref{lem:exp_dsc}. First, we prove that DGD is a descent algorithm in the next lemma.  
\begin{lemma}[Descent lemma]\label{lem:descent_log}
Consider DGD under the assumptions and notations of Lemma \ref{lem:con_lem14}. Moreover, let Assumption \ref{ass:8} hold, then by choosing  $\eta_t\le \frac{\delta}{\hat F(\bar w^{(t)} )}$ for $t\le T$, where 
\begin{align}\label{eq:alpha}
\delta:=1\Big/\max\Big\{\frac{4h^3 N}{\tau^2},h^2,\frac{6 h^2\beta_2}{1-\beta_1},\frac{4 h^2\sqrt{\beta_2}}{\tau(1-\beta_1)}\Big\},
\end{align}

DGD is a descent algorithm, i.e., for all $T\geq1$.
\bea\nn
\hat F(\bar w^{(T+1)})\le \hat F(\bar w^{(T)}).
\eea
\end{lemma}

\begin{proof}
With the self-boundedness assumption on the Hessian (Assumption \ref{ass:laplace}) and applying the Taylor's remainder theorem for step $T+1$ of DGD, we obtain the following,
\bea
\hat F(&\bar w^{(T+1)})\nn\\
&\le \hat F (\bar w^{(T)})+ \langle \nabla \hat F(\bar w^{(T)}),\bar w^{(t+1)}-\bar w^{(T)} \rangle + \frac{1}{2}\max_{v\in [\bar w^{(T)},\bar w^{(T+1)}]} \|\nabla ^2 \hat F(v)\|\|\bar w^{(T+1)}-\bar w^{(T)}\|^2\nn\\
&\le  \hat F (\bar w^{(T)})-\eta_{_T} \langle\nabla \hat F(\bar w^{(T)}),\bar \nabla \hat F(W^{(T)})\rangle + \frac{\eta_{_T}^2}{2}\max_{v\in [\bar w^{(T)},\bar w^{(T+1)}]} \|\nabla ^2 \hat F(v)\|\|\bar \nabla \hat F(W^{(T)})\|^2\nn\\
&\le \hat F (\bar w^{(T)})-\eta_{_T} \langle\nabla \hat F(\bar w^{(T)}),\bar \nabla \hat F(W^{(T)})\rangle + \frac{h\eta_{_T}^2}{2}\max_{v\in [\bar w^{(T)},\bar w^{(T+1)}]} \hat F(v)\|\bar \nabla \hat F(W^{(T)})\|^2\nn\\
&\le \hat F (\bar w^{(T)})-\eta_{_T} \langle\nabla \hat F(\bar w^{(T)}),\bar \nabla \hat F(W^{(T)})\rangle + \frac{h\eta_{_T}^2}{2}\max \{\hat F(\bar w^{(T)}),\hat F(\bar w^{(T+1)})\}\|\bar \nabla \hat F(W^{(T)})\|^2,\label{eq:taylor_lem17}
\eea
where for the third step we used
\begin{align*}
\|\nabla^2\hat F(w)\|&\le \frac{1}{N}\sum_{\ell=1}^N \|\nabla^2 \hat F_\ell(w)\| \le \frac{h}{N}\sum_{\ell=1}^N \hat F_\ell(w) = h\,\hat F(w).
\end{align*}
% last term:
% \bea
% \|\bar \nabla \hat F(W^{(t)})\|^2 &\le \frac{1}{N}\sum_{i=1}^n \|\nabla\hat F_i(w_i^{(t)})\|^2\\
% &\le  \frac{h^2}{N}\sum_{i=1}^n  \hat F_i ^2 (w_i^{(t)})\\
% &\le h^2 N \hat F^2(W^{(t)})
% \eea
In the next step of the proof, we upper-bound the second and third terms in \eqref{eq:taylor_lem17}. For the second term, by noting that $2\langle a,b\rangle=\|a\|^2+\|b\|^2-\|a-b\|^2$, we can write
\bea\label{eq:a-b}
\langle\nabla \hat F(\bar w^{(T)}),\bar \nabla \hat F(W^{(T)})\rangle = \frac{1}{2}\|\nabla \hat F(\bar w^{(T)})\|^2 +\frac{1}{2} \|\bar \nabla \hat F(W^{(T)})\|^2 - \frac{1}{2}\|\nabla \hat F(\bar w^{(T)})-\bar \nabla \hat F(W^{(T)})\|^2.
\eea
By recalling \eqref{eq:mt1} and Lemma \ref{lem:con_acc_lem15} (which we can apply because of Remark \ref{rem:Mt}), we find an upper-bound the last term in \eqref{eq:a-b} as follows, 
%\ct{To call Lemma 16, how do you guarantee $\eta_{t}\leq M_t$? I understand this would hold if }
\begin{align}
\|\nabla \hat F(\bar w^{(T)})-\bar \nabla \hat F(W^{(T)})\|^2 &=\frac{1}{N^2}\| \nabla \hat F(W^{(T)})-\nabla\hat F (\bar W^{(T)})\|_F^2 \nn\\&\le h^2 M_{(T)}^2  \| W^{(T)}-\bar W^{(T)}\|_F^2\nn\\
&\le  h^4 M_{(T)}^2 N^2 \beta_2 \sum_{t=1}^{T-1}\beta_1^{t-1}\eta_{_{T-t}}^2\hat F^2(\bar w_{T-t}).\label{eq:l17-lip}
\end{align}
Returning back to \eqref{eq:taylor_lem17}, thus far we have derived the following,
\begin{align}
\hat F(&\bar w^{(T+1)})\le \hat F (\bar w^{(T)})- \frac{\eta_{_T}}{2}\|\nabla \hat F(\bar w^{(T)})\|^2 -\frac{\eta_{_T}}{2} \|\bar \nabla \hat F(W^{(T)})\|^2 \nn \\&+ \frac{1}{2}\eta_{_T} h^4 N^2 M_{(T)}^2 \beta_2 \sum_{t=1}^{T-1}\beta_1^{t-1}\eta_{_{T-t}}^2\hat F^2(\bar w_{T-t}) 
+ \frac{h \eta_{_T}^2}{2} \|\bar \nabla \hat F(W^{(T)})\|^2\cdot\max \{\hat F(\bar w^{(T)}),\hat F(\bar w^{(T+1)})\}.\label{eq:lem17_l}
\end{align}
We aim to prove that $\hat F(\bar w^{(T+1)})\le \hat F(\bar w^{(T)})$ for all $T\ge1$. 
%
%For $t=1$, we have 
%\begin{align*}
%\hat F(\bar w^{(2)})\le \hat F (\bar w^{(1)})- \frac{\eta_1}{2}\|\nabla \hat F(\bar w^{(1)})\|^2 -&\frac{\eta_1}{2} \|\bar \nabla \hat F(W^{(1)})\|^2  \\
%&+ \frac{h \eta_1^2}{2} \|\bar \nabla \hat F(\bar w^{(1)})\|^2\cdot\max (\hat F(\bar w^{(1)}),\hat F(\bar w^{(2)})).
%\end{align*}
%If $\hat F(\bar w^{(2)}) > \hat F(\bar w^{(1)})$ the above inequality leads to 
%\begin{align*}
%\hat F(\bar w^{(2)})&\le \hat F (\bar w^{(1)}) + \frac{\eta_1}{2} \|\bar \nabla \hat F(\bar w^{(1)})\|^2(\eta_1 h \hat F(\bar w^{(2)})-1)\\
%&\le \hat F (\bar w^{(1)}) + \frac{\eta_1 h^2}{2} \hat F^2(\bar w^{(1)})(\eta_1 h \hat F(\bar w^{(2)})-1).
%\end{align*}
%This leads to contradiction since by the condition on the step-size $\eta_1\le \min(\frac{1}{h^2 \hat F(\bar w^{(1)})}, \frac{1}{h \hat F(\bar w^{(1)})})$, we have,
%\begin{align*}
%\hat F(\bar w^{(2)}) \le \hat F (\bar w^{(1)}) + \frac{1}{2} \hat F(\bar w^{(1)})(\frac{\hat F(\bar w^{(2)})}{\hat F(\bar w^{(1)})}-1),
%\end{align*}
%which yields the contradictory results $\hat F(\bar w^{(2)}) \le \hat F (\bar w^{(1)})$.
%
%Now we assume for all $t\le T-1$ where $T\ge2$, it holds that $\hat F(\bar w^{(t+1)})\le \hat F(\bar w^{(t)})$ and we aim to prove $$\hat F(\bar w^{(T+1)})\le \hat F(\bar w^{(T)}).$$ 
If $\hat F(\bar w^{(T+1)}) > \hat F(\bar w^{(T)})$, applying \eqref{eq:lem17_l} with the assumption $\eta_t< \frac{\delta}{\hat F(\bar w^{(t)} )}$ yields,
\begin{align}
\hat F(\bar w^{(T+1)})&\le \hat F (\bar w^{(T)})- \frac{\eta_{_T}}{2}\|\nabla \hat F(\bar w^{(T)})\|^2 + \frac{1}{2(1-\beta_1)}\eta_{_T}\delta^2 h^4 N^2 M_{(T)}^2 \beta_2  \nn\\
&+ \frac{h \eta_{_T}^2}{2} \|\bar \nabla \hat F(W^{(T)})\|^2\cdot\hat F(\bar w^{(T+1)}).\label{eq:descent}
\end{align}
Note that it holds due to \eqref{eq:l17-lip} that, 
\begin{align*}
 \|\bar \nabla \hat F(W^{(T)})\|^2 &\le  2\| \nabla \hat F(\bar w^{(T)})\|^2 +2  \|\bar \nabla \hat F(W^{(T)})-\nabla \hat F(\bar w^{(T)})\|^2\\
 &\le 2\| \nabla \hat F(\bar w^{(T)})\|^2+  \frac{2}{1-\beta_1}\delta^2  h^4 N^2 M_{(T)}^2 \beta_2.
\end{align*}
Replacing this in \eqref{eq:descent} and noting that $M_{(T)}\le 2\hat F(\bar w^{(T)})$ by Lemma \ref{lem:mt}, we can simplify the inequality \eqref{eq:descent} as follows,
\begin{align*}
\hat F&(\bar w^{(T+1)})\\ &\le
\hat F (\bar w^{(T)}) + \eta_{_T} \|\nabla \hat F(\bar w^{(T)})\|^2(h\eta_{_T}\hat F(\bar w^{(T+1)})-\frac{1}{2}) + C'h^2 N^2 \delta^2\eta_{_T}\hat F^2(\bar w^{(T)}) \, (1+\eta_{_T} \hat F(\bar w^{(T+1)})) \\
&\le \hat F (\bar w^{(T)})  + \eta_{_T}  h^2  \hat F^2(\bar w^{(T)})((h\eta_{_T}+\delta^2\eta_{_T} N^2 C')\hat F(\bar w^{(T+1)})- \frac{\tau^2}{2h^2} + \delta^2  N^2C'),
\end{align*}
where for the ease of notation we define $C':= 4 h^2\beta_2(1-\beta_1)^{-1}$. Recalling $\eta_T< \frac{\delta}{\hat F(\bar w^{(T)} )}$ and noting that by the assumption of the lemma $\delta \le\frac{\tau^2}{4h^3}$, $\delta< \frac{1}{h^2}$ and $\delta< \frac{\tau}{2h N\sqrt{C'}}$ we conclude that,
\bea\nn
\hat F(\bar w^{(T+1)})\le\hat F (\bar w^{(T)}) + \frac{\tau}{2h}\hat F (\bar w^{(T)})(\frac{\hat F(\bar w^{(T+1)})}{\hat F(\bar w^{(T)})}-1).
\eea
Dividing both sides by $\hat F (\bar w^{(T)})$ leads to the contradiction due to the fact that $\tau\le h$ and thus $\tau/2h <1$. Thus $\hat F(\bar w^{(T+1)})\le \hat F (\bar w^{(T)})$. This completes the proof. 
\end{proof}
%%%%%%%%%%%%%%%%%%%%%%%%%%%%%%%%%%%%%%%%%%%%%%%%%%%%%%%%%
%%%%%%%%%%%%%%%%%%%%%%%%%%%%%%%%%%%%%%%%%%%%%%%%%%%%%%%%%
\subsection{Proof of Theorem \ref{lem:exp_dsc}}\label{sec:pf-thm4}
\begin{theorem}[Restatement of Theorem \ref{lem:exp_dsc}]
Consider DGD with the loss functions and mixing matrix satisfying Assumptions \ref{ass:mixing}, \ref{ass:convex}, \ref{ass:laplace}, \ref{ass:8} and Assumption \ref{ass:self-bounded} with $\alpha=1$ and $c=h$. Assume that the step-size satisfies $\eta<\frac{\delta}{\hat F(1)}$ for $\delta$ defined in \eqref{eq:alpha}. Also, recall positive constants $\beta_1,\beta_2$ depending only on the mixing matrix as defined in Lemma \ref{lem:con_lem14}. %$\alp:=(\max(\frac{4h^3}{\tau^2},h^2,\frac{6 h^2\beta_2}{1-\beta_1},\frac{4 h^2\sqrt{\beta_2}}{\tau(1-\beta_1)}))^{-1}$
%, where $\beta_1,\beta_2$ are defined same as in Lemma \ref{lem:con_lem14}. 
Then DGD is a descent algorithm i.e, for all $T\ge1$ it holds that 
$$\hat F(\bar w^{(T+1)})\le \hat F(\bar w^{(T)}).$$ 
Moreover, the train loss and the consensus error of DGD at iteration $T$ satisfy the following for all $w\in \R^d$, 
\bea
\hat F(\bar w^{(T)}) &\le 4\hat F(w) + \frac{2\|w\|^2}{\eta T},\label{eq:thm18-tr}\\
\frac{1}{N^2}\|W^{(T)} - \bar W^{(T)}\|_F^2 %&=
%O(  \eta^2 \hat F^2(w) + \frac{\|w\|^4}{T^2}).
&\leq \frac{8\beta_2 h^2 }{1-\beta_1}(4\eta^2\hat F^2(w) + \frac{\|w\|^4}{T^2}). 
\label{eq:thm18-con}
\eea
\end{theorem}

\begin{proof}
First, we note that by Lemma \ref{lem:descent_log}, under the assumption $\eta_t\le\delta/\hat F(\bar w^{(t)})$  for $t\le T$, we have $\hat F(\bar w^{(T+1)})\le \hat F(\bar w^{(T)}).$ Thus fixing $\eta \le \delta /\hat F(\bar w^{(1)})$, ensures that $\hat F(\bar w^{(T+1)})\le \hat F(\bar w^{(T)}),$ for all $T$. 

Next, we derive the train loss and consensus error under the assumptions of the theorem. Start with,
\bea\label{eq:thm18}
\| \bar w^{(t+1)}-w\|^2  =\|\bar {w}^{(t)} -w\|^2 + \eta^2\|\bar\nabla\hat F (W^{(t)})\|^2-2\eta \langle \bar {w}^{(t)} -w,\bar\nabla\hat F (W^{(t)})\rangle.
\eea
For the second term, by self-boundedness of gradient, we can write,
\begin{align*}
\|\bar\nabla\hat F (W^{(t)})\| = \frac{1}{n}\|\sum_{\ell=1}^n \nabla \hat F_\ell(w_\ell^{(t)})\| \le \frac{h}{n} \sum_{\ell=1}^n \hat F_\ell(w_\ell^{(t)})=h \hat F(W^{(t)}).
\end{align*}
For the third term in \eqref{eq:thm18}, we have,
\bea
-\langle \bar {w}^{(t)} -w,\bar\nabla\hat F (W^{(t)})\rangle &= -\frac{1}{N} \sum_{\ell=1}^N \langle \bar {w}^{(t)} -w,\nabla\hat F_\ell (w^{(t)}_\ell)\rangle\nn \\
& = -\frac{1}{N}\sum_{\ell=1}^N \langle \bar {w}^{(t)} -w^{(t)}_\ell,\nabla\hat F_\ell (w^{(t)}_\ell)\rangle -\frac{1}{N} \sum_{\ell=1}^N \langle w^{(t)}_\ell -w,\nabla\hat F_\ell (w^{(t)}_\ell)\rangle\nn \\
&\le \frac{1}{N} \sum_{\ell=1}^N \|w^{(t)}_\ell-\bar {w}^{(t)}\|\|\nabla \hat F_i(w_\ell^{(t)})\| - \frac{1}{N} \sum_{\ell=1}^N \langle w^{(t)}_\ell -w,\nabla\hat F_\ell (w^{(t)}_\ell)\rangle\nn\\
&\le  \frac{1}{N} \sum_{\ell=1}^N \|w^{(t)}_\ell-\bar {w}^{(t)}\|\|\nabla \hat F_\ell(w_\ell^{(t)})\| + \frac{1}{N}\sum_{\ell=1}^N (\hat F_\ell(w)-\hat F_\ell(w_\ell^{(t)}))\label{eq:26}\\
&\le \frac{h}{N}  \sum_{\ell=1}^N \hat F_\ell(w_\ell^{(t)}) \|w^{(t)}_\ell-\bar {w}^{(t)}\|+ \frac{1}{N}\sum_{\ell=1}^N (\hat F_\ell(w)-\hat F_\ell(w_\ell^{(t)}))\label{eq:27}\\
&\le h \hat F(W^{(t)})\| W^{(t)}- \bar W^{(t)}\|_F \;+ \;\hat F(w)- \hat F(W^{(t)}). \nn
\eea
Here \eqref{eq:26} follows by convexity of $\hat F_i$, and \eqref{eq:27} follows by the assumption on self-boundedness of the gradient. 

Thus, the inequality \eqref{eq:thm18} can be written as follows,
\bea\label{eq:thm18_2}
\| \bar w^{(t+1)}-w\|^2  \le \|\bar {w}^{(t)} -w\|^2 + \eta^2 h^2 \hat F^2(W^{(t)}) &+2\eta h \hat F(W^{(t)})\| W^{(t)}- \bar W^{(t)}\|_F+ 2\eta \hat F(w)- 2\eta \hat F(W^{(t)}). 
\eea
Moreover, by Lemma \ref{lem:con_acc_lem15} and the assumption on $\eta$,
\bea\nn
\| W^{(t)}- \bar W^{(t)}\|_F \le( \beta_2 h^2\sum_{t=1}^{T-1}\beta_1^{t-1} \eta_{_{T-t}}^2\hat F^2(\bar w_{T-t}) )^{1/2}\le \frac{1}{4h}
\eea
and
\bea
\eta\hat F(W^{(t)})\le \frac{1}{2h^2}.\nn
\eea
Thus \eqref{eq:thm18_2}  changes into,
\bea\nn
\| \bar w^{(t+1)}-w\|^2  \le \|\bar {w}^{(t)} -w\|^2 - \eta \hat F(W^{(t)}) + 2\eta \hat F(w).
\eea
Telescoping sum leads to 
\bea\label{eq:thm18-trW}
\frac{1}{T}\sum _{t=1}^T \hat F(W^{(t)}) \le 2\hat F(w) + \frac{\| \bar w^{(1)}-w\|^2}{\eta T}.
\eea
By Lemma \ref{lem:mt}, we have $\hat F(\bar w^{(t)})\le 2\hat F(W^{(t)} )$. Finally, as we proved in the beginning,  DGD is a descent algorithm, implying
\bea\nn
\hat F(\bar w^{(T)})\le \frac{1}{T}\sum _{t=1}^T \hat F(\bar w^{(t)})
\eea
In view of \eqref{eq:thm18-trW}, this yields the claim of the theorem for the train loss \eqref{eq:thm18-tr}. Finally, appealing to Lemma \ref{lem:con_acc_lem15}, gives \eqref{eq:thm18-con}. This completes the proof of the theorem. 
\end{proof}

%\begin{lemma}[Stability for non-smooth losses]
%Under the conditions of the previous lemma the following hold for stability of non-smooth functions:
%\bea
%\frac{1}{n} \sum_{i=1}^n \|\bar w^{(T)}- \bar w_{\neg i}^{(T)}\|^2  = O \left( \frac{h^2\eta^2T^{2}\hat F(\bar w^1)}{n^{2}}(\frac{1}{T}\sum_{t=1}^{T}\hat F(\bar {w}^{(t)}))^2 +{\eta^2 h^2\hat F(\bar w^1)}(\sum_{t=1}^{T}\|W^{(t)}-\bar W^{(t)}\|_F)^2\right)
%\eea
%\end{lemma}
%
%
%\begin{proof}
%Following the proof from Lemma \ref{lem:main} and noting the non-expansiveness of GD for non-smooth functions:
%\begin{align*}
%&\|\bar w^{(t+1)} - \bar w_{\neg i}^{(t+1)}\|\le \| \bar {w}^{(t)} - \bar w_{\neg i}^{(t)}\| + \frac{h \eta}{n} \; f(\bar {w}^{(t)},x_i)+\eta \|\bar \nabla \hat F(W_{\neg i}^{(t)}) -  \nabla \hat F(\bar W_{\neg i}^{(t)})\| + \eta \|\bar \nabla \hat F ( W^{(t)}) -  \nabla \hat F (\bar W^{(t)})\| \\
%&\le  \| \bar {w}^{(t)} - \bar w_{\neg i}^{(t)}\| + \frac{h \eta}{n} \; f(\bar {w}^{(t)},x_i) +  \eta h M_{(t)}  \| W_{\neg i}^{(t)}-\bar W_{\neg i}^{(t)}\|_F +\eta h M_{(t)}  \| W^{(t)}-\bar W^{(t)}\|_F
%\end{align*}
%\bea
%\frac{1}{n} \sum_{i=1}^n \|\bar w^{(T)} - \bar w_{\neg i}^{(T)}\|^2 \le  \frac{3h^2\eta^2T^{2}\hat F(\bar w^1)}{n^{2}}(\frac{1}{T}\sum_{t=1}^{T}\hat F(\bar {w}^{(t)}))^{2}+\frac{3\eta^2 h^2\hat F(\bar w^1)}{n} \sum_{i=1}^n e_{\neg{i}}^{(T)} +{3\eta^2 h^2\hat F(\bar w^1)}e^{(T)}
%\eea
%\end{proof}

%%%%%%%%%%%%%%%%%%%%%%%%%%%%%%%%%%%%%%%%%%%%%%%%%%%%%%%%%%%
\section{Proofs for Section \ref{sec:polyak}}\label{sec:polyak_proofs}

 \begin{lemma}[Train loss under PL condition]\label{Lem:app_pl}
Let Assumptions \ref{ass:mixing},\ref{ass:smooth} and \ref{ass:polyak} hold,  and let \\$\eta\le\min\{\frac{1}{\mu},\sqrt{\frac{(1-\alpha_1)\mu}{4L^4\alpha_2}},\frac{1}{L}\}$ and $\bar\zeta:=\max\{\frac{1+\alpha_1}{2},1-\frac{\eta\mu}{2}\}$, where $\alpha_1 := \frac{3+\lambda}{4},\alpha_2:= 4(\frac{2}{1-\lambda}-1)$ same as in Lemma \ref{lem:con},
Then for $t\ge1$
\bea
\hat F(\bar w^{(t)}) \le \bar\zeta^{t-1} \hat F(\bar w^{(1)}).
\eea
 \end{lemma}
 \begin{proof}
 By $L-$ smoothness we have
 \begin{align*}
 \hat F(\bar w^{(t+1)}) &\le \hat F(\bar {w}^{(t)})- \eta \langle \nabla \hat F(\bar {w}^{(t)}),\bar \nabla \hat F(W^{(t)})\rangle + \frac{\eta^2 L}{2} \|\bar \nabla \hat F(W^{(t)})\|^2\\
 & = \hat F(\bar {w}^{(t)})- \frac{\eta-\eta^2 L}{2}\|\bar\nabla \hat F(W^{(t)})\|^2 -\frac{\eta}{2} \|\nabla \hat F(\bar {w}^{(t)})\|^2 + \frac{\eta}{2}  \|\bar \nabla \hat F(W^{(t)}) - \nabla \hat F(\bar {w}^{(t)})\|^2\\
 &\le  \hat F(\bar {w}^{(t)}) -\frac{\eta}{2}  \|\nabla \hat F(\bar {w}^{(t)})\|^2 + \frac{\eta}{2}  \|\bar \nabla \hat F(W^{(t)}) - \nabla \hat F(\bar {w}^{(t)})\|^2\\
 &\le \hat F(\bar {w}^{(t)}) -\frac{\eta}{2}  \|\nabla \hat F(\bar {w}^{(t)})\|^2 + \frac{\eta L^2}{N} \|W^{(t)} - \bar W^{(t)}\|_F^2.
 \end{align*}
 By $\mu-$PL condition we have,
 \bea\nn
  \hat F(\bar w^{(t+1)}) &\le(1-\eta\mu) \hat F(\bar {w}^{(t)})+ \frac{\eta L^2}{N} \|W^{(t)} - \bar W^{(t)}\|_F^2.
 \eea
 By Lemma \ref{lem:con},
\bea\label{eq:con-mupl}
\frac{1}{N}\|W^{(t)}- \bar W^{(t)}\|_F^2 < \alpha_2\eta^2 L \sum_{i=1}^{t-1}\alpha_1^{i-1}\hat F(\bar w^{(t-i)}).
\eea
which results in,
\bea\nn
  \hat F(\bar w^{(t+1)}) &\le(1-\eta\mu) \hat F(\bar {w}^{(t)})+ \alpha_2 L^3 \eta^3  \sum_{i=1}^{t-1} \alpha_1^{i-1}\hat F(\bar w^{(t-i)}).
\eea
By induction assume $\hat F(\bar w^{(t)}) \le \bar\zeta^{t-1} \hat F(\bar w^{(1)})$ then using the assumptions on $\bar\zeta$ and $\eta$ yield the following inequalities, 
\begin{align*}
  \hat F(\bar w^{(t+1)}) &\le(1-\eta\mu)  \bar\zeta^{t-1} \hat F(\bar w^{(1)})+ \alpha_2\eta^3 L^3 \hat F(\bar w^{(1)})\sum_{i=1}^{t-1} \alpha_1^{i-1}\bar\zeta^{t-i-1}\\
  & \le(1-\eta\mu) \bar \zeta^{t-1} \hat F(\bar w^{(1)})+ \alpha_2\eta^3 L^3 \hat F(\bar w^{(1)})\frac{\bar\zeta^{t-2}}{1-\alpha_1/\bar\zeta}\\
  & =(1-\eta\mu + \alpha_2\eta^3L^3/(\bar\zeta-\alpha_1))\bar\zeta^{t-1}\hat F(\bar w^{(1)})\\
  &\le (1-\eta\mu + 2\alpha_2\eta^3L^3/(1-\alpha_1))\bar\zeta^{t-1}\hat F(\bar w^{(1)})\\
  &\le (1-\eta\mu + \eta\mu/2)\bar\zeta^{t-1}\hat F(\bar w^{(1)})\\
  & \le \bar\zeta^t\hat F(\bar w^{(1)}).
\end{align*}
This completes the proof of the lemma. 
 \end{proof}
% \begin{lemma}
%Under the assumptions of Lemma \ref{Lem:app_pl}, the consensus error satisfies for $T > 1$, 
%\bea\nn
%\|W^{(T)} - \bar W^{(T)}\|_F^2  \le   \frac{2\alpha_2\eta^2L^2\hat F(\bar w^{(1)})}{1-\alpha_1}\zeta^{T-1}
%\eea
%\end{lemma}
\subsection{Proof of Lemma \ref{lem:pl-tr}}\label{sec:pf-lem5}
 \begin{lemma}[Restatement of Lemma \ref{lem:pl-tr}]\label{lem:pltrainfinal}
Let Assumptions \ref{ass:mixing},\ref{ass:smooth} and \ref{ass:polyak} hold and let the step-size $\eta\le\min\{\frac{1-\alpha_1}{\mu},\frac{1}{2L^2}\sqrt{\frac{(1-\alpha_1)\mu}{\alpha_2}},\frac{1}{L}\}$, where the constants $\alpha_1\in(0,1)$ and $\alpha_2>0$ are defined same as in Lemma \ref{Lem:app_pl}. Define  $\zeta:=1-\frac{\eta\mu}{2}$, then under the data separability assumption, the iterates of DGD satisfy for all $t\ge1$,
\bea
\hat F(\bar w^{(t)})  &\le \zeta^{t-1} \hat F(\bar w^{(1)}),\nn\\
\frac{1}{N}\|W^{(t)} - \bar W^{(t)}\|_F^2  &\le   \frac{2\alpha_2\eta^2L^2\hat F(\bar w^{(1)})}{1-\alpha_1}\zeta^{t-1}.\nn
\eea
 \end{lemma}
 \begin{proof}
 The bound on the train loss follows directly by Lemma \ref{Lem:app_pl}, after noting that $\eta\le\frac{1-\alpha_1}{\mu}$ implies $\frac{1+\alpha_1}{2}\le1-\eta\mu/2$. The consensus error is derived by \eqref{eq:con-mupl} and using the bound on $\hat F(\bar w^{(t)})$. 
 \end{proof}
 \subsection{Proof of Theorem \ref{thm:PL}}\label{sec:pf-thm6}
\begin{theorem}[Restatement of Theorem \ref{thm:PL}]
Let Assumptions \ref{ass:mixing}-\ref{ass:self-bounded} and \ref{ass:polyak} hold, and let $\eta$ and $\zeta$ be as in Lemma \ref{Lem:app_pl}. Then the iterates of DGD under the data separability assumption satisfy for all $T\ge1$,
 \bea\nn
 \E\Big[F(\bar w^{(T)})\Big] =O \Big(\zeta^T  + \frac{ L^2 c^2}{n^{3-2\alpha} \mu^{2\alpha}}  (\eta T)^{2-2\alpha} +\frac{\eta^2 L^4}{\mu^2 N}\Big).
\eea
\end{theorem}

\begin{proof}
By simplifying Lemma \ref{lem:main} using the convergence bounds in Lemma \ref{lem:pltrainfinal}, we end up with the following, 
 \bea\nn
 \E\Big[F(\bar w^{(T)})\Big] =O \Big(\zeta^T  + \frac{ L^2 c^2\eta^2}{n^{3-2\alpha} (1-\zeta)^{2 \alpha}}   T^{2-2\alpha} +\frac{\eta^2 L^4}{(1-\sqrt{\zeta})^2 N}\Big).
 \eea
Based on the definition of $\zeta$, we have $(\frac{\eta}{1-\zeta})^{2\alpha}=(\frac{2}{\mu})^{2\alpha}$ and $\frac{\eta^2}{(1-\sqrt{\zeta})^2}\le \frac{4}{\mu^2}$. This proves the statement of the theorem. 
\end{proof}

%%%%%%%%%%%%%%%%%%%%%%%%%%%%%%%%%%% LAST THEOREM %%%%%%%%%%%%%%%%
\section{Proof of Theorem \ref{thm:alg}}\label{sec:alg_proofs}
\begin{theorem}[Restatement of Theorem \ref{thm:alg}]
Consider $\mathsf{FDRL}$(Algorithm \ref{alg:alg1}) on separable dataset, and choose $\eta=O (1/\sqrt{t})$. Then for all $\ell\in [N]$
$$
\lim_{t\rightarrow\infty} \frac{{w}_\ell^{(t)}}{\|{w}_\ell^{(t)}\|} =  \frac{{w}_{_{\rm MM}}}{\|{w}_{_{\rm MM}}\|},
$$
where recall that $w_{\rm MM}$ denotes the solution to hard-margin SVM problem. 
\end{theorem}
\begin{proof}
Replace $\frac{{v}_\ell^{(t)}}{\|{v}_\ell^{(t)}\|}$ in step 2 of Algorithm 1 by arbitrary perturbations $\eps_\ell^{(t)}$ of unit norm. Then note that the sequence $\{{w}_\ell^{(t)}\}$ generated by step 2 is identical to decentralized GD with $\eta\|\eps_\ell^{(t)}\|\rightarrow 0$. Thus by \cite[Lemma 1]{nedic2014distributed}, consensus is asymptotically achieved for all $\ell\in [N]$, i.e., 
\bea\nn
\lim_{t\rightarrow\infty} \|{w}_\ell^{(t)} - \bar{w}^{(t)}\| = 0.
\eea
Thus
\bea\nn
\lim_{t\rightarrow\infty} \|{w}_\ell^{(t+1)} - {w}_\ell^{(t)}\| = \lim_{t\rightarrow\infty} \|\bar{{w}}^{(t+1)} - \bar{{w}}^{(t)}\| =  \lim_{t\rightarrow\infty} \eta \| \bar{\eps}^{(t)}\| = 0.
\eea
This implies that for all $i\in[N]$ we have $ \lim_{t\rightarrow\infty} \|\nabla{\hat F}_\ell({w}_\ell^{(t+1)})-\nabla{\hat F}_\ell({w}_\ell^{(t)})\| =0$, thus by appealing again to \cite[Lemma 1]{nedic2014distributed} and applying it to step (5) of Algorithm \ref{alg:alg1}, we find that, 
\bea\nn
\lim_{t\rightarrow\infty} \|{v}_\ell^{(t)} - \bar{v}^{(t)}\| = 0.
\eea
Aggregations of gradients in step (5) imply that $\bar{v}^{(t)} = \frac{\one^\top\nabla\Lm(W^{(t)})}{N}\rightarrow \nabla{\hat F}(\bar{w}^{(t)})$. Thus step (2) of $\mathsf{FDLR}$ for every agent $i$ converges to $\bar{{w}}^{(t)} - \eta\frac{\nabla{\hat F}(\bar{w}^{(t)})}{\|\nabla{\hat F}(\bar{w}^{(t)})\|}$, i.e., 
\bea\nn
\Big\|\Big({w}_\ell^{(t)}-\eta\frac{{v}_\ell^{(t)}}{\|{v}_\ell^{(t)}\|} \Big) - \Big( \bar{{w}}^{(t)} - \eta\frac{\nabla{\hat F}(\bar{w}^{(t)})}{\|\nabla{\hat F}(\bar{w}^{(t)})\|}\Big)\Big\| \overset{t\rightarrow\infty}\Longrightarrow 0. 
\eea
Thus for all $\ell$, the sequence $\{{w}_\ell^{(t)}\}$ converges to the solution of normalized GD, i.e., the max-margin separator ${w}_{_{\rm MM}},$ for linearly separable datasets (\cite[Theorem 5]{nacson2019convergence}). This leads to the statement of the theorem. 
\end{proof}

\section{Auxiliary Results}\label{sec:aux}

\begin{proposition}[Bounds on the exponential loss]\label{propo:exp}
Consider linear classification with the exponential loss $f(w,(a,y))=\exp(-y\cdot w^\top a)$ over linearly separable dataset $(a_i,y_i)_{i=1}^n$ with binary labels $y_i$ and with $\max_i \|a_i\|\le r$ for a constant $r$. The training loss in this case satisfies for all $w\in\R^d$, 
\bea\label{eq:explossass}
\|\nabla \hat F(w)\| \in [c' F(w), c F(w)],~~ \|\nabla^2 \hat F(w)\|\le h F(w), 
\eea
for constants $c,c'$ and $h$ independent of $w$.  
\end{proposition}
\begin{proof}
Using
$\hat F(w) = \frac{1}{n}\sum_{i=1}^n \exp(-y_i\cdot w^\top a_i)$, one can deduce that,
\bea\nn
\nabla \hat F(w) &= -\frac{1}{n}\sum_{i=1}^n y_i a_i \exp(-y_i\,w^\top a_i),\\
\nabla^2 \hat F(w) &= \frac{1}{n}\sum_{i=1}^n  a_ia_i^\top \exp(-y_i\,w^\top a_i).\nn
\eea
Therefore it holds that, 
\begin{align*}
\|\nabla \hat F(w) \| &= \frac{1}{n} \| \sum_{i=1}^n y_i a_i \exp(-y_i w^\top a_i)\|
\le \frac{1}{n} \sum_{i=1}^n \|y_i a_i \exp(-y_i w^\top a_i)\|=   \frac{1}{n} \sum_{i=1}^n \|y_i a_i\| \exp(-y_i w^\top a_i)\le r \hat F(w).
\end{align*}
A similar approach for the Hessian of $\hat F$ results in the following inequality, 
\bea\nn
\|\nabla^2 \hat F(w) \| \le r^2 \hat F(w). 
\eea
Moreover, due to linear separability there exists a $w^\star\in\R^d$ such that,
\bea\nn
\frac{y_i{w^\star}^\top a_i}{\|w^\star\|} \ge \gamma,~~ \forall i\in[N],
\eea
where $\gamma>0$ denotes the margin. Therefore, using the supremum definition of norm we can write,
\begin{align*}
\|\nabla \hat F(w) \| &= \frac{1}{n} \left\| \sum_{i=1}^n y_i a_i \exp(-y_i w^\top a_i)\right\|\\
&= \sup_{\substack{{v\in\R^d}\\[2pt] {\text{s,t.}~\|v\|=1}}}\left\langle\frac{1}{n} \sum_{i=1}^n y_i a_i \exp(-y_i w^\top a_i),v\right\rangle\\
&\ge \left\langle\frac{1}{n} \sum_{i=1}^n y_i a_i \exp(-y_i w^\top a_i),\frac{w^\star}{\|w^\star\|}\right\rangle\\
&\ge \frac{1}{n} \sum_{i=1}^n \gamma\cdot\exp(-y_i w^\top a_i)\\
&= \gamma\hat F(w).
\end{align*}
This completes the proof. 
\end{proof}

\begin{proposition}[Bounds on the logistic loss]
Consider linear classification with the logistic loss $f(w,(a,y))=\log(1+\exp(-y\cdot w^\top a))$ over linearly separable dataset $(a_i,y_i)_{i=1}^n$ with binary labels $y_i$ and with $\max_i \|a_i\|\le r$ for a constant $r$. The training loss in this case satisfies for all $w\in\R^d$, 
\bea
\|\nabla \hat F(w)\| \in [c' \Phi(w), c F(w)],~~ \|\nabla^2 \hat F(w)\|\le h F(w), \nn
\eea
for $\Phi(w):= \frac{1}{n}\sum_{i=1}^n \frac{\exp(-y_i\,w^\top a_i)}{1+\exp(-y_i\, w^\top a_i)} $ and constants $c,c'$ and $h$ independent of $w$.  
\end{proposition}
\begin{proof}
The training loss is now
$
\hat F(w) = \frac{1}{n}\sum_{i=1}^n \log(1+\exp(-y_i\cdot w^\top a_i)).
$
Thus,
\bea\nn
\nabla \hat F(w) &= \frac{1}{n}\sum_{i=1}^n(-y_i a_i ) \frac{\exp(-y_i\,w^\top a_i)}{1+\exp(-y_i\cdot w^\top a_i)},\\
\nabla^2 \hat F(w) &= \frac{1}{n}\sum_{i=1}^n  a_ia_i^\top\frac{\exp(-y_i\,w^\top a_i)}{ (1+\exp(-y_i\,w^\top a_i))^2}.\nn
\eea
By considering the norm and noting that $\exp(t)/(1+\exp(t))\le  \log(1+\exp(t))$,
\begin{align*}
\|\nabla \hat F(w) \| &=\frac{1}{n}\left\|\sum_{i=1}^n(-y_i a_i ) \frac{\exp(-y_i\,w^\top a_i)}{1+\exp(-y_i\cdot w^\top a_i)}\right\|\\
&\le \frac{1}{n}\sum_{i=1}^n\left\|y_i a_i\right\| \frac{\exp(-y_i\,w^\top a_i)}{1+\exp(-y_i\cdot w^\top a_i)}\\
&\le \frac{r}{n}\sum_{i=1}^n \log(1+\exp(-y_i w^\top a_i))= r \hat F(w). 
\end{align*}
Likewise, since $\exp(t)/(1+\exp(t))^2\le 2 \log(1+\exp(t))$, we can conclude that the operator norm of the Hessian satisfies,
\bea
\nabla^2 \hat F(w)\le 2r^2 \hat F(w). \nn
\eea
This completes the proof of upper-bounds for the gradient and Hessian. For the lower-bound on gradient note that by using the supremum definition of norm and recalling the max-margin separator satisfies $\frac{y_i{w^\star}^\top a_i}{\|w^\star\|} \ge \gamma$ for the margin $\gamma>0$ and all $i\in[n]$, we can write,
\begin{align*}
\|\nabla \hat F(w) \| &= \frac{1}{n} \left\| \sum_{i=1}^n y_i a_i \frac{\exp(-y_i\,w^\top a_i)}{1+\exp(-y_i\cdot w^\top a_i)}\right\|= \sup_{\substack{{v\in\R^d}\\[2pt] {\text{s,t.}~\|v\|=1}}}\left\langle\frac{1}{n} \sum_{i=1}^n y_i a_i \frac{\exp(-y_i\,w^\top a_i)}{1+\exp(-y_i\cdot w^\top a_i)},v\right\rangle\\
&\ge \left\langle\frac{1}{n} \sum_{i=1}^n y_i a_i \frac{\exp(-y_i\,w^\top a_i)}{1+\exp(-y_i\cdot w^\top a_i)},\frac{w^\star}{\|w^\star\|}\right\rangle\\
&\ge \frac{1}{n}\sum_{i=1}^n \gamma \frac{\exp(-y_i\,w^\top a_i)}{1+\exp(-y_i\cdot w^\top a_i)}.
\end{align*}
This yields the lower bound $\gamma\Phi(w)$ on the norm of gradient and completes the proof. 
\end{proof}
\begin{proposition}[Realizability of the exponential and logistic loss \cite{schliserman2022stability}]\label{prop:realizable}
On linearly separable data with margin $\gamma>0$, the exponential loss function satisfies the realizability assumption (Assumption \ref{ass:realizable}) with $\rho(\eps) = -\frac{1}{\gamma}\log(\eps)$, where $\gamma$ denotes the margin. Moreover, the logistic loss function satisfies the realizability assumption with $\rho(\eps) = -\frac{1}{\gamma}\log(\exp(\eps)-1)$. 
\end{proposition}
%%%%%%%%%%%%%%%%%%%%%%%%%%%%%%%%%%%%%%%%%%%%%%%%%%%%%%%%%%%%%%

\section{Additional experiments}\label{sec:appG}
\subsection{Experiments on over-parameterized Least-squares}\label{sec:plexp}
In Fig. \ref{fig:4}, we conduct experiments for highly over-parameterized Least-squares ($f(w,x)=(1-w^\top x)^2$), where $d$ is typically significantly larger than $n$ to ensure perfect interpolation of dataset. Note that, the train loss is not strongly-convex in this case, instead it satisfies the PL condition(Assumption \ref{ass:polyak}). Notably, as predicted by Lemma \ref{lem:pl-tr}, we notice the linear convergence of the train loss and the consensus error in Fig. \ref{fig:4} (Left). On the other hand, for the test loss, we observe its remarkably fast convergence (after approximately 50 iterations) to the optimal value, which is followed by a sharp increase in the subsequent iterations. 

\subsection{On the update rule of $\mathsf{FDLR}$}

In the final section of the paper, we state a remark regarding the update rule of $\mathsf{FDLR}$. Recall the update rule of DGD, 
\bea\label{eq:dec_app}
w^{(t+1)}_\ell = \sum_{k\in\Nn_\ell} A_{\ell k} w_k^{(t)} - \eta_t\nabla \hat F_\ell (w_\ell^{(t)}). 
\eea
\begin{figure}[t]
\includegraphics[width=6.4cm,height=5.3cm]{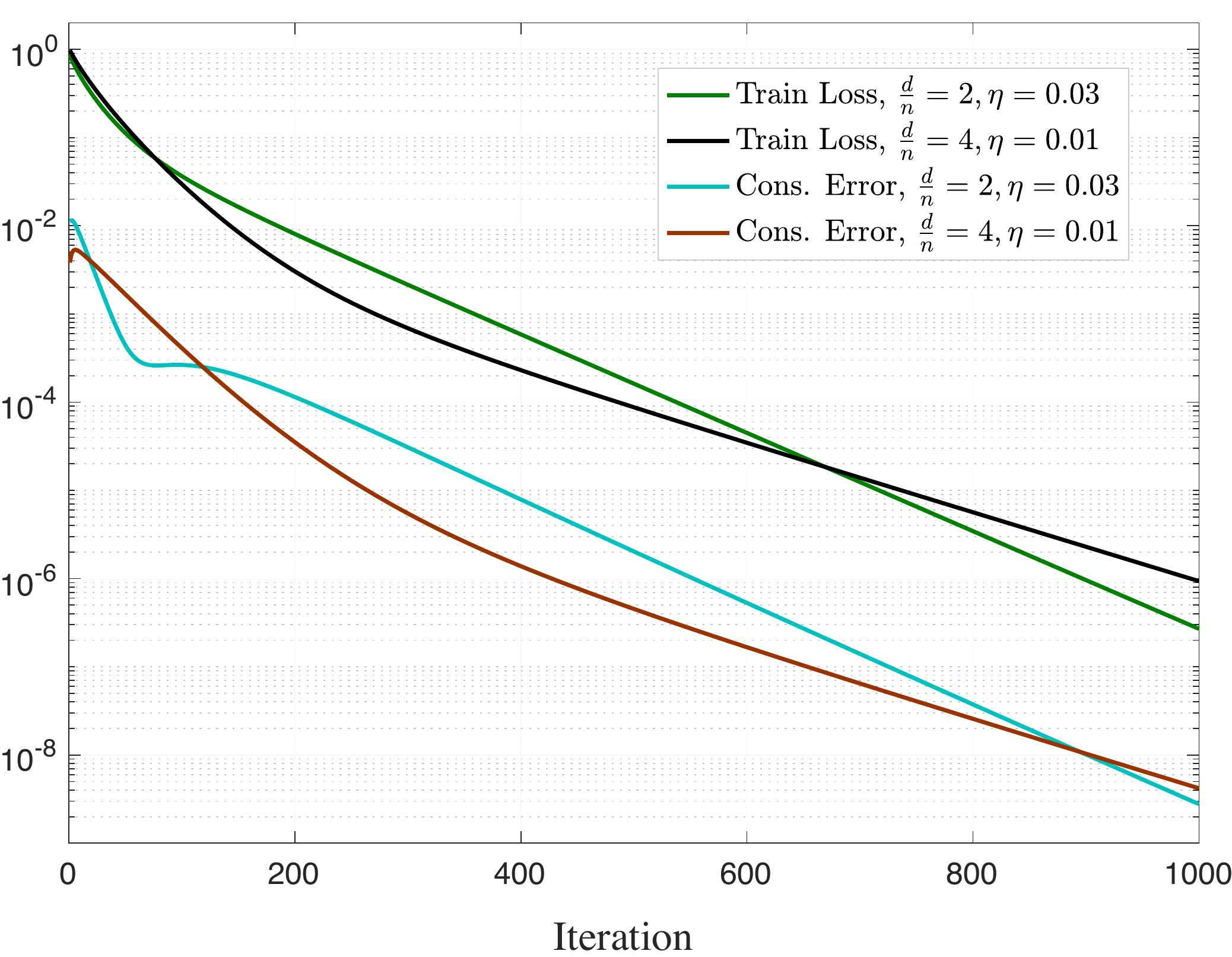}
\centering
%\caption{Misclassification Train Errors(\%) for our proposed algorithms compared to the vanilla decentralized Gradient-Descent algorithm in logistic regression with Signed measurements.}
\;\;\;\;\;
\includegraphics[width=6.4cm,height=5.3cm]{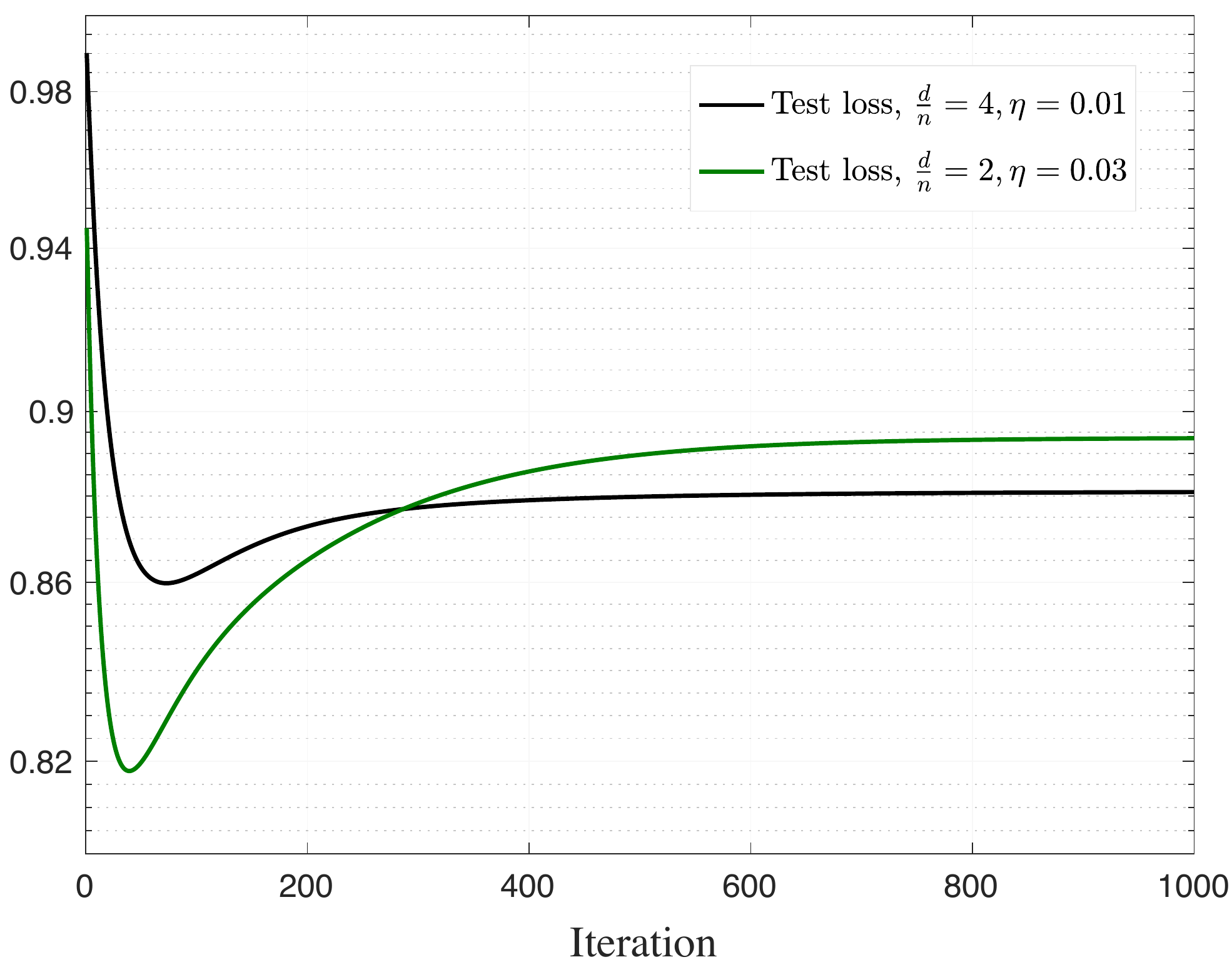}
\caption{Consensus error, train loss and test loss for DGD with over-parameterized least-squares(square loss). The test loss achieves its optimum at the very early stages of DGD.}
\label{fig:4}
\end{figure}
\begin{figure}[t]
\includegraphics[width=7.4cm,height=5.9cm]{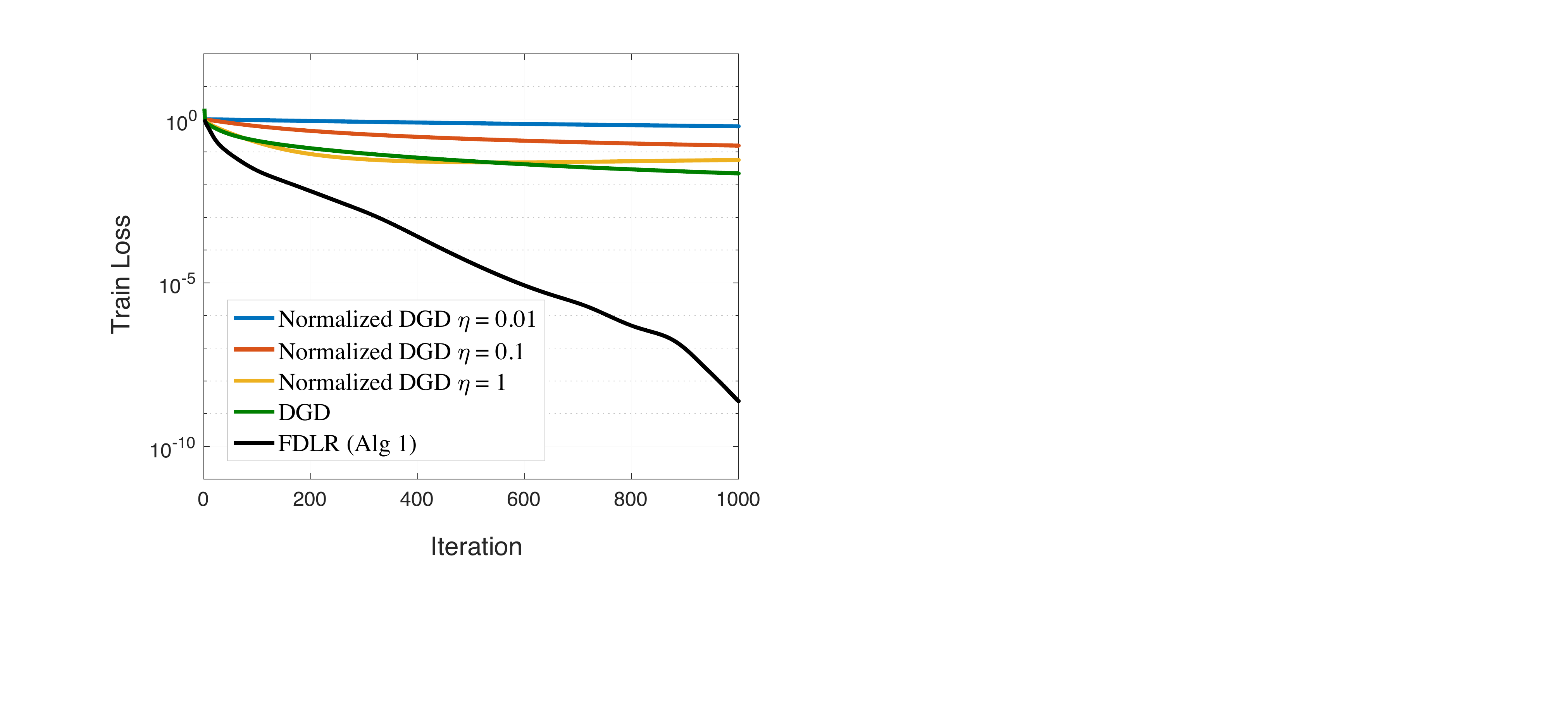}
\centering
\caption{Normalized DGD with the update rule in Eq.\eqref{eq:normalizedDGDapp} for different step-sizes $\eta$ compared to DGD (Eq.\eqref{eq:dec_app}) and to $\mathsf{FDLR}$ (Alg \ref{alg:alg1}). The step-sizes for DGD and FDLR are fine-tuned so that best of each algorithm is depicted. Normalized DGD cannot outperform DGD while FDLR is significantly faster than DGD. Here we consider linear classification with the exponential loss function and the dataset is generated according to signed measurements with Gaussian features and $n=100,d=50.$}
\label{fig:NDGD}
\end{figure}

Notably, we expect $\mathsf{FDLR}$ to be perhaps the simplest approach for accommodating normalized gradients in decentralized learning setting since in DGD the agents only have access to their local gradients. In particular consider a Normalized DGD algorithm with the same update as in \eqref{eq:dec_app} but with $\nabla \hat F_i(w_i^{(t)})$ replaced by $\nabla \hat F_i(w_i^{(t)})/\|\nabla \hat F_i(w_i^{(t)})\|$, i.e.,

\bea\label{eq:normalizedDGDapp}
w^{(t+1)}_\ell = \sum_{k\in\Nn_\ell} A_{\ell k} w_k^{(t)} - \eta_t\frac{\nabla \hat F_\ell (w_\ell^{(t)})} {\|\nabla \hat F_\ell(w_\ell^{(t)})\|}
\eea

The Normalized DGD algorithm above does not lead to faster convergence. This is due to the fact that in DGD the local gradient norm $\|\nabla \hat F_i(w_i^{(t)})\|$ can be different than the global gradient norm $\|\nabla \hat F(w_i^{(t)})\|$. Thus even if with the update rule \eqref{eq:normalizedDGDapp} the local parameters $w_i^{(t)}$ converge to the global optimal solution, still the update rule for the averaged parameter $\bar w^{(t)}$ is different than the update rule of centralized normalized GD.  
Our numerical experiment in Fig. \ref{fig:NDGD} demonstrates the incapability of Normalized DGD in speeding up DGD. In particular, we note that for any choice of step-size Normalized DGD does not lead to acceleration  compared to DGD whereas $\mathsf{FDLR}$ massively outperforms DGD.

\section{A note about convergence rates of DGD}
As mentioned in the paper's introduction, many prior works on investigate convergence of DGD and of its stochastic variant decentralized stochastic gradient descent (DSGD) under various assumptions, e.g.
\cite{jiang2017collaborative,wang2019slowmo,koloskova2019decentralized,koloskova2020unified} and many references therein. Most recently,  \cite{koloskova2020unified} has presented a powerful unifying analysis of DSGD under rather weak assumptions. Specialized to convex $L$-smooth functions for which there exists $w^*$ such that $\|\nabla f_i(w^*)\|=0$ (i.e. interpolation) \cite[Thm.~2]{koloskova2020unified} shows a rate of $\Oc(LR_0/T)$ for average DSGD updates. Here, $R_0=\|w_1-w^*\|_2$. Ignoring logarithmic factors, this rate is the same as what we obtained in \eqref{eq:trainlosscvx} (as a consequence of Lemma \ref{lem:trainloss-convex}) for DGD specifically applied to logistic loss over separable data. However, our result does \emph{not} directly follow from \cite[Thm.~2]{koloskova2020unified}. The reason is that logistic loss on separable data does \emph{not} attain a bounded estimator. In fact, we believe the $\log^2 T$ dependence of the rate that shows up in our analysis (see Eq. \eqref{eq:trainloss_convex}), is a consequence of the  infinitely normed-optimizers in our setting and we expect the bound to be tight as suggested by our experiments (see Fig \ref{fig:3}) and in agreement with convergence bounds for logistic regression on separable data in the centralized case derived recently in \cite[Thm.~1.1]{ji2018risk}. On the other hand, the results of \cite{koloskova2020unified} are applicable to finite optimizers, which yields $\Oc(1/T)$ convergence rates without $\log$ factor. Besides the above, in Theorem \ref{lem:exp_dsc}, we prove novel last-iterate (as opposed to averaged in the literature) convergence bounds for the train loss and faster consensus error rates of $\tilde \Oc(1/T^2).$ This is possible by leveraging additional Hessian self-bounded (Ass. \ref{ass:laplace}) and self-lower-boudned (Ass. \ref{ass:8}) assumptions, which hold for example for the exponential loss. Finally, 
 we recall that our main focus is on studying finite time \emph{generalization} bounds for DGD (e.g. Thm. \ref{thm:testloss_cvx}), which to the best of our knowledge are new in this setting. Having discussed these, it is worth noting that the analysis of \cite{koloskova2020unified}
 applies under a relaxed assumption on the mixing matrix (see \cite[Ass.~4]{koloskova2020unified}) than the corresponding assumptions (e.g.  Assumption \ref{ass:mixing}) in the literature.  For example, this relaxed assumption covers decentralized local SGD (with multiple local updates per iteration) as a special case and is interesting to extend our results (on logistic regression over separable data) to such settings.

\end{document}

%% file: commands.tex
\makeatletter
\newcommand*{\rom}[1]{\expandafter\@slowromancap\romannumeral #1@}
\makeatother

\newcommand{\mathleft}{\@fleqntrue\@mathmargin0pt}
\newcommand{\mathcenter}{\@fleqnfalse}

\makeatletter
\newcommand{\ssymbol}[1]{^{\@fnsymbol{#1}}}
\makeatother
\newcommand{\R}{\mathbb{R}}

\newcommand{\Lm}{\mathcal{L}}

\newcommand{\simiid}{\widesim{\text{\small{iid}}}}

\theoremstyle{theorem}
\newtheorem{ass}{Assumption}

\theoremstyle{remark}
\newtheorem{remark}{Remark}%[subsection]
\newcommand{\eps}{\varepsilon}

\newcommand{\sign}{\texttt{sign}}

\newcommand{\one}{\mathbf{1}}

\newcommand{\E}{\mathbb{E}}                    % expectation
\newcommand{\nn}{\notag}

\newcommand{\Dc}{\mathcal{D}}

\newcommand{\Nn}{\mathcal{N}}

\newcommand{\Oc}{\mathcal{O}}

\newcommand{\beq}{\begin{equation}}
\newcommand{\eeq}{\end{equation}}
\newcommand{\bea}{\begin{align}}
\newcommand{\eea}{\end{align}}